\documentclass[twoside,11pt]{article}

\usepackage{blindtext}

\usepackage[abbrvbib, preprint]{jmlr2e}

\usepackage[T1]{fontenc}

\usepackage{hyperref}

\hypersetup{colorlinks=true, pdfauthor={Name}}
\usepackage[]{xcolor}

\usepackage{amsmath}

\usepackage{enumerate}
\usepackage{algorithm}
\usepackage{algorithmic}

\usepackage{subcaption}
\usepackage{placeins}

\DeclareMathOperator{\dir}{\text{Dir}}

\DeclareMathOperator{\KL}{\textbf{KL}}

\DeclareMathOperator{\p}{\mathbb{P}}
\DeclareMathOperator{\M}{\mathcal{M}}

\newcommand{\bA}{\mathbf{A}}
\newcommand{\bZ}{\mathbf{Z}}
\newcommand{\bW}{\mathbf{W}}
\newcommand{\bTheta}{\mathbf{\Theta}}
\newcommand{\bpi}{\boldsymbol{\pi}}
\newcommand{\brho}{\boldsymbol{\rho}}
\newcommand{\bbeta}{\boldsymbol{\beta}}
\newcommand{\balpha}{\boldsymbol{\alpha}}

\usepackage{ulem}

\usepackage{lastpage}
\jmlrheading{VV}{2024}{1-\pageref{LastPage}}{MM/YY; Revised MM/YY}{MM/YY}{XX-XXXX}{Kylliann De Santiago, Marie Szafranski and Christophe Ambroise}

\usepackage{bbold}

\ShortHeadings{Mixture of multilayer SBM models for multiview clustering}{De Santiago, Szafranski and Ambroise}
\firstpageno{1}

\begin{document}

\title{Mixture of multilayer stochastic block models for multiview clustering}

\author{\name Kylliann De Santiago \email kylliann.desantiago@univ-evry.fr \\
       \addr Université Paris-Saclay, CNRS, Univ Evry,\\ Laboratoire de Mathématiques et Modélisation d'Evry,\\
       91037, \'Evry-Courcouronnes, France.\\
       \AND
       \name Marie Szafranski \email marie.szafranski@univ-evry.fr \\
       \addr ENSIIE,  91025, \'Evry-Courcouronnes, France.\\
       Université Paris-Saclay, CNRS, Univ Evry,\\ Laboratoire de Mathématiques et Modélisation d'Evry,\\
       91037, \'Evry-Courcouronnes, France.\\
        \AND
        \name Christophe Ambroise \email christophe.ambroise@univ-evry.fr \\
       \addr Université Paris-Saclay, CNRS, Univ Evry,\\ Laboratoire de Mathématiques et Modélisation d'Evry,\\
       91037, \'Evry-Courcouronnes, France.}
\editor{My editor}

\maketitle

\begin{abstract}
In this work, we propose an original method for aggregating multiple clustering coming from different sources of information.  Each partition is encoded by a co-membership matrix between observations.  
Our approach uses a  mixture of multilayer Stochastic Block Models (SBM) to group co-membership matrices with similar information into components and to partition observations into different clusters, taking into account their specificities within the components. The identifiability of the model parameters is established and a variational Bayesian EM algorithm is proposed for the estimation of these parameters. The Bayesian framework allows for selecting an optimal number of clusters and components. The proposed approach is compared using synthetic data with consensus clustering and tensor-based algorithms for community detection in large-scale complex networks. Finally, the method is utilized to analyze global food trading networks, leading to structures of interest.
\end{abstract}


\section{Introduction}

Most everyday learning situations are achieved by integrating different sources of information, such as vision, touch and hearing. A source of information in a given format will be referred to as a {modality} or 
a \textit{view}.  
Multimodal or multiview machine learning aims to learn models from multiple views (e.g. text, sound, image, etc.) in order to represent, translate, align, fusion, or co-learn~\citep[see][for instance]{zhao2017multi, baltruvsaitis2018multimodal, cornuejols2018collaborative}.

Graphs provide a powerful and intuitive way to represent complex systems of relationships between individuals. They provide an effective and informative representation of the system. Constructing graphs from each view allows to use graph machine learning for multimodal clustering \citep{ektefaie2023multimodal}.



In clustering framework, the output of algorithms are often a partition or a membership matrix $\bZ$. This information, although useful, has the drawback of strongly depending on the number of clusters chosen when using the algorithm. To avoid this problem, $\bZ$ can be transform into an adjacency matrice $\bA$ with 
\begin{equation*}
\label{adjacencySBM}
    {A}_{ij}=
    \begin{cases}
      1, & \text{if}\ \text{individuals $i$, $j$ are linked in the same cluster}, \\
      0, & \text{otherwise}.
    \end{cases}
\end{equation*}

The process of combining numerous data clusters that have already been discovered using various clustering algorithms or approaches is known as \textit{meta clustering}. Finding connections and similarities across clusters that might not be immediately obvious when looking at them separately is the aim of meta or \textit{consensus clustering} \citep{montiConsensusClusteringResamplingBaseda,li2015large,liu2018multi}.
Model-based consensus clustering offer advantages: knowing the redundancy of information sources and their complementarity, obtaining a final clustering from all the  outputs already carried out, allowing the best possible grouping of individuals through different information sources.
Moreover, model-based consensus clustering allows to have an evaluation criterion on the performance of the model (e.g. log-likelihood, evidence, etc.) and, at least in the Bayesian framework, criteria for model selection \citep{biernacki2010exact}. 

The corresponding learning models vary based on their fusion strategy. The three main categories of methods are early, intermediate, and late fusion of views. 
Late fusion is well suited to clustering since each view 
is often associated to dedicated efficient clustering algorithms. 
 
\paragraph{Contribution.} 
In this work we propose to estimate a coordinated representation produced by learning separate clustering for each view and then coordination through a probabilistic model: MIxture of Multilayer Integrator Stochastic Block Model (mimi-SBM).  Our model is a Bayesian mixture of multilayer SBM  
that takes into account several sources of information, and where the membership clustering is traversing as illustrated in Figure~\ref{fig: mimi-sbm}. 


In simpler terms, each individual belongs exclusively to one group, and not to a multitude of groups across views or sources. In the context of meta-clustering, this has the advantage of allowing the model to find common information for each group, by striving for clustering redundancy across sources.
Moreover, by applying a mixture model on the views, we can take into account the particularities of each source of information, and define the redundant and complementary information sources in order to draw a maximum of information from them.
Finally, with the Bayesian framework, the development of a model selection criterion, both for the mixture of views and the number of clusters is possible by deriving it from the evidence lower bound. The identifiability of the model parameters is established and a variational Bayesian EM algorithm is proposed for the estimation of these parameters.

\begin{figure}[!ht]
    \centering
    \includegraphics[scale=0.75]{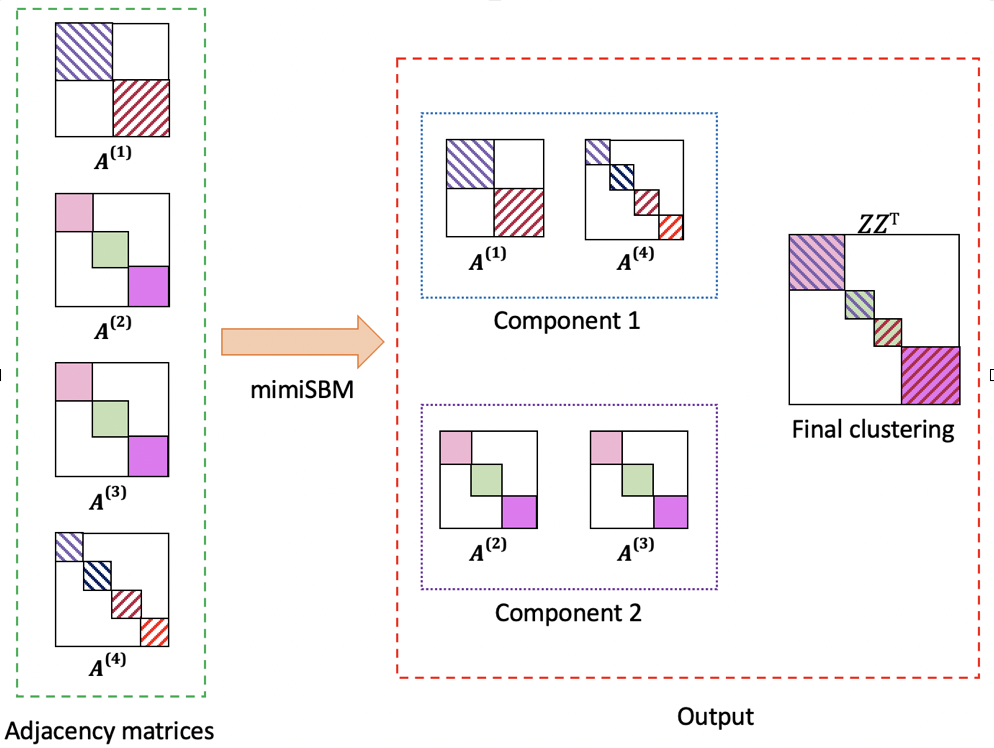}
    \caption{Illustration of mimi-SBM. Left: Four adjacency matrices $\bA^{(1)}, \cdots, \bA^{(4)}$ coming from four different views organized into two components. Right: identification of the two components from the views (local and complementary information) and clustering of the observations described by the classification matrix $\bZ$ (global and consensus information).}
    \label{fig: mimi-sbm}
\end{figure}

\paragraph{Organization of the paper.}
The paper is structured as follows: Firstly, we delve into the related works in more detail, providing a selective survey of the existing literature and research in the field.  Next, we present the description of our innovative \textit{mimi-SBM} model, outlining its key components and parameter estimation.  We then focus on model selection and variational parameter initialization, where we compare different criteria to determine the most effective approach. To evaluate the performance of our proposed approach, we conduct synthetic experiments that allow for a thorough comparison and evaluation. Finally, we engage in a discussion about the results obtained, and their implications, and provide insights into potential future directions for further research.


\section{Background}

In multiview clustering, various fusion strategies have been devised to effectively combine information from different sources. One possible taxonomy of these strategies is related to the timing of fusion: early, intermediate, or late. \textit{Early fusion} starts by merging the different views before clustering. \textit{Intermediate fusion} considers the integration of views within the clustering algorithm. \textit{Late fusion} consists of clustering each view separately and then integrating all the resulting partitions into a unique integrated clustering. This paper focuses on late multiview clustering.  In this scenario, all layers represented as adjacency matrices collectively form a tensor.

This latter strategy harnesses the benefits of employing specific and well-suited clustering approaches for each view and takes advantage of their complementarity and redundancies by merging their results in a subsequent phase. In this context, \textit{consensus clustering} serves as a baseline and will be first described in this section. 
On the other hand, \textit{latent} or \textit{stochastic block models} provide advantageous characteristics for multiview clustering in terms of model selection. In this section, we also provide a focused overview of these approaches, to which our method, tailored to late fusion clustering, belongs.






\subsection{Consensus clustering}



Consensus clustering \citep{montiConsensusClusteringResamplingBaseda,fred2005combining},  also known as cluster ensemble \citep{strehl2002cluster, golalipour2021clustering}, is a technique used to find a single partition from multiple clustering solutions. 

It is used to integrate and analyze multiple clustering results obtained from different algorithms, parameter settings, or subsets of the data. It aims to find a consensus or agreement among the individual clustering solutions to obtain a more robust and reliable clustering result. Consensus clustering is like asking for multiple opinions (clusters) and then finding a common answer (consensus) that represents an overall agreement.

Each run generates a set of clusters, which can be represented as a partition matrix where each entry indicates the cluster assignment of each data point. An agreement matrix is constructed from all the partition matrices. The entries in this matrix indicate the frequency with which a pair of data points co-occur in the same cluster across all solutions.
The final clusters are determined by applying a clustering algorithm 
on the agreement matrix. 

In particular, the Monte Carlo reference-based consensus clustering (M3C) \citep{montiConsensusClusteringResamplingBaseda} combines multiple clustering solutions generated by applying different clustering algorithms and parameters to the same dataset using random (re)sampling techniques.

\subsection{Block models for multiview clustering}

In the framework of block models for multiview clustering, the views are commonly denoted as a collection of $V$ \textit{graphs} or more often as $V$ \textit{layers} within a single network, and the terminologies of  \textit{multigraph}, \textit{multilayer} or even \textit{multiplex} may be employed. 

With a wide expanse of literature existing on this subject, we narrow our focus on studies in which the distinct views represent varying types of interactions among a common set of $N$ observations. However, it's important to note that our scope excludes studies that aim to establish partitions with overlaps or mixed memberships, as well as those where views show specific dependencies (spatial or temporal for instance).
Works of this nature can be discovered in the references provided below.

\subsubsection{Multilayer SBM} 


Multilayer SBM (MLSBM) approaches are focused on identifying a partition  $\bZ$ with $K$ blocks of observations that encompass the different layers. A popular kind of inference for block model estimation relies on Variational Expectation Maximization (VEM) algorithms ~\citep{daudin2008mixture}. Besides, when the data originate from a MLSBM, different estimation techniques can also be applied to identify the partition. Out of these, 
spectral clustering finds widespread usage.~\citep{von2007tutorial, von2008consistency}.

\paragraph{VEM inference.} 
In this setting, different SBM based approaches have been proposed for 
multilayer~\citep{han2015consistent, subhadeep2016consistent} or similarly for multiplex~\citep{barbillon2017stochastic} networks. \cite{han2015consistent} 
propose a consistent maximum-likelihood estimate (MLE) and explore the asymptotic properties of class memberships when the number of relations grows. \cite{subhadeep2016consistent} also study the consistency 
of two other MLEs when the number of nodes or the number of types of edges grow together. \cite{barbillon2017stochastic}  introduce an Erd\"{o}s-Rényi model that may also integrate covariates related to pairs of observations and use an Integrated Completed Likelihood (ICL)~\cite{biernacki2010exact} 
for the purpose of model selection. Also based on VEM estimation, the work of~\cite{boutalbi2021implicit} is grounded on Latent Block Models (LBM).

\paragraph{Spectral clustering.}   
Here, we shed light on several extensive research efforts focused on spectral clustering under the assumption of data generated by a Multilayer Stochastic Block Model.  \cite{han2015consistent} investigate the asymptotic characteristics related to spectral clustering. \cite{chen2017multilayer} introduce a framework for multilayer spectral graph clustering that includes a scheme for adapting layer weights, while also offering statistical guarantees for the reliability of clustering. \cite{mercado2018power} presents a spectral clustering algorithm for multilayer graphs that relies on the matrix power mean of Laplacians. \cite{paul2020spectal} show the consistency of the global optimizers of co-regularized spectral clustering and also for the orthogonal linked matrix factorization. Finally, \cite{huang2022spectral} propose integrated spectral clustering methods based on convex layer aggregations.



\subsubsection{Multiway block models} 

In contrast to the works presented above, where the aim is to establish a partition across the observations, the approaches presented below focus on establishing multiway structures, especially between- and within-layer partitions. In this context, the MLSBM can evolve into either a Mixture of Multilayer SBM (MMLSBM) or expand into a Tensor Block Model (TBM), depending on the specific research communities and their focus.

\paragraph{Mixture of multilayer SBM.} \cite{stanley2016clustering} introduced one of the first approaches that integrated a multilayer SBM with a mixture of layers, using a two-step greedy inference method. In the initial step, it infers a  SBM for each layer and groups together SBMs with similar parameters. In the second step, these outcomes serve as the starting point for an iterative procedure that simultaneously identifies $Q$ strata spanning the $V$ layers. In each stratum $s$, the nodes are independently distributed into $K_s$ blocks, leading to $Q$ membership matrices $\{\bZ^1, \cdots, \bZ^Q\}$. 

In pursuit of the same goal, \cite{fan2022alma} develop an alternating minimization algorithm that offers theoretical guarantees for both between-layer and within-layer clustering errors. 
\cite{rebafka2023modelbased} proposes a Bayesian framework for a finite mixture of MLSBM and employs a hierarchical agglomerative algorithm for the clustering process. It initiates with individual singleton clusters and then progressively merges clusters of networks according to 
an ICL criterion also used for model selection. 

Also, \cite{pensky2021clustering} presents a versatile model for diverse multiplex networks, including both MLSBM and MMLSBM. Note that in the latter scenario, they make the assumption that the number of blocks within each group of layers remains consistent, such that $K_s=K$, $\forall s$. They perform a spectral clustering on the layers and then aggregate the resulting block connectivity matrices to determine the between-layer partition of observations. Using this model as a foundation, \cite{noroozi2022sparse} introduces a more efficient resolution technique rooted in sparse subspace clustering~\citep{elhamifar2013sparse}. They demonstrate that this algorithm consistently achieves strong between-layer clustering results. 
In both studies, they provide valuable insights comparing their assumption that $K_s=K$ for all components with the scenario where $K_s$ is considered a known value for each component $s$. 
 This discussion is particularly relevant in the context of methods that are not designed for the task of model selection.

\paragraph{Tensor block models.}  A different strategy for addressing the challenge of late fusion multiview clustering involves tensorial modeling and estimation techniques.  \cite{wang2019multiway} position their research within the context of higher-order tensors. They introduce a least-square estimation method for (sparse) TBM and demonstrate the reliability of block structure recovery as the data tensor's dimension increases by providing consistency guarantees. \cite{han2022exact} suggest employing high-order spectral clustering as an initialization of a high-order Lloyd algorithm. They establish convergence guarantees and statistical optimality under the assumption of sub-Gaussian noise. 

\cite{boutalbi2020tensor} introduce an extension of Latent Block Models to handle tensors. They consider multivariate normal distributions for continuous data and Bernoulli distributions for categorical data, implementing a VEM algorithm for this purpose. 

Finally, \cite{jingCommunityDetectionMixture2020a} employs the Tucker decomposition to conduct alternating regularized low-rank approximations of the tensor. This technique consistently uncovers connections both within and across layers under near-optimal network sparsity conditions. They also establish a consensus clustering of observations by applying a $k$-means algorithm to the local membership decomposition matrix.

\section{Mixture of  Multilayer SBM}

Our model builds on SBM and 
considers two sets of latent variables corresponding respectively to the structure of the observations and the structure of the views. This proposal is at the crossroads of MLSBM, which involves the discovery of a traversing membership matrix of observations spanning all layers, and MMLSBM, which involves uncovering structural patterns within the layers.

\paragraph{Observations.}
We consider the observed data to be a tensor $\bA \in \{0,1\}^{N\times N \times V}$   where $N$ is the number of observations (vertices), and $V$ the number of views.  Each of the $V$ slices of $\bA$ is an adjacency matrix corresponding to a graph $\mathcal{G}^v$. The tensor is thus a stack of adjacency matrices for multiple view graphs $(\mathcal{G}^1,\cdots,\mathcal{G}^V)$ with corresponding vertices. Let denote $(i,j)$ an edge between observations $i$ and $j$, we have by definition  $A_{ijv} = \mathbb I_{((i,j)\in E^v)}$ where $E^v$ is the set of edges of the graph $\mathcal{G}^v$. 


\paragraph{Latent structures.}
Let $\bZ \in  \{0,1\}^{N\times K}$ be the indicator membership matrix of observations, where $K$ is the number of view traversing clusters. We have by definition  $Z_{ik}= \mathbb{I}_{ (i \in k)}$, where $i$ denotes an observation and $k$ is a cluster across the views. 

Let denote $\bW \in  \{0,1\}^{V\times Q}$, the indicator membership matrix for the views where $Q$ is the number of components of the view mixture. We have $W_{vs}= \mathbb{I}_{(  v \in s)}$, where  $v$ is a view and $s$ a cluster of views.

%

\subsection{A mixture of observations through a mixture of views} 

The $V$ views are assumed to be generated by a mixture model of $Q$ components. Each component  $s$  is a SBM.  Each line of  matrix $\bW$  is assumed to follow a multinomial distribution, $\displaystyle{\bW_v \sim \mathcal{M}(1, {\brho} = (\rho_1,\dots, \rho_Q))}$, with   
\begin{equation*}
    \mathbb{P}(\bW \mid {\brho} ) = \prod_{v = 1}^V \prod_{s = 1}^Q  \rho_v^{W_{vs}}.
\end{equation*}

Although we use multiple views with their known cluster structure (SBM), we assume a \textit{traversing structure}  \textit{across all views} described by the latent variable $\bZ$. By leveraging all available sources of information, we aim to achieve a more comprehensive understanding of the data and obtain community structures that are consistent across all views. The individuals are thus assumed to come from a number $K$ of sub-populations. 

Each  latent class vector for the observation follows a multinomial distribution, with  $\displaystyle{\bZ_i \sim \mathcal{M}(1, {\bpi} = (\pi_1,\dots, \pi_K))}$, and 
\begin{equation*}
    \mathbb{P}(\bZ \mid \boldsymbol{\pi}) = \prod_{i = 1}^N \prod_{k = 1}^K  \pi_k^{Z_{ik}}.
\end{equation*}


Each observation $\displaystyle{A_{ijv}}$ conditionally to the latent structure $\bZ$ follows a Bernoulli distribution: $\displaystyle{A_{ijv} \mid Z_{ik} = 1, Z_{jl} = 1 \sim \mathcal{B}(\alpha_{kls})}$.  The probability of all observations given the latent variables $\bZ$, $\bW$ and a vector of parameters $\bTheta$, is  thus
%
%
\begin{align*} 
\mathbb{P}(\bA \mid \bZ, \bW, 
{\bTheta})  & = \prod_{\substack{i=1,\\ i<j}}^N  \prod_{\substack{k=1\\l=1}}^K   \prod_{v=1}^V  \prod_{s=1}^Q \left(  \alpha_{kls}^{A_{ijv}}  \left(1- \alpha_{kls} \right)^{1-A_{ijv}}   \right)^{Z_{ik}Z_{jl}W_{vs}}.
\end{align*}

\subsection{Identifiability}
\label{sec:identifiability}

\begin{theorem}
Let $N \geq \max(2K,4Q)$ and $V \geq 2K$. Assume that for any $k,l \in \{1,\dots,K \}$ and every $s \in \{1,\dots,Q \}$, the coordinates of $\boldsymbol{\pi}^T  \boldsymbol{\alpha}_{k..}  \boldsymbol{\rho}$ are all different, $(\boldsymbol{\pi}^T  \boldsymbol{\alpha}_{..s} \boldsymbol{\pi})_{s=1:Q}$ are distinct, and each $(\boldsymbol{\alpha}_{kl.} \boldsymbol{\rho})_{k,l=1:K}$  differs. Then, the mimi-SBM parameters $\boldsymbol{\Theta} =\left( \boldsymbol{\pi},\boldsymbol{\rho},\boldsymbol{\alpha}\right)$ are identifiable.
\end{theorem}
\begin{proof}
The proof of this theorem is given in Appendix \ref{sec:proof_identifiability}.
\end{proof}

\subsection{Bayesian modeling} 

Bayesian modeling provides a natural framework for incorporating prior knowledge which can improve the accuracy of the estimated block structure, mainly when the available data is limited or noisy. 


In this context, we follow \cite{latoucheVariationalBayesianInference2010} and define the chosen conjugate distributions both for the proportion of the mixture and the proportions of the blocks. Conjugate priors lead to closed-form posterior distributions. 
\begin{align}
    \mathbb{P}(\boldsymbol{\pi} \mid  \boldsymbol{\beta}^0 &=(\beta^0_1,\dots,\beta^0_K)  ) = \dir(\boldsymbol{\pi};\boldsymbol{\beta}^0), \\
    \mathbb{P}(\boldsymbol{\rho} \mid  \boldsymbol{\theta}^0 &=(\theta^0_1,\dots,\theta^0_Q)  ) = \dir(\boldsymbol{\rho};\boldsymbol{\theta}^0),
\end{align}
where $\dir(.)$ stands for the Dirichlet distribution. 
\begin{equation}
    \mathbb{P}(\boldsymbol{\alpha} \mid  \boldsymbol{\eta}^0 = (\eta^0_{kls}), \boldsymbol{\xi}^0=(\xi^0_{kls})  ) = \prod_{k,k<l} \prod_{s} \mathrm{Beta}(\alpha_{kls}; \eta^0_{kls},\xi^0_{kls}).
\end{equation}

\begin{figure}[!ht]
    \centering
    \includegraphics[scale=0.28]{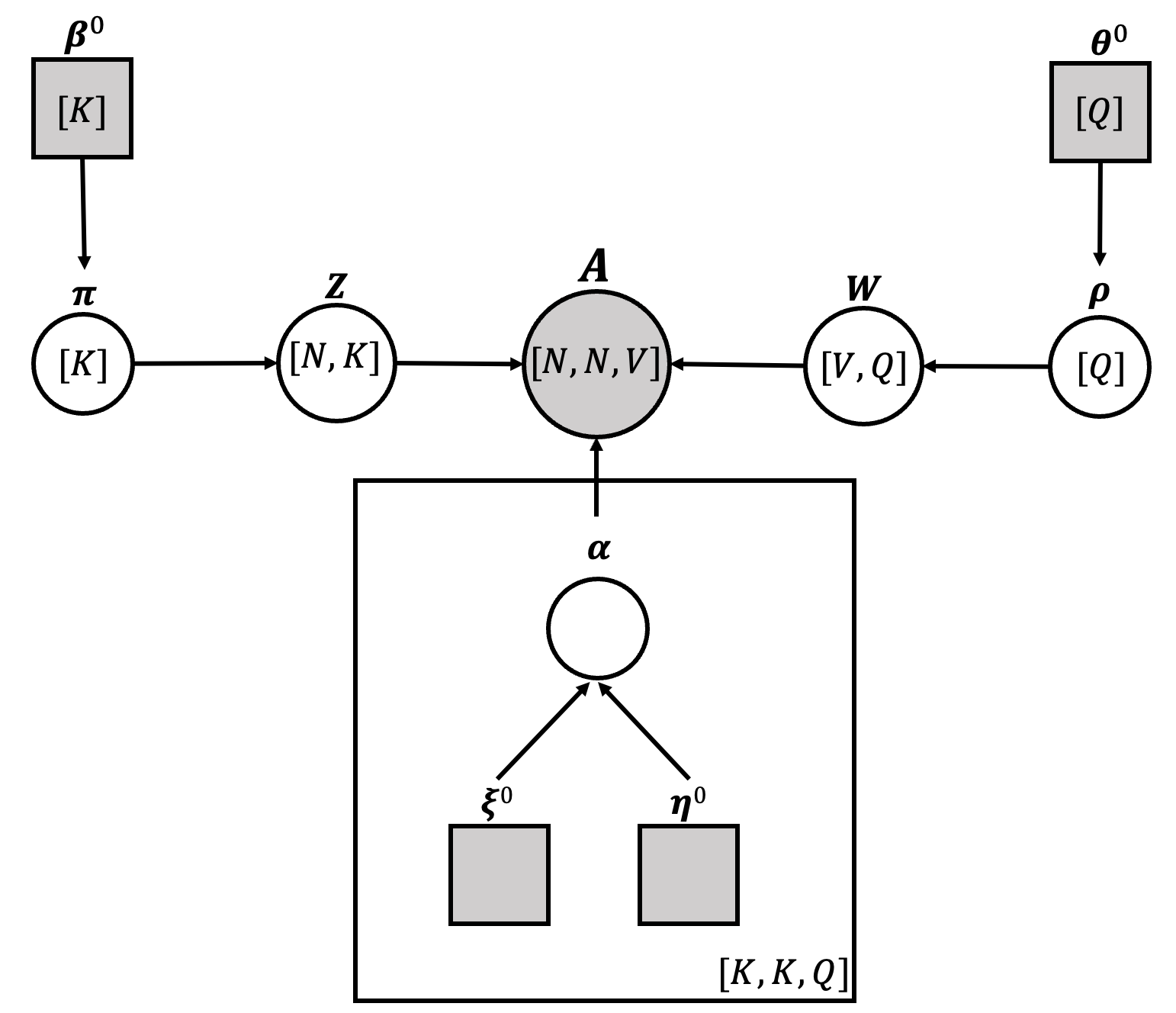}
    \caption{Illustration of mimi-SBM with bayesian notations}
\end{figure}


The parameters 
${{\boldsymbol{\beta}^0,\boldsymbol{\theta}^0},\boldsymbol{\eta}^0,\boldsymbol{\xi}^0 }$ are chosen according to Jeffreys priors which are often considered non-informative or weakly informative. They do not introduce strong prior assumptions or biases into the analysis. 

For the Dirichlet distribution, a suitable choice for $\beta^0_k$ and $\theta^0_s$ is setting them both to 1/2, which directly corresponds to an objective Jeffreys prior distribution. Similarly, for the Beta distribution, $\eta^0_{kls}$ and $\xi^0_{kls}$ can be chosen as 1/2 for all appropriate indices k, l, and s.

\section{Variational Bayes Expectation Maximisation for mimi-SBM}

Computing the marginal likelihood is a challenging problem in Stochastic Block Models:
\begin{equation}
\begin{aligned} 
\p\left( \bA \right) = \sum_\bZ \sum_\bW \int \int \int \p\left( \bA,\bZ,\bW, \boldsymbol{\alpha},\boldsymbol{\pi},\boldsymbol{\rho} \right)\  d\boldsymbol{\alpha} \ d\boldsymbol{\pi} \ d\boldsymbol{\rho}.
\end{aligned}
\end{equation}
The computation of integrals in the formula for this marginal likelihood presents analytical challenges or becomes unfeasible, while sums over $\bZ$ and $\bW$ become impractical when the number of parameters or observations is substantial.

Approximation of complex posterior distributions is usually performed either by sampling (Markov Chain Monte Carlo or related approaches) or by  Variational Bayes inference introduced by ~\cite{attias1999variational}.
Variational Bayes Expectation Maximization algorithm offers several advantages such as reduced computation time and the ability to work on larger databases.

\subsection{Evidence Lower Bound}

Variational inference  is computationally efficient and scalable to large datasets and work especially well for SBM models. It formulates the problem as an optimization task, where the goal is to find the best approximation to the true posterior distribution. This optimization framework allows for efficient computation of the variational parameters by maximizing a lower bound on the log-likelihood, known as the evidence lower bound (ELBO). Optimization techniques like stochastic gradient descent (SGD) or Expectation-Maximisation algorithm can be employed to find the optimal variational parameters.

The distribution $\p(\bZ, \bW| \bA,\boldsymbol{\alpha},\boldsymbol{\pi},\boldsymbol{\rho})$ is intractable when taking SBM into account, hence we approximate the entire distribution $\p(\bZ,\bW,\boldsymbol{\alpha},\boldsymbol{\pi},\boldsymbol{\rho} \mid \bA)$. Given a variational distribution $q$ over $\{\bZ,\bW,\boldsymbol{\alpha},\boldsymbol{\pi},\boldsymbol{\rho}\}$, we can decompose the marginal log-likelihood into Evidence Lower BOund (ELBO) part and KL-divergence between variational and posterior distribution :
\begin{eqnarray*}
\log P(\bA) &= &  \mathbb E_{q}\left[ \log \frac{P\left( \bA,\bZ,\bW, \boldsymbol{\alpha},\boldsymbol{\pi},\boldsymbol{\rho} \right)}{P(\bZ,\bW, \boldsymbol{\alpha},\boldsymbol{\pi},\boldsymbol{\rho}  | \bA)}\right]\\
&= & \underbrace{\mathbb E_{q}\left[ \log \frac{P\left( \bA,\bZ,\bW, \boldsymbol{\alpha},\boldsymbol{\pi},\boldsymbol{\rho} \right)}{q(\bZ,\bW,\boldsymbol{\alpha},\boldsymbol{\pi},\boldsymbol{\rho})}\right]}_{ELBO = \mathcal{L}\left(q(.) \right) } + \underbrace{\mathbb E_{q}\left[ \log \frac{q(\bZ,\bW,\boldsymbol{\alpha},\boldsymbol{\pi},\boldsymbol{\rho}) }{P(\bZ,\bW, \boldsymbol{\alpha},\boldsymbol{\pi},\boldsymbol{\rho}  \mid \bA)}\right]}_{\KL \left( q(\cdot) \| \p(\cdot\mid \bA)\right)}
\end{eqnarray*}
where
$\KL \left(q(.)  \mid \p(\cdot\mid \bA)\right) = -\mathbb E_q[\log{\frac{p}{q}}]\geq -\log \mathbb E_q[{\frac{p}{q}}]\geq 0$
from Jensen inequality.

The ELBO is given by
\begin{equation}
\label{ELBO}
\mathcal{L}\left(q(\cdot)\right) = \sum_{\bZ,\bW} \int \int \int q(\bZ,\bW,\boldsymbol{\alpha},\boldsymbol{\pi},\boldsymbol{\rho}) \log\dfrac{p\left( \bA,\bZ,\bW, \boldsymbol{\alpha},\boldsymbol{\pi},\boldsymbol{\rho} \right)}{q(\bZ,\bW,\boldsymbol{\alpha},\boldsymbol{\pi},\boldsymbol{\rho})} \  d\boldsymbol{\alpha} \ d\boldsymbol{\pi} \ d\boldsymbol{\rho}.
\end{equation}



The variational distribution is typically selected from an easier-to-handle family of distributions, such as the exponential family. The variational distribution's parameters are then adjusted to reduce the KL divergence to the posterior distribution.
If $q(.)$ is exactly $p(.|\bA)$ the $\KL$ term is equal to $0$, and the ELBO is maximized.



We assume a mean-field approximation for $q(\cdot)$: 
\begin{equation}
\begin{aligned} 
\label{meanfield}
q(\bZ,\bW,\boldsymbol{\alpha},\boldsymbol{\pi},\boldsymbol{\rho}) &=  \prod_{i=1}^N q(\bZ_i) \  \prod_{v=1}^V q(\bW_v) \ \prod_{s=1}^Q \prod_{k,k\leq l}^K q(\alpha_{kls}) \  q(\boldsymbol{\pi}) \  q(\boldsymbol{\rho})\\
&=  \dir(\boldsymbol{\pi};\boldsymbol{\beta}) \  \dir(\boldsymbol{\rho};\boldsymbol{\theta}) \prod_{i=1}^N \M(\bZ_i;1,\boldsymbol{\tau}_i) \  \prod_{v=1}^V \M(\bW_i;1,\boldsymbol{\nu}_v) \\ 
&\quad \prod_{s=1}^Q \prod_{k,k\leq l}^K  \mathrm{Beta}(\alpha_{kls}; \eta_{kls},\xi_{kls}),
\end{aligned}
\end{equation}
where $\tau_{ik}$ (resp. $\nu_{vs}$) are variational  parameters indicating the probability that individual $i$ (resp. a view $v$) belongs to cluster $k$ (resp. component $s$).


According to (\ref{ELBO}), given a distribution $q(.)$, the ELBO is given by
\begin{equation}
\label{equation:ILvb}
\begin{aligned}
\mathcal{L}\left(q(.)\right) &=  \log \left\{\frac{\Gamma\left(\sum_{k=1}^K \beta_k^0\right) \prod_{k=1}^K \Gamma\left(\beta_k\right)}{\Gamma\left(\sum_{k=1}^K \beta_k\right) \prod_{k=1}^K \Gamma\left(\beta_k^0\right)}\right\}
+\log \left\{\frac{\Gamma\left(\sum_{s=1}^Q \theta_s^0\right) \prod_{s=1}^Q \Gamma\left(\theta_s\right)}{\Gamma\left(\sum_{s=1}^Q \theta_s\right) \prod_{s=1}^Q \Gamma\left(\theta_s^0\right)}\right\}\\
&\quad +\sum_{k \leq l}^K  \sum_{s=1}^Q  \log \left\{\frac{\Gamma\left(\eta_{k l s}^0+\xi_{k l s }^0\right) \Gamma\left(\eta_{k  l s}\right) \Gamma\left(\xi_{k l s}\right)} {\Gamma\left(\eta_{k l s} +\xi_{k  l s} \right) \Gamma\left(\eta_{k l s}^0\right) \Gamma\left(\xi_{k l s}^0\right)}\right\} \\
&\quad -\sum_{i}^N \sum_{k}^K  \tau_{i k} \log \tau_{i k} \ - \sum_{v}^V \sum_{s}^Q  \nu_{v s} \log \nu_{v s},
\end{aligned}
\end{equation}
where $\Gamma(.)$ is the Gamma function.
This function is also called Integrated Likelihood variational Bayes (ILvb, \cite{latoucheVariationalBayesianInference2010}), since it can be used for model selection.

\subsection{Lower bound optimization}

We consider a Variational Bayes EM algorithm for estimating the parameters. The algorithm starts by initializing the model parameters and then iteratively performs two steps: the Variational Bayes Expectation step (VBE-step) and the Maximization step (M-step) (See Algorithm \ref{algo}). 

In the VBE-step, the variational distributions are optimized over latent variables: 
\(q(\bZ_i) \ \forall i \in \{1,\dots,N\}$ and $q(\bW_v) \ \forall v \in \{1,\dots,V\}\)  to approximate the true posterior distribution.

In the M-step, the parameters of the model are updated to maximize a lower bound on the log-likelihood, with respect to parameters computed in the VBE-step: $\boldsymbol{\beta}$, $\boldsymbol{\theta}$, $\boldsymbol{\eta}$, and $\boldsymbol{\xi}$.






There exist multiple techniques for initializing the EM algorithm. One prevalent approach involves using random initial values, where the model parameters are assigned random values drawn from a designated distribution. Nevertheless, this method may lack reliability and fail to provide satisfactory starting values for the algorithm.
According to the initialization method proposed in \cite{stanley2016clustering}, the parameters $(\tau_{ik})$ and $(\nu_{vs})$ are initialized based on the outcomes of a stochastic block model (SBM) applied separately to each view. The objective is to capture the overall structure of the data from each view using SBM, combine this information using K-means clustering, and subsequently refine the obtained results using our model.





\begin{algorithm}
\caption{mimi-SBM \label{algo}}

\begin{algorithmic} 
\REQUIRE Tensor of adjacency matrices \textbf{A},  Number of clusters $K$,  Number of components of the views $Q$, precision \textit{eps}.
\STATE Initialization : $\tau_{ik}^{(old)}$ and $\nu_{ik}^{(old)}$ 
\WHILE{$\| \mathcal{L}\left(q^{new}(.) \right) - \mathcal{L}\left(q^{old}(.)\right) \| < \textit{eps}  $}
\STATE \textbf{VBE-step}
\STATE Compute $\tau_{ik}^{(new)} \ \forall i \in \{1,\dots,N\}$ and $\forall k \in \{1,\dots,K\}$
\STATE Compute $\nu_{vs}^{(new)} \ \forall v \in \{1,\dots,V\}$ and $\forall s \in \{1,\dots,Q\}$
\STATE \textbf{M-step}
\STATE Optimize  $\boldsymbol{\beta}$, $\boldsymbol{\theta}$, $\boldsymbol{\eta}$, $\boldsymbol{\xi}$ with respect to $(\tau_{ik}^{(new)})$ and $(\nu_{vq}^{(new)})$
\STATE \textbf{ELBO}
\STATE Compute $\mathcal{L}\left(q^{new}(.)\right)$
\ENDWHILE
\end{algorithmic}
\end{algorithm}

 
\section{Model selection}

In the context of clustering, model selection often refers to the process of determining the ideal number of clusters for a given dataset. In our situation, the key decision lies in selecting appropriate values for $K$ and $Q$ to strike a balance between data attachment and model complexity. To achieve this, several criteria based on penalized log-likelihood can be employed, such as the Akaike Information Criterion (AIC) \citep{akaike1998information}, Bayesian Information Criterion (BIC) \citep{schwarz1978estimating} and more recently the Integrated Completed Likelihood (ICL) \citep{biernacki2000assessing}. We specifically consider the ICL criterion and its associated penalties as they frequently yield good trade-offs in the selection of mixture models \citep{biernacki2010exact}.


The ICL is based on the log-likelihood integrated over the parameters of the complete data. Furthermore, if we assume that parameters of component-connection probability $\boldsymbol{\alpha}$, parameters for mixture of communities $\boldsymbol{\pi} $ and parameters for views mixture $\boldsymbol{\rho}$ are independent, we have:
\begin{equation}
\begin{aligned}
    \operatorname{ICL}(\textbf{A},K,Q)  &= \log \mathbb{P}(\textbf{A},\textbf{Z},\textbf{W} \mid K,Q) \\
    &=  \log \int_{\boldsymbol{\alpha}} \p (\textbf{A} \mid \textbf{Z},\textbf{W}, \boldsymbol{\alpha}) \p(\boldsymbol{\alpha}) d\boldsymbol{\alpha} \\
    &\quad + \log \int_{\boldsymbol{\pi}} \p (\textbf{Z} \mid  \boldsymbol{\pi}) \p(\boldsymbol{\pi}) d\boldsymbol{\pi} \\
    &\quad + \log \int_{\boldsymbol{\rho}} \p (\textbf{W} \mid  \boldsymbol{\rho}) \p(\boldsymbol{\rho}) d\boldsymbol{\rho}.
    \end{aligned}
\end{equation}

In our variational framework, $\mathbf{Z}$ and $\mathbf{W}$ must be estimated. $\hat{\textbf{Z}}$ (resp. $\hat{\textbf{W}}$)  can be chosen as the variational parameters $\boldsymbol{\tau}$ (resp. $\boldsymbol{\nu}$) directly or by a Maximum a Posteriori (MAP):
\begin{equation*}
    \hat{\textbf{Z}}_i = \operatornamewithlimits{argmax}_{k \in 1:K} \tau_{ik}.
\end{equation*}

By using approximations, such as Stirling's approximation formula on $\p(\boldsymbol{\pi})$ and $\p(\boldsymbol{\rho})$ and the Laplace asymptotic approximation on $\p(\boldsymbol{\alpha})$, we can define an \textit{approximate ICL}:
\begin{equation}
\begin{aligned}
    \operatorname{ICL}(\textbf{A},K,Q) 
    &\approx \log \mathbb{P}(\textbf{A},\hat{\textbf{Z}},\hat{\textbf{W}} \mid K,Q) - \operatorname{pen}(K,Q) \\
    &\approx \mathcal{L}\left(q(.)\right) -  \operatorname{pen}(K,Q),
\end{aligned}
\end{equation}
where 
\begin{equation*}
    \operatorname{pen}(K,Q) =  
    \dfrac{1}{2}\dfrac{K(K+1)}{2} Q \log(V\dfrac{N(N-1)}{2}) + 
    \dfrac{1}{2}(K-1)\log(N) +  \dfrac{1}{2}(Q-1)\log(V).
\end{equation*}

The penalization in this \textit{approximate ICL} is composed of a part depending on the number of parameters of component-connection probability tensor $\boldsymbol{\alpha}$ and the number of vertices taken into account,  and a part that takes into account the number of degree of freedom in mixture parameters and the number of variables related to them.
Recall that our model is based on undirected (symmetric) adjacency matrices, so we only consider the upper triangular matrices (without the diagonal).

However, in the Bayesian framework with conjugate priors, it is possible to define an exact ICL~\citep{come2015model}.
Moreover, it can be obtained from the previously defined ILvb~\eqref{equation:ILvb} when the entropy of the latent variables is zero and the Expectation-Maximization algorithm is a Classification EM~\cite[CEM,][]{celeux1992classification}. In other words, variational parameters are equal to $1$ if it is the MAP and $0$ otherwise. 
Thus, this \textit{exact ICL} can be defined as:
\begin{equation}
\begin{aligned}
\operatorname{ICL_{exact}}(\textbf{A},K,Q)  &=  \log \left\{\frac{\Gamma\left(\sum_{k=1}^K \beta_k^0\right) \prod_{k=1}^K \Gamma\left(\beta_k\right)}{\Gamma\left(\sum_{k=1}^K \beta_k\right) \prod_{k=1}^K \Gamma\left(\beta_k^0\right)}\right\}\\
&\quad +\log \left\{\frac{\Gamma\left(\sum_{s=1}^Q \theta_s^0\right) \prod_{s=1}^Q \Gamma\left(\theta_s\right)}{\Gamma\left(\sum_{s=1}^Q \theta_s\right) \prod_{s=1}^Q \Gamma\left(\theta_s^0\right)}\right\}\\
&\quad +\sum_{k \leq l}^K  \sum_{s=1}^Q  \log \left\{\frac{\Gamma\left(\eta_{k l s}^0+\xi_{k l s }^0\right) \Gamma\left(\eta_{k  l s}\right) \Gamma\left(\xi_{k l s}\right)} {\Gamma\left(\eta_{k l s} +\xi_{k  l s} \right) \Gamma\left(\eta_{k l s}^0\right) \Gamma\left(\xi_{k l s}^0\right)}\right\}.
\end{aligned}
\end{equation}

Instead of using a CEM, it is possible to use the v
Variational parameters directly, and to derive a \textit{variational ICL} from the previous criterion.
Figure~\ref{ICL fig} summarizes the links between the various selection criteria and clearly shows that in a particular context \textit{exact ICL}, \textit{Variational ICL}, and \textit{ILvb} criteria are identical.

\begin{figure}[!h]
    \centering
    \includegraphics[scale=0.33]{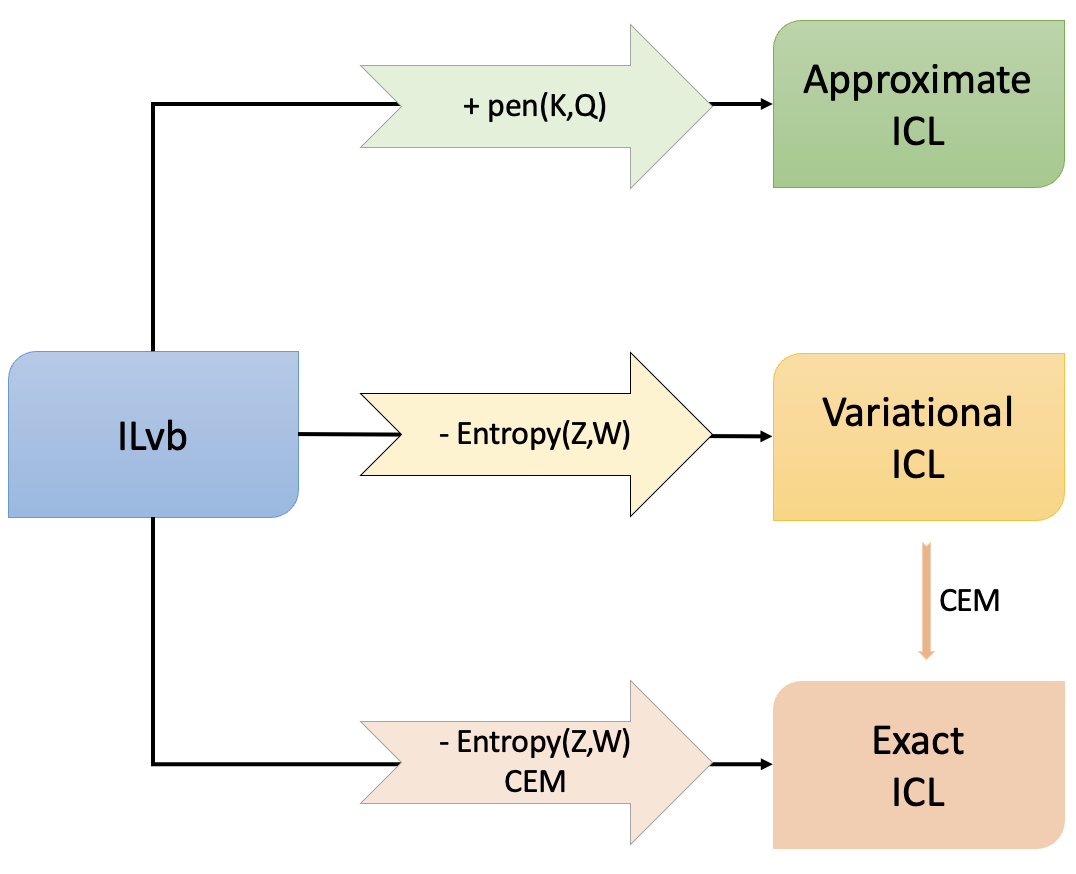}
    \caption{Diagram of links between different model selection criteria.  }
    \label{ICL fig}
\end{figure} 


\section{Experiments}


\subsection{Simulation study}
 

To ensure that \textit{mimi-SBM} behaves consistently, we have
developed a simulation scheme, as depicted in Figure~\ref{fig:schema_XP}. The simulated data have been designed to reflect the
complexity found in real-world data clustering challenges. In
particular, this scheme allows for diverse clustering patterns across
different view components. Also, it includes the possibility of
controlling clustering errors, with observations being inaccurately
assigned to an incorrect  group, implying inconsistencies in the
adjacency matrices.




\paragraph{Simulated data.} Artificial adjacency matrices are generated from observations and views. We aim to establish a link between the simulated adjacency matrices and the final clustering that most accurately represents a problem of meta-clustering.

Various parameter values are tested for  $N$ (observations), $V$ (views), $K$ (clusters), and $Q$ (components) in different scenarios. Besides, $\boldsymbol{\pi},\boldsymbol{\rho}$ correspond to an equiprobability of belonging to a cluster or component, thus $\displaystyle \{\pi_k\}_{k=1}^{K}=1/K$ and $\displaystyle \{\rho_s\}_{s=1}^{Q}=1/Q$ (Figure~\ref{fig:schema_XP}, \textit{Parameters}).


First of all, it is assumed that the number of real clusters ($K$) will always be equal or higher than the number of clusters coming from each component. Also, each component has a precise number of clusters ($K^q$), and each view belonging to this component will have this number of clusters $\displaystyle K^q \sim \mathcal{U}(\{2,\dots,K\})$, where $\mathcal{U}$ is the discrete uniform distribution (Figure~\ref{fig:schema_XP}, \textit{Clusters per component}). 


Now, for each component, we randomly associate a link between the final consensus clusters $\bZ$ and clusters coming from the component $\bZ^q$, ensuring that no component cluster remains empty (Figure~\ref{fig:schema_XP}, \textit{Links between clusters}).

Eventually, 
for each pair of nodes $(i,j)$ and each layer $v$, an edge is generated with probability $\alpha_{Z_i Z_j W_v}$, leading to set the corresponding entry in the multilayer adjacency tensor $\mathbf{A}$ to $1$ or $0$  (Figure~\ref{fig:schema_XP}, \textit{Generation of edges}).

\begin{figure}[!htt]
     \centering
     \begin{subfigure}[b]{\textwidth}
         \centering
         \includegraphics[width=\textwidth]{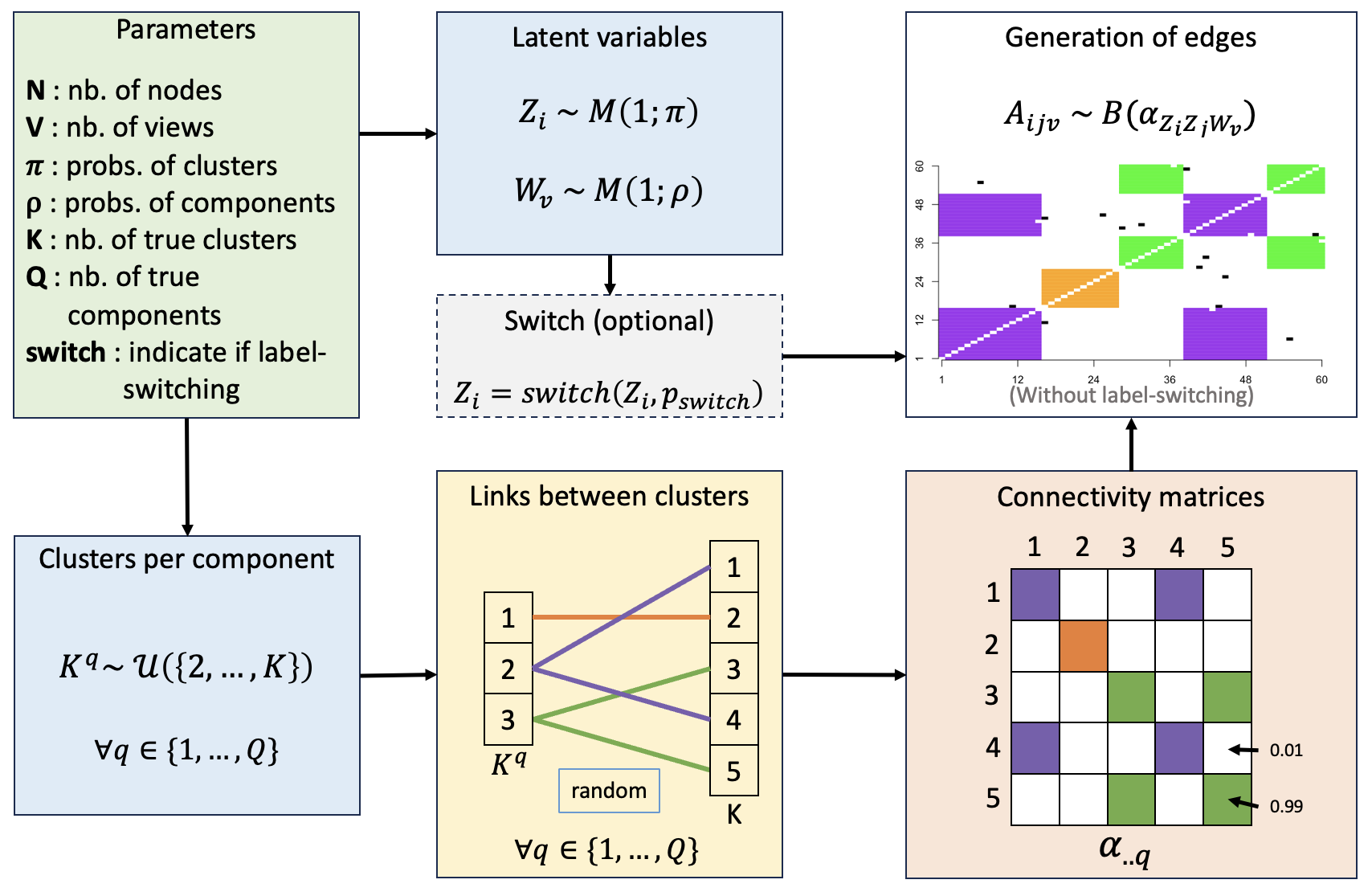}
     \end{subfigure}
     \caption{
     Diagram of the simulation process. Example of adjacency matrices resulting from 
     mixing and traversing clusters across views, with potentially label-switching, for $K = 5$. For each view component, a number of clusters $K^q$ is randomly drawn (discrete uniform distribution). 
     Each cluster in the $q^{\textrm{th}}$ component is then linked to certain clusters in the final partition. 
    For this component, $K^q=3$, and final clusters $1$ (respectively $3$) and $4$ (resp. $5$) are merged into cluster $2$ (resp. $3$) of the component, and the first cluster of the component corresponds perfectly to the final consensus cluster $2$.
     Afterwards, these links are represented by a very strong connectivity within the $\balpha_{..q}$ matrix ($p = 0.99$) and a very weak one ($p=0.01$) for the others.
     \label{fig:schema_XP}}
\end{figure}

\medskip
The simulated data are used to assess three aspects:
\begin{enumerate}
    \item \textbf{\textit{Model selection}.} In Section~\ref{sec:model_selection_XP}, various criteria are examined to 
    to recover the true parameters $K$ and $Q$.
    \item \textbf{\textit{Clustering ability}.} In Sections~\ref{sec:SOTA_XP} and~\ref{sec:SOTA_noisy_XP}, \textit{mimi-SBM} is evaluated against other state-of-the-art techniques regarding the clustering of observations and the clustering of views using ARI scores (see below). 
    \item \textbf{\textit{Robustness}.} In Section~\ref{sec:robustness_XP}, we investigate further the model ability to handle noisy configurations inherent in real-world clustering problems. 
\end{enumerate}

The code for the simulations is available on the CRAN, and on GitHub in the repository \textit{mimiSBM}.~\footnote{\url{https://github.com/Kdesantiago/mimiSBM}.}

\paragraph{Adjusted Rand Index.}

The Adjusted Rand Index \cite[ARI,][]{hubert1985comparing} quantifies the similarity between two partitions. In the simulation to follow, it quantifies the similarity between the prediction by our clustering models and the true partition. It corresponds to the proportion of pairs $(i, j)$ of observations jointly grouped or separated. The more similar the partitions, the closer the ARI is to $1$.


\subsubsection{Comparing model selection criteria}
\label{sec:model_selection_XP}
In this section, our goal is to undertake a comparative analysis of model selection criteria to determine the optimal choice of criterion. 

The use of simulations gives us a complete control over the hyperparameters that generated the data. 
To do this, we generated $50$ different datasets with hyperparameters $K=10$ and $Q=5$. 
The model selected for each criterion is the one that maximizes its value.

\begin{figure}[!htt]
     \centering
     \begin{subfigure}[b]{0.49\textwidth}
         \centering
         \includegraphics[width=\textwidth]{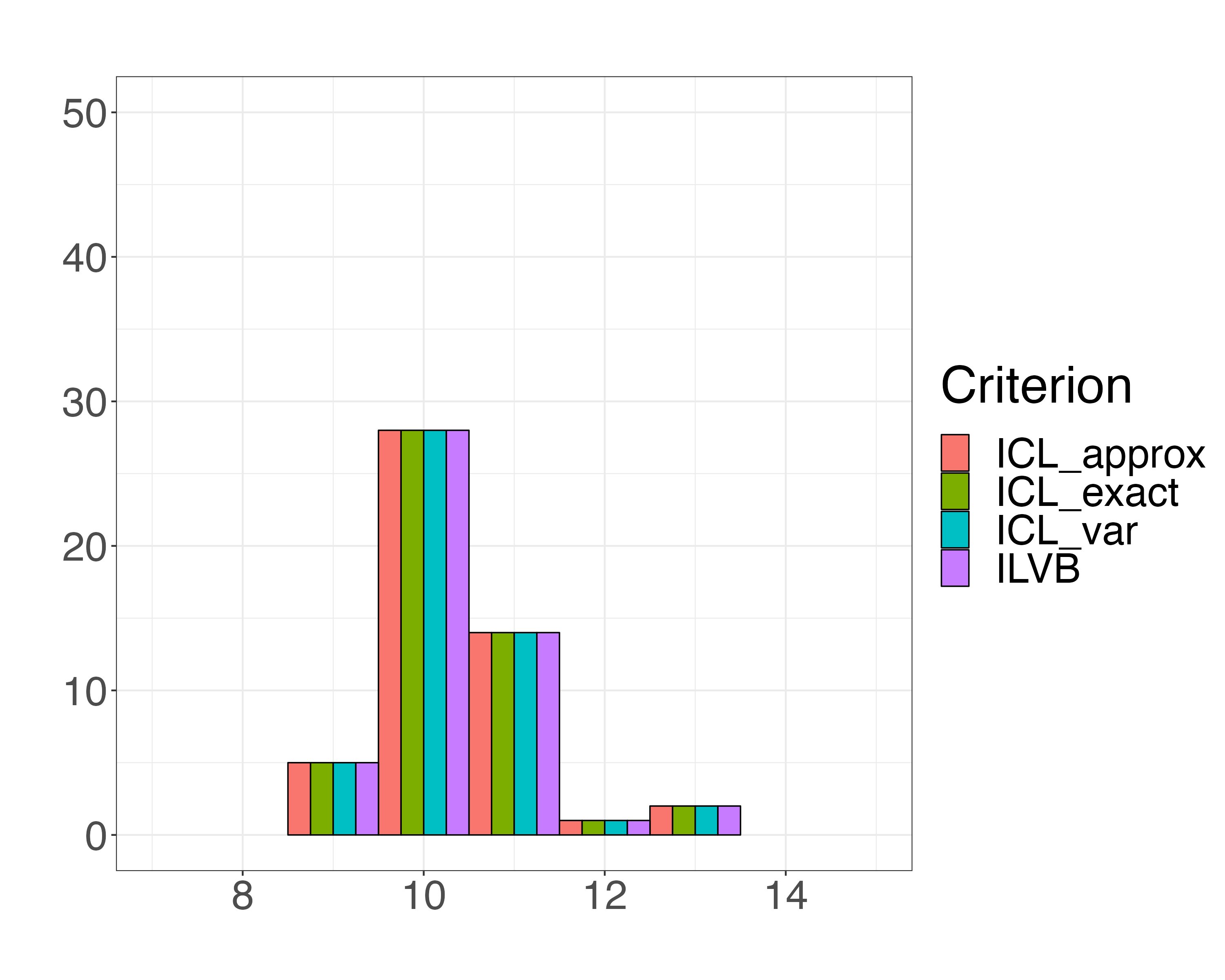}
         \caption{Model selection on $K$, with $Q$ fixed}
        \label{fig: Choix model a}
     \end{subfigure}
     \hfill
     \begin{subfigure}[b]{0.49\textwidth}
         \centering
         \includegraphics[width=\textwidth]{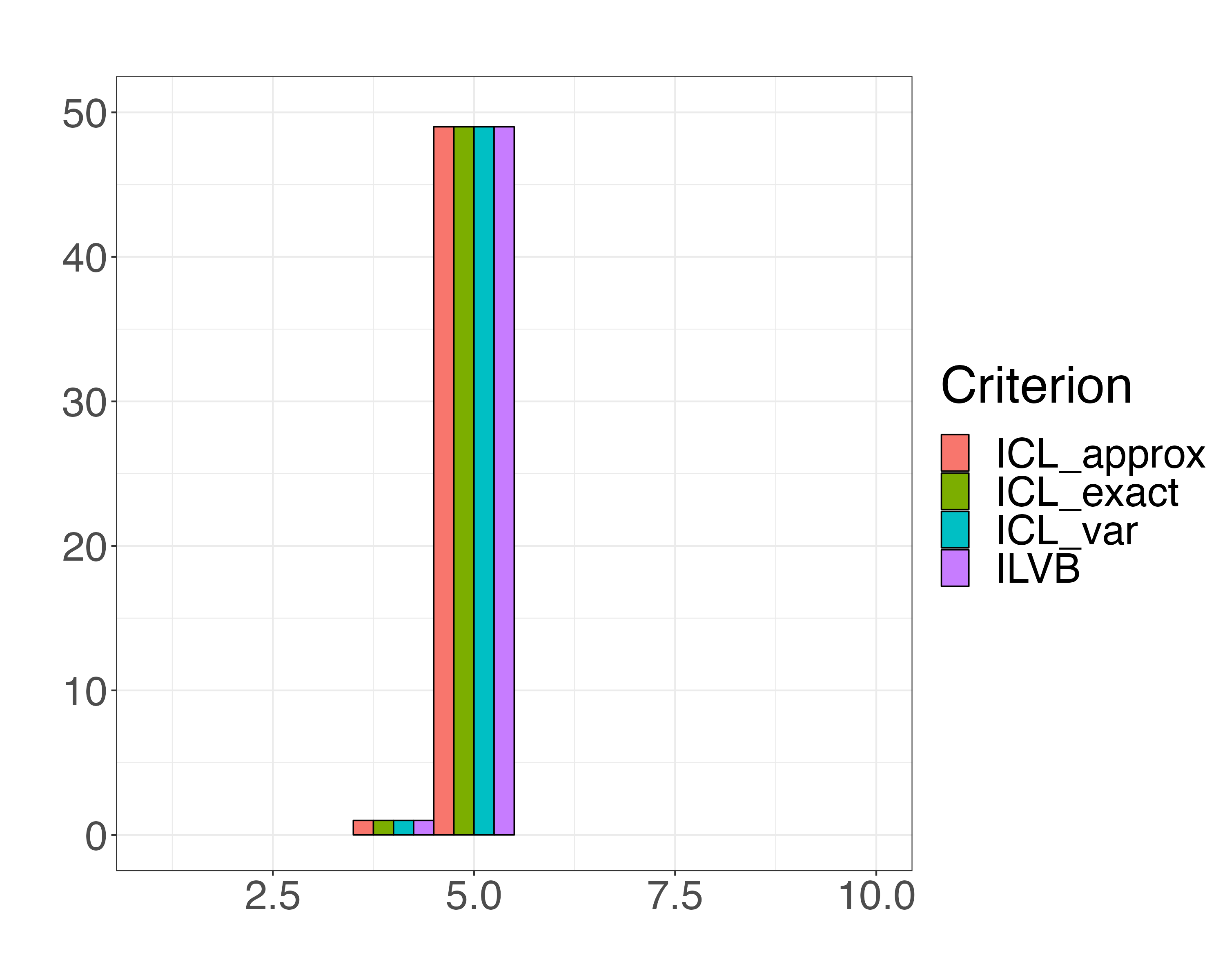}
         \caption{Model selection on $Q$, with $K$ fixed}
        \label{fig: Choix model b}
     \end{subfigure}
     \hfill
     \begin{subfigure}[b]{0.49\textwidth}
         \centering
         \includegraphics[width=\textwidth]{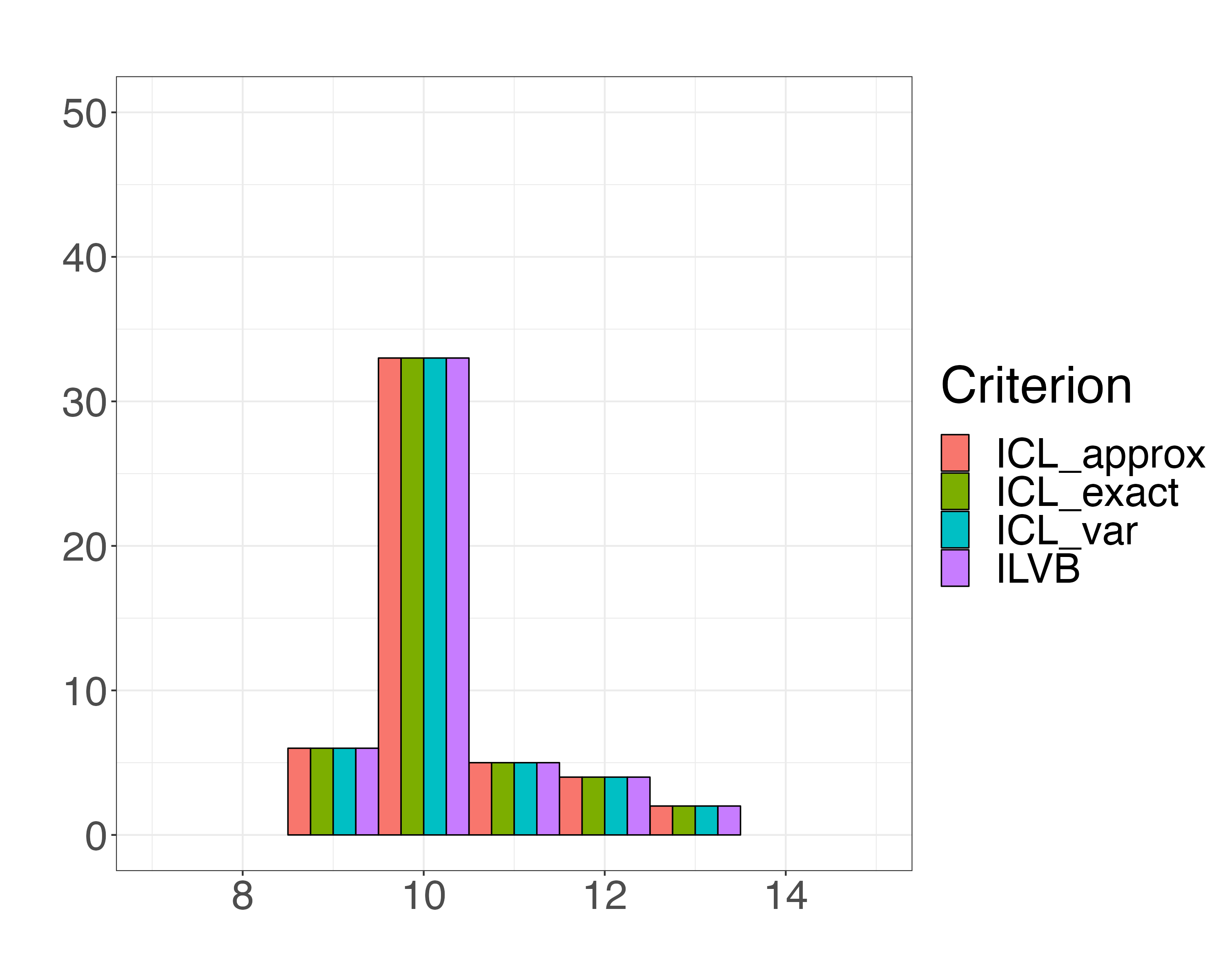}
         \caption{Model selection on $K$, with $Q$ free}
            \label{fig: Choix model c}
     \end{subfigure}
     \hfill
      \begin{subfigure}[b]{0.49\textwidth}
         \centering
         \includegraphics[width=\textwidth]{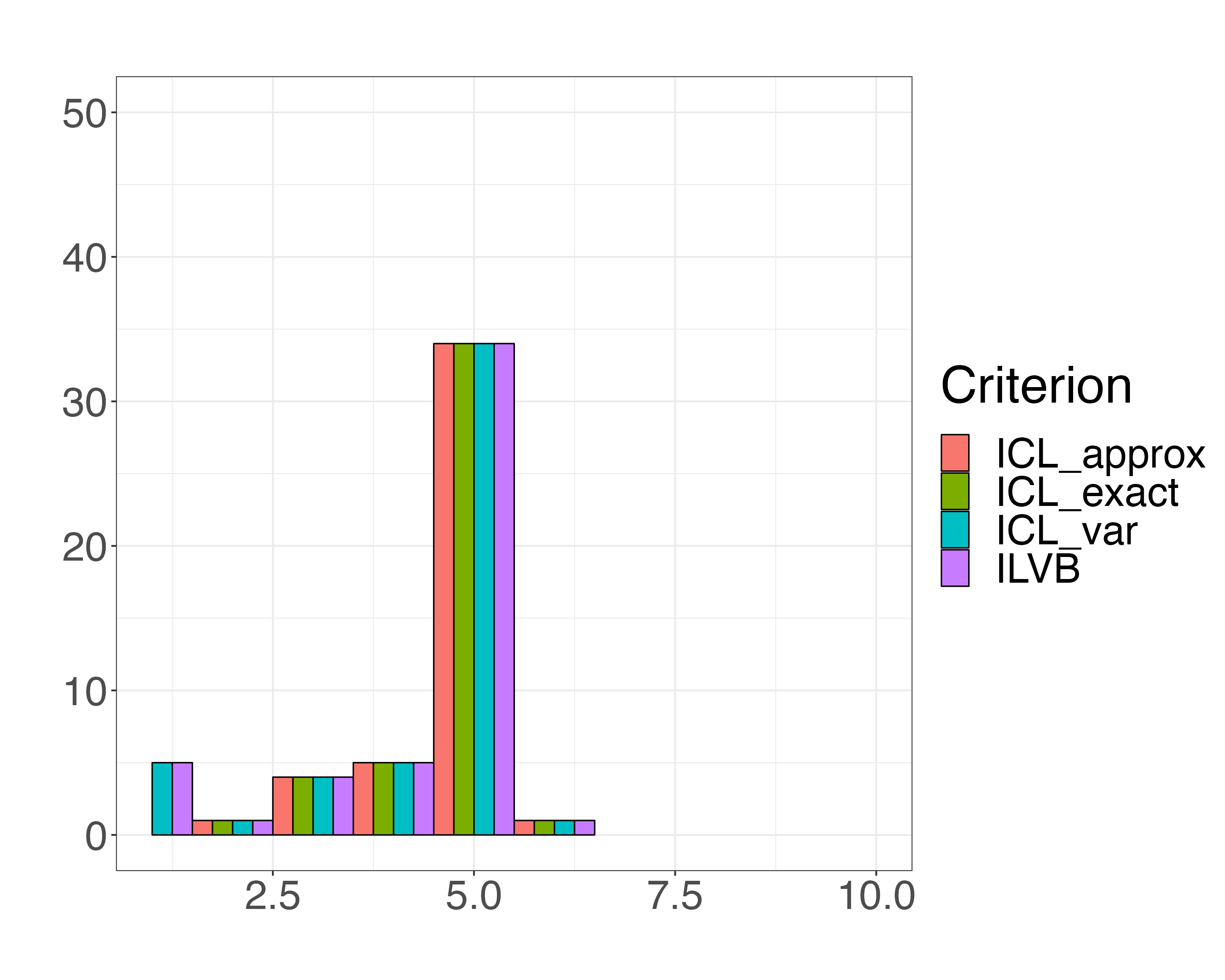}
         \caption{Model selection on $Q$, with $K$ free}
        \label{fig: Choix model d}
     \end{subfigure}
        \caption{Bar plots of model selection criteria on $50$ simulations with $10$ true clusters for observations and $5$ true components for views. 
        Figure (a) (resp. (b)) indicates the number of times the $K$ (resp. $Q$) value selected while the other parameters is set to the true value. Figures (c) and (d) show the same information when hyperparameters are optimized at the same time.}
        \label{fig: Choix model}
\end{figure}

Simulation results in Figure~\ref{fig: Choix model}, clearly show that, in all scenarios, each criterion delivers consistent and comparable performance. Without exception, the criteria consistently produce the same selection of clusters and number of view components.

In Figure~\ref{fig: Choix model a}, only the hyperparameter $K$ varies. This parameter was mostly well estimated because, during model selection, the true number of clusters was typically identified in the majority of cases. However, it was crucial to note that in the context of individual clustering, the criteria tended to overestimate the number of clusters.

In Figure~\ref{fig: Choix model b},  the number of components parameter $Q$ is variable, while $K$ remains constant.
In the majority of scenarios, it was observed that the number of
components was accurately estimated. Furthermore, when a fixed
parameter for clustering was considered, the task inherently became
more tractable due to the use of abundant information for the estimation of view components.

In Figures~\ref{fig: Choix model c} and~\ref{fig: Choix model d}, the selection of hyperparameters is aligned with fixed-parameter results. The criteria consistently demonstrate an aptitude for identifying the optimal cluster and view component quantities. Nonetheless, akin to previous instances, the model occasionally exhibits errors in hyperparameter estimation, often underestimating the number of components while overestimating the number of classes.


\subsubsection{Simulations without label-switching}
\label{sec:SOTA_XP}

\paragraph{Comparison of clustering.} In Figure~\ref{fig: competitors clustering}, it has been observed that \textit{mini-SBM} achieved the best clustering results for each considered experimental configuration. Indeed,  \textit{mini-SBM} recorded the highest ARI score for all data sizes, number of clusters and sources. 
On the other hand, the \textit{M3C} model improved as the number of views, clusters and sources increased. In contrast, the \textit{TWIST} model showed poorer performances as the clustering problem became more complex, suggesting that this model may be less suitable for difficult clustering problems.

\begin{figure}[!ht]
     \centering
     \begin{subfigure}[b]{0.49\textwidth}
         \centering
         \includegraphics[width=\textwidth]{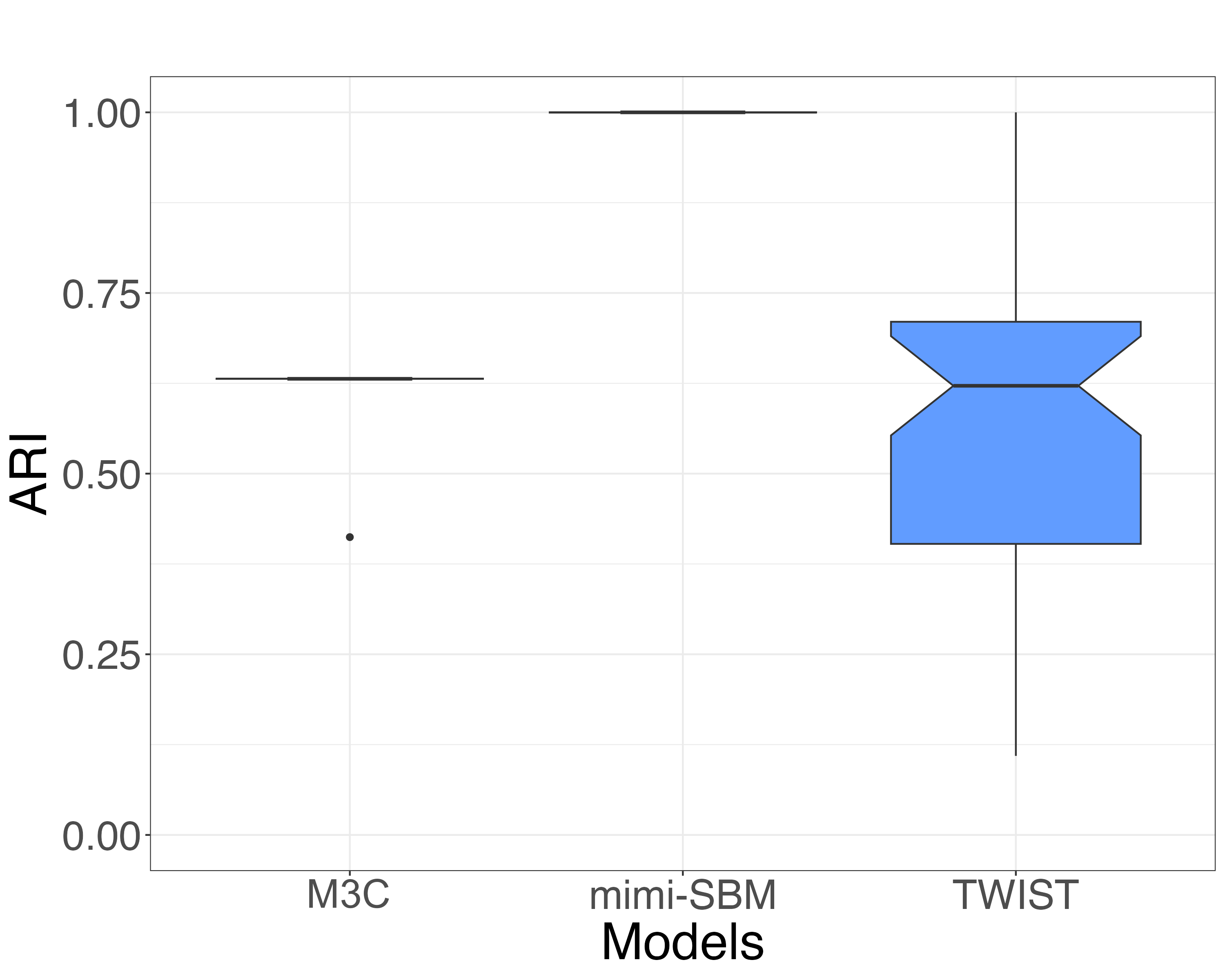}
         \caption{$N=50,V=15,K=5,Q=3$}

     \end{subfigure}
     \hfill
     \begin{subfigure}[b]{0.49\textwidth}
         \centering
         \includegraphics[width=\textwidth]{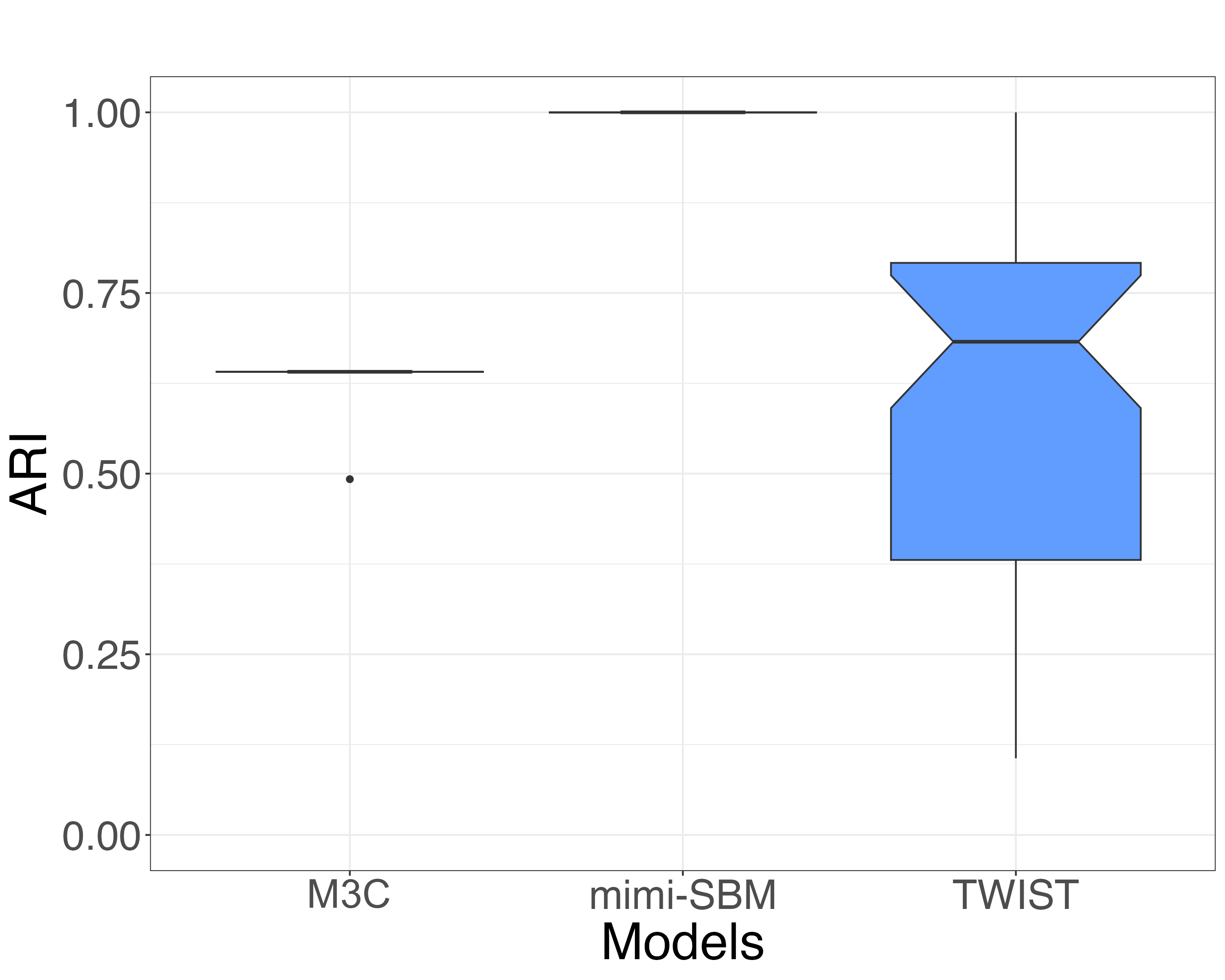}
         \caption{$N=200,V=15,K=5,Q=3$}

     \end{subfigure}
     \hfill
     \begin{subfigure}[b]{0.49\textwidth}
         \centering
         \includegraphics[width=\textwidth]{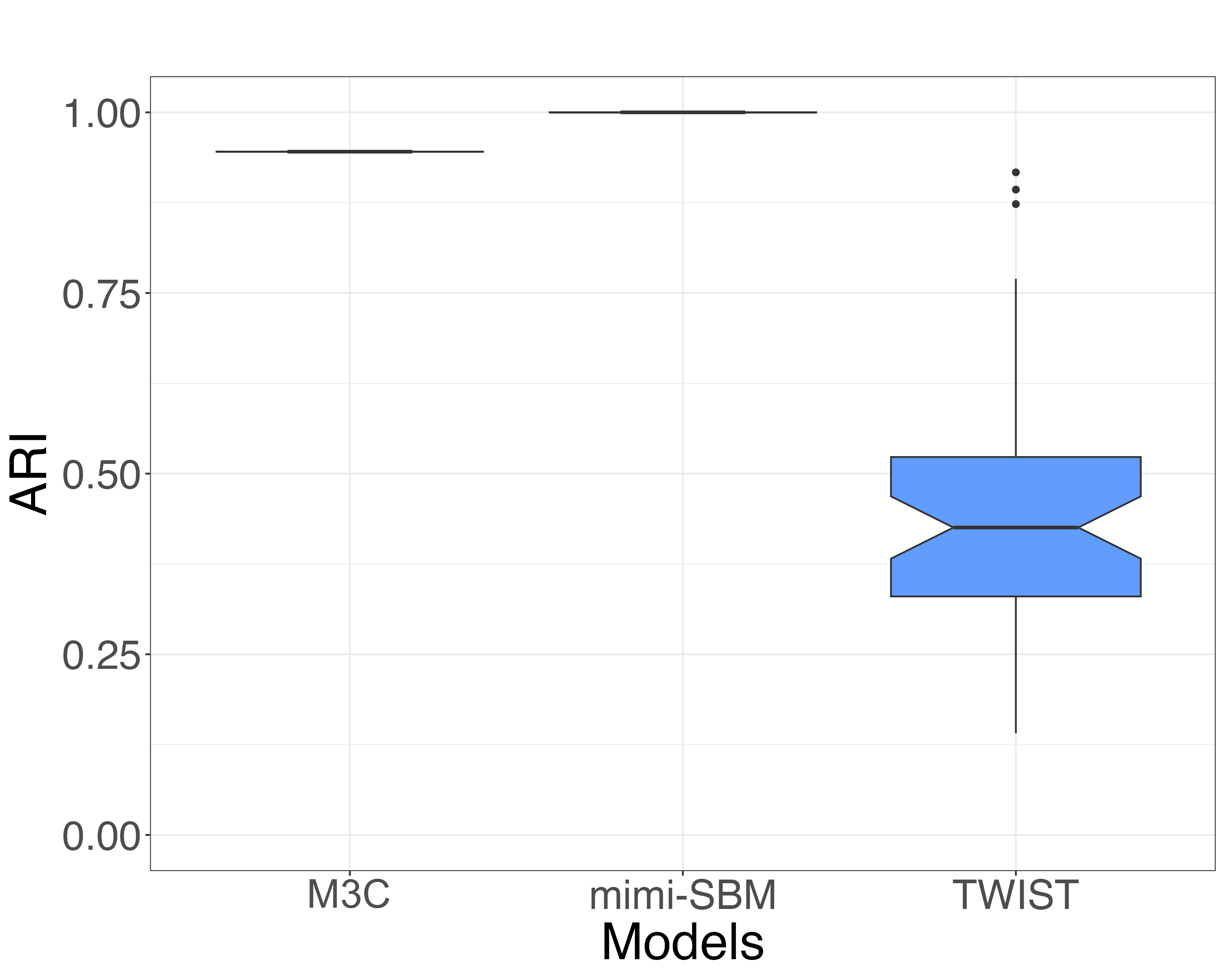}
         \caption{$N=200,V=50,K=10,Q=10$}

     \end{subfigure}
     
        \caption{Boxplot of ARI measure between true partition and output partition of \textit{M3C}, \textit{mimi-SBM} and \textit{TWIST} models.}
        \label{fig: competitors clustering}
\end{figure}

\paragraph{Comparison of view components.} In  Figure~\ref{fig: competitors views},  as the number of observations increases, the performances of the models generally tends to improve. However, the \textit{graphclust} model appears to identify the true sources less frequently than the \textit{mimi-SBM} and \textit{TWIST} models. While \textit{TWIST} often identifies the true members of the sources perfectly, it does make some errors, visible as outliers on the boxplot.

\begin{figure}[!ht]
     \centering
     \begin{subfigure}[b]{0.49\textwidth}
         \centering
         \includegraphics[width=\textwidth]{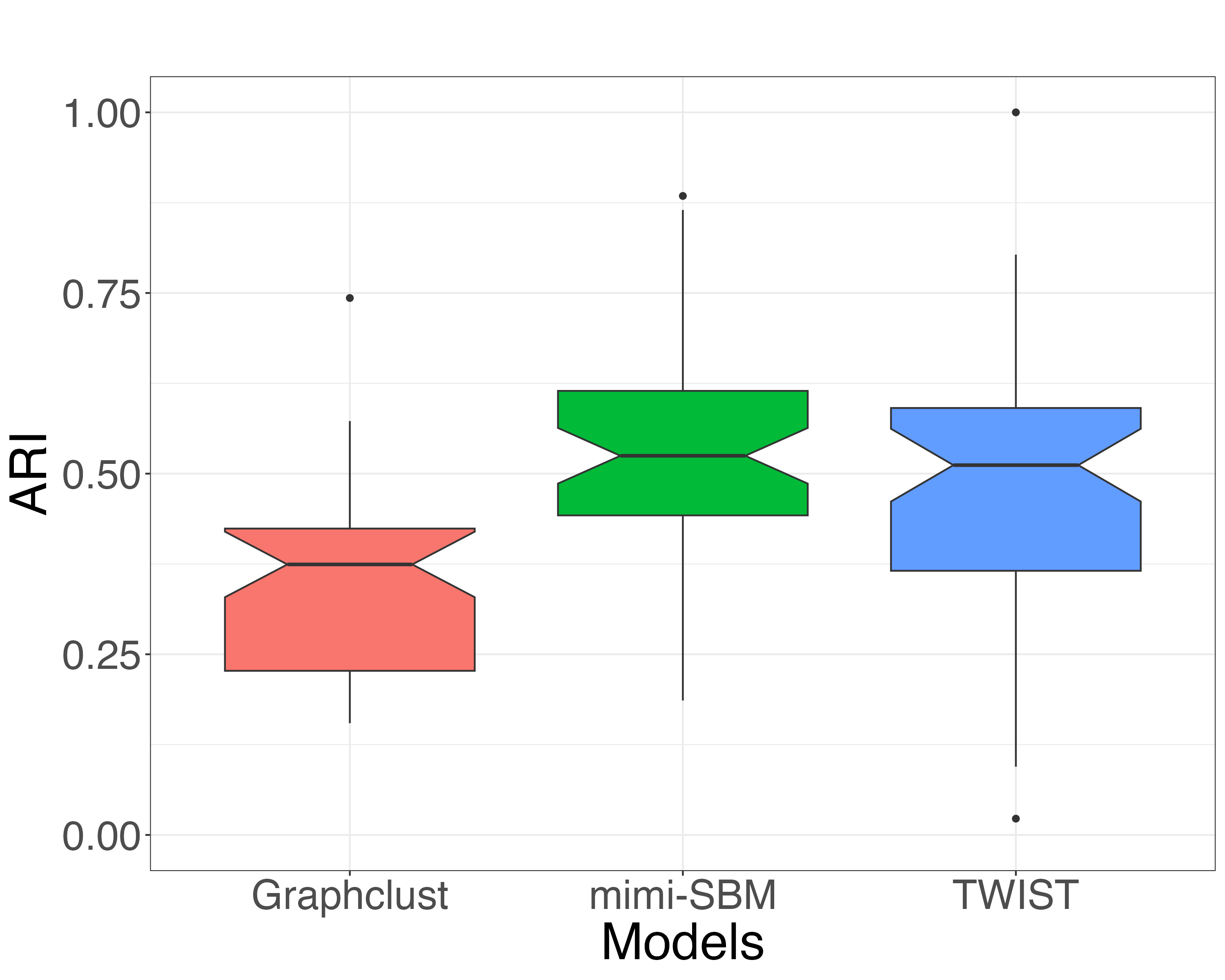}
         \caption{$N=50,V=15,K=5,Q=3$}

     \end{subfigure}
     \hfill
     \begin{subfigure}[b]{0.49\textwidth}
         \centering
         \includegraphics[width=\textwidth]{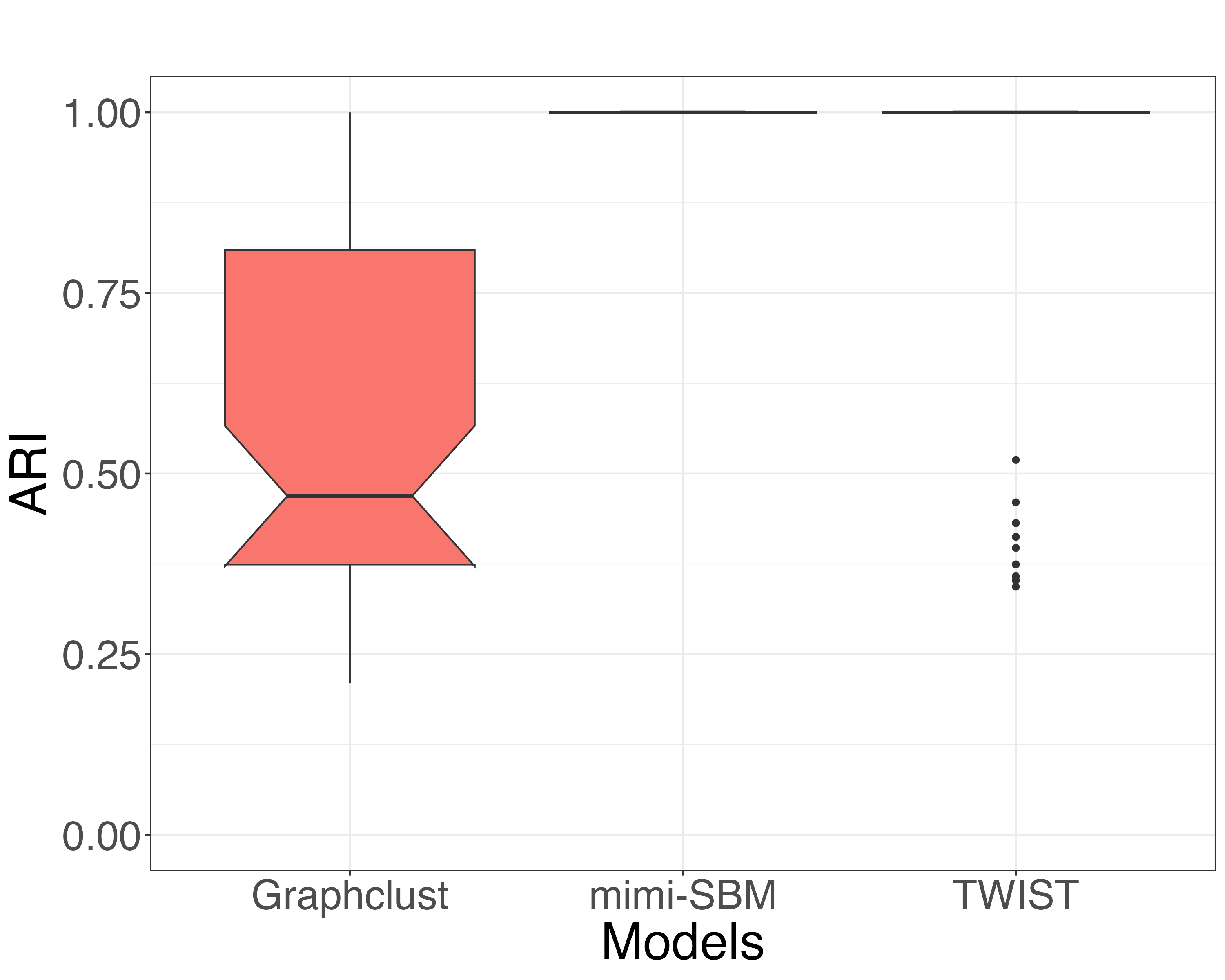}
         \caption{$N=200,V=15,K=5,Q=3$}

     \end{subfigure}
     \hfill
     \begin{subfigure}[b]{0.49\textwidth}
         \centering
         \includegraphics[width=\textwidth]{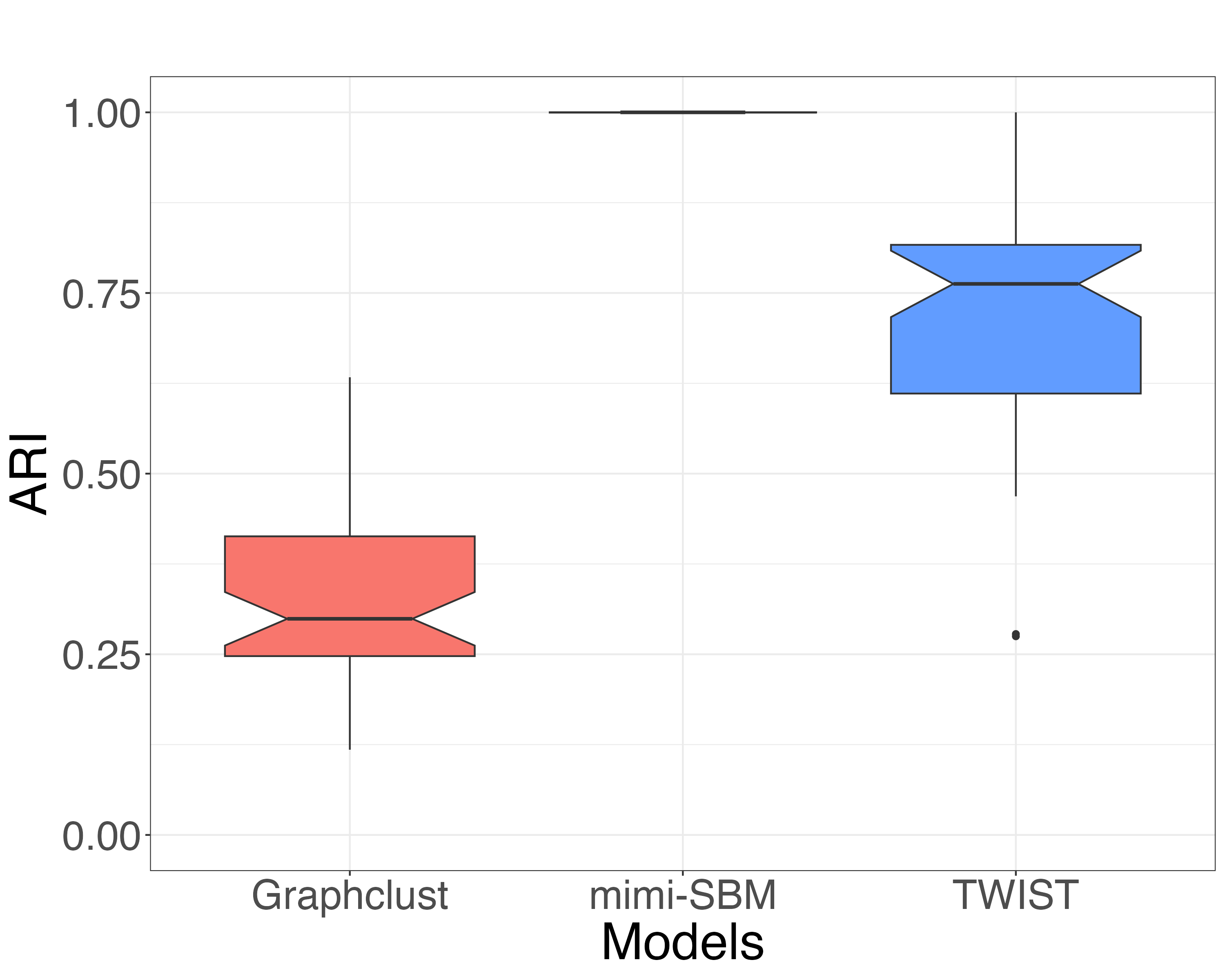}
         \caption{$N=200,V=50,K=10,Q=10$}
     \end{subfigure}
        \caption{Boxplot of ARI measure between true view clustering and output clustering of \textit{graphclust}, \textit{mimi-SBM} and \textit{TWIST} models.}
        \label{fig: competitors views}
\end{figure}


\subsubsection{Simulations with label-switching}
\label{sec:SOTA_noisy_XP}

In this section, we revisit the analyses from the previous section, but with a focus on a more challenging issue: label-switching.
The idea of perturbing the cluster labels within the generation process simulates the fact that an individual has been associated with another cluster during the process of creating adjacency matrices.

In our context, we simulate the fact that an individual belongs to the real cluster, and then we simulate the representation of this clustering by the link between the final clustering and the one specific to each view component, to obtain the different affinity matrices.

The perturbation occurs during the generation of individual component-based clusters. For each view, a perturbation is introduced for each individual with a probability of $p_\text{switch} = 0.1$. In such perturbation, the respective individual is then associated with one of the other available clusters.
As a result, the probability of creating a link between individuals is influenced.

\paragraph{Comparison of clustering.} The analysis summarized in Figure~\ref{fig: competitors clustering LS} reveals that across all examined experimental setups, the \textit{mimi-SBM} consistently attained the most favorable clustering outcomes.  
Furthermore, it is noteworthy that the score variability associated with \textit{mimi-SBM} is notably lower than that observed for other models.
The effectiveness of the \textit{M3C} and \textit{TWIST} model showed improvement as the number of views, clusters, and sources increased, yet it maintained a relatively high level of variance.
The number of individuals to be clustered plays a crucial role in minimizing errors. This effect stems from the fact that a larger number of individuals subject to clustering contributes to a more robust estimation of the parameters.

\begin{figure}[!ht]
     \centering
     \begin{subfigure}[b]{0.49\textwidth}
         \centering
         \includegraphics[width=\textwidth]{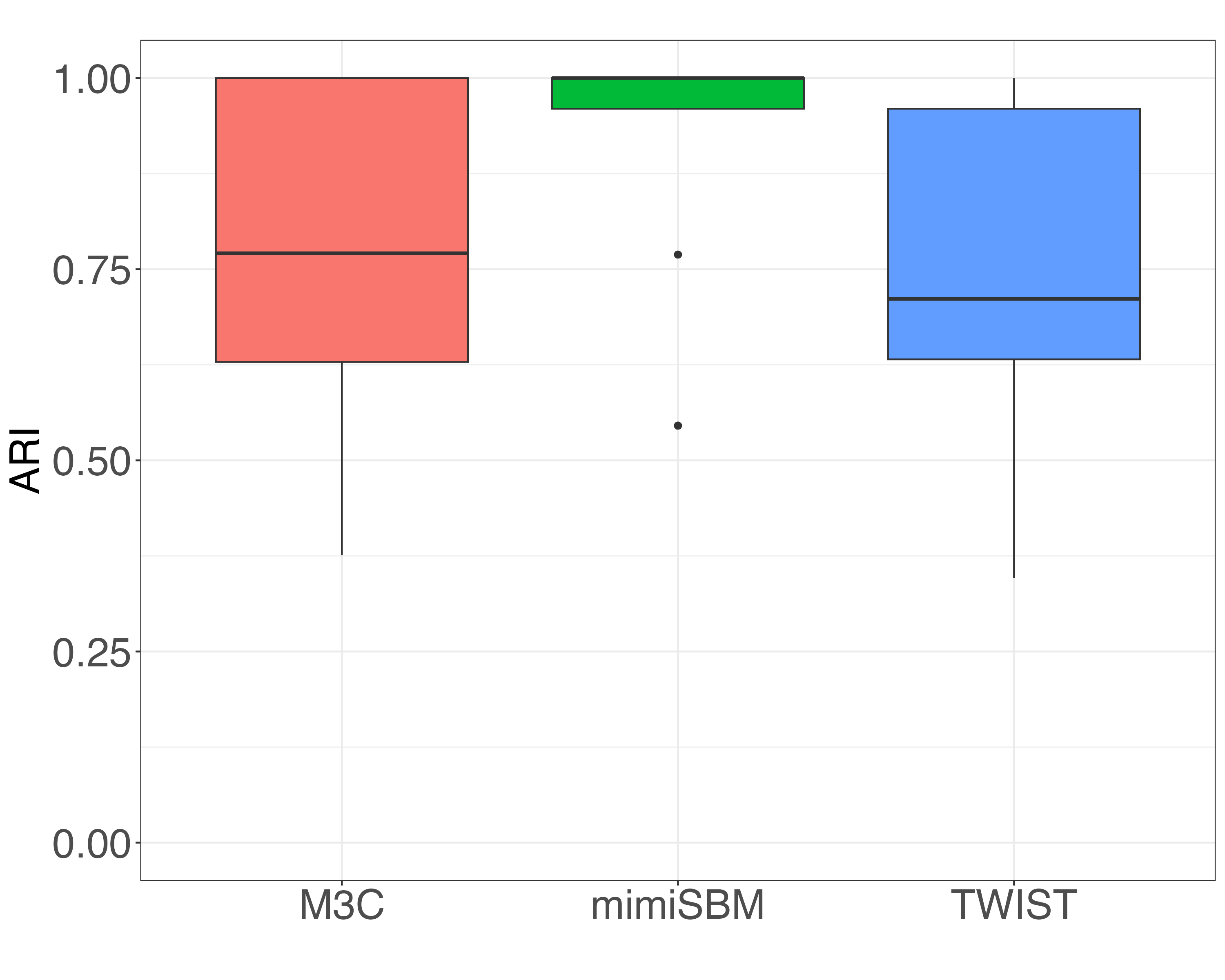}
         \caption{$N=50,V=15,K=5,Q=3$}

     \end{subfigure}
     \hfill
     \begin{subfigure}[b]{0.49\textwidth}
         \centering
         \includegraphics[width=\textwidth]{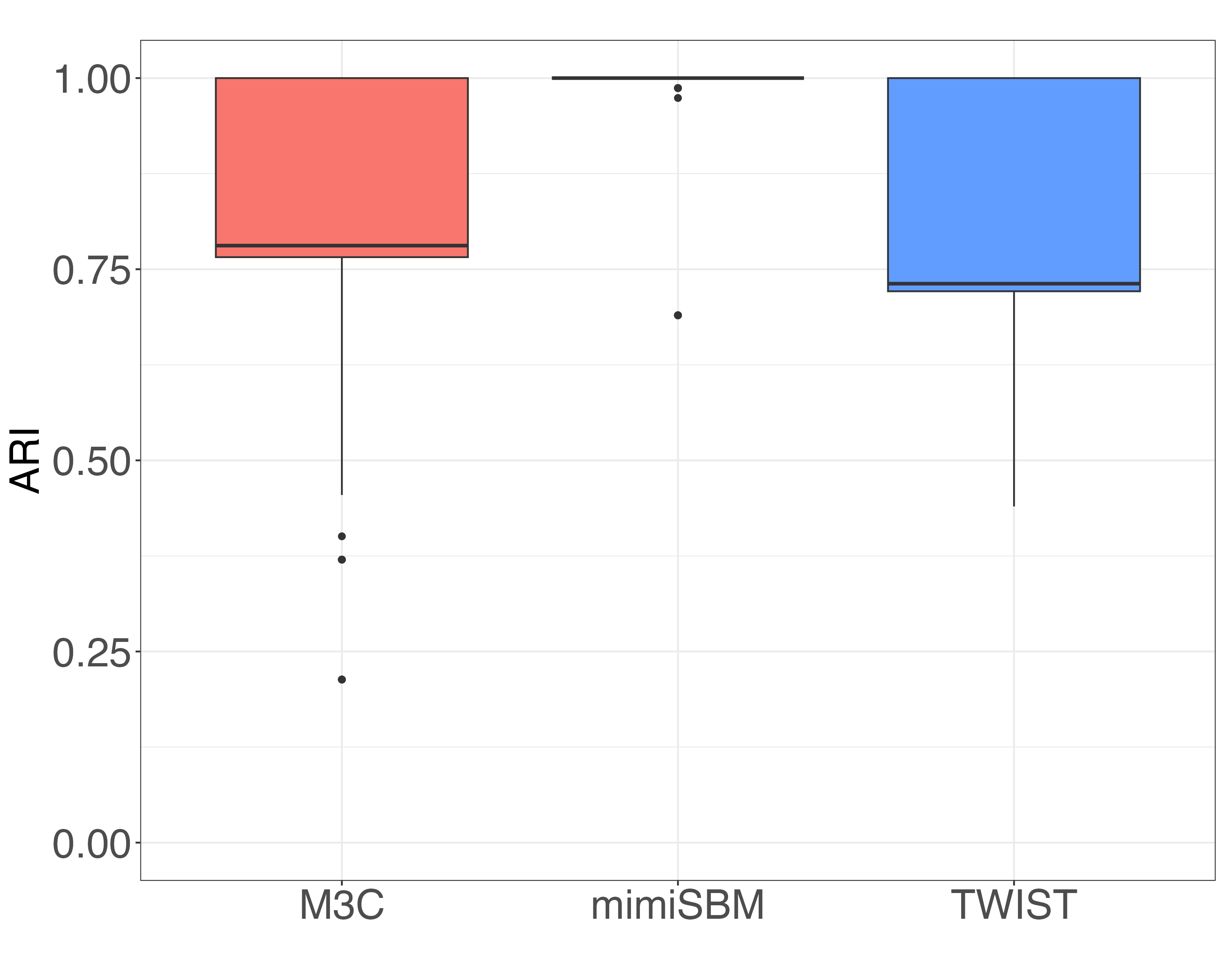}
         \caption{$N=200,V=15,K=5,Q=3$}

     \end{subfigure}
     \hfill
     \begin{subfigure}[b]{0.49\textwidth}
         \centering
         \includegraphics[width=\textwidth]{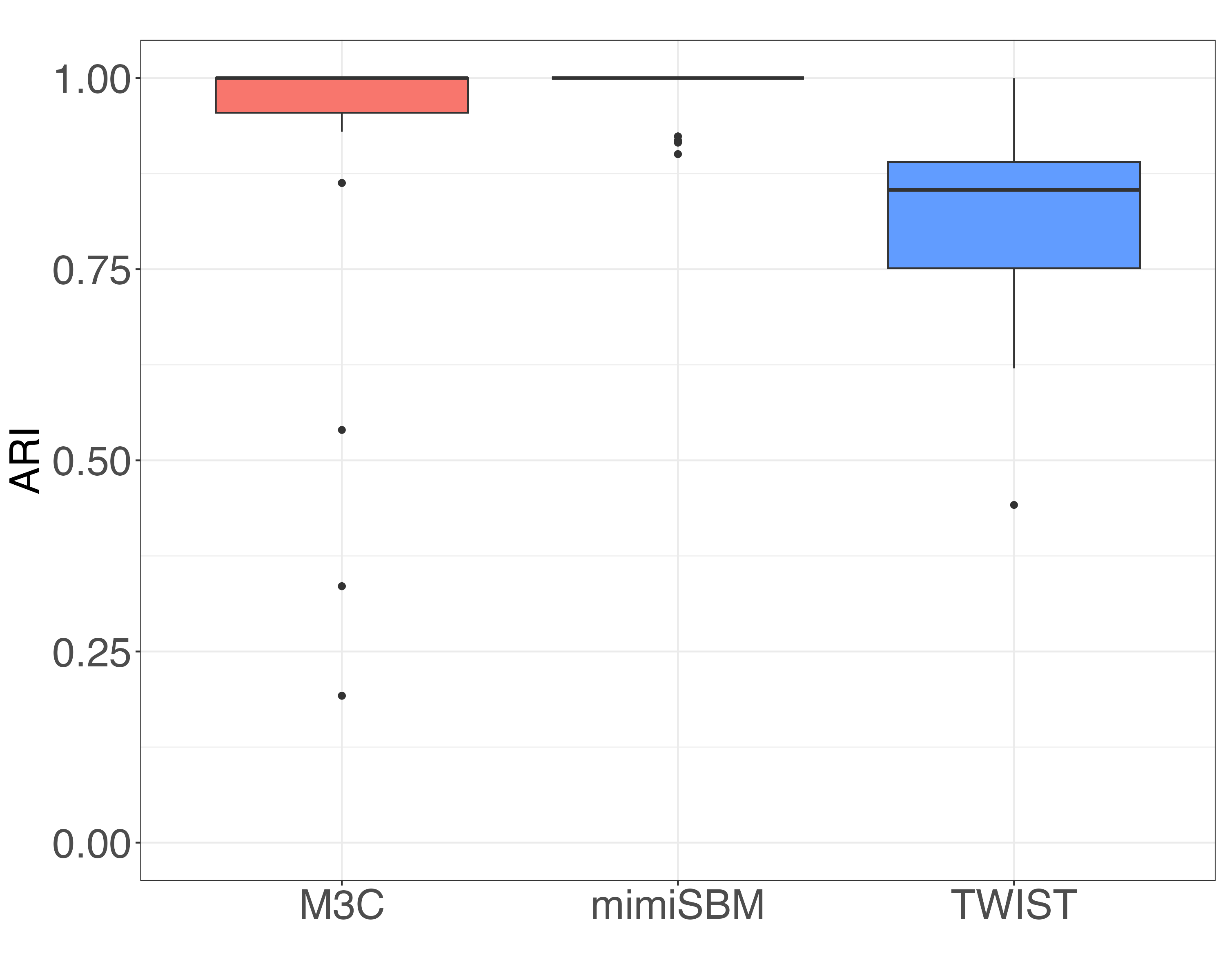} 
         \caption{$N=200,V=50,K=10,Q=10$}

     \end{subfigure}
     
        \caption{Boxplot of ARI measure between true partition and output partition of \textit{M3C}, \textit{mimi-SBM} and \textit{TWIST} models.}
        \label{fig: competitors clustering LS}
\end{figure}

\paragraph{Comparison of view components.} In Figure~\ref{fig: competitors views LS}, as the quantity of observations increases, the models typically exhibit enhanced performance. 
Nevertheless, the \textit{graphclust} model seems to less frequently pinpoint the actual sources compared to the \textit{mimi-SBM} and \textit{TWIST} models.
When faced with a small number of perspectives, \textit{TWIST} model displays significant variability. While it consistently delivers good results, it remains vulnerable to unfavorable initializations, which can lead to notably suboptimal clustering outcomes. Moreover, when label-switching is introduced, the model's performance is observed to be slightly less effective compared to the precedent scenario.

Similar to the scenario without label-switching, the model experiences considerable variability in its estimation when dealing with a limited number of individuals and perspectives. However, as the number of individuals and views increases, the variance of ARI decreases noticeably, accompanied by an improvement in performance. \textit{Mimi-SBM} model consistently demonstrates efficacy across all cases, even including perturbations in the adjacency matrices used for clustering.

\begin{figure}[!ht]
     \centering
     \begin{subfigure}[b]{0.49\textwidth}
         \centering
         \includegraphics[width=\textwidth]{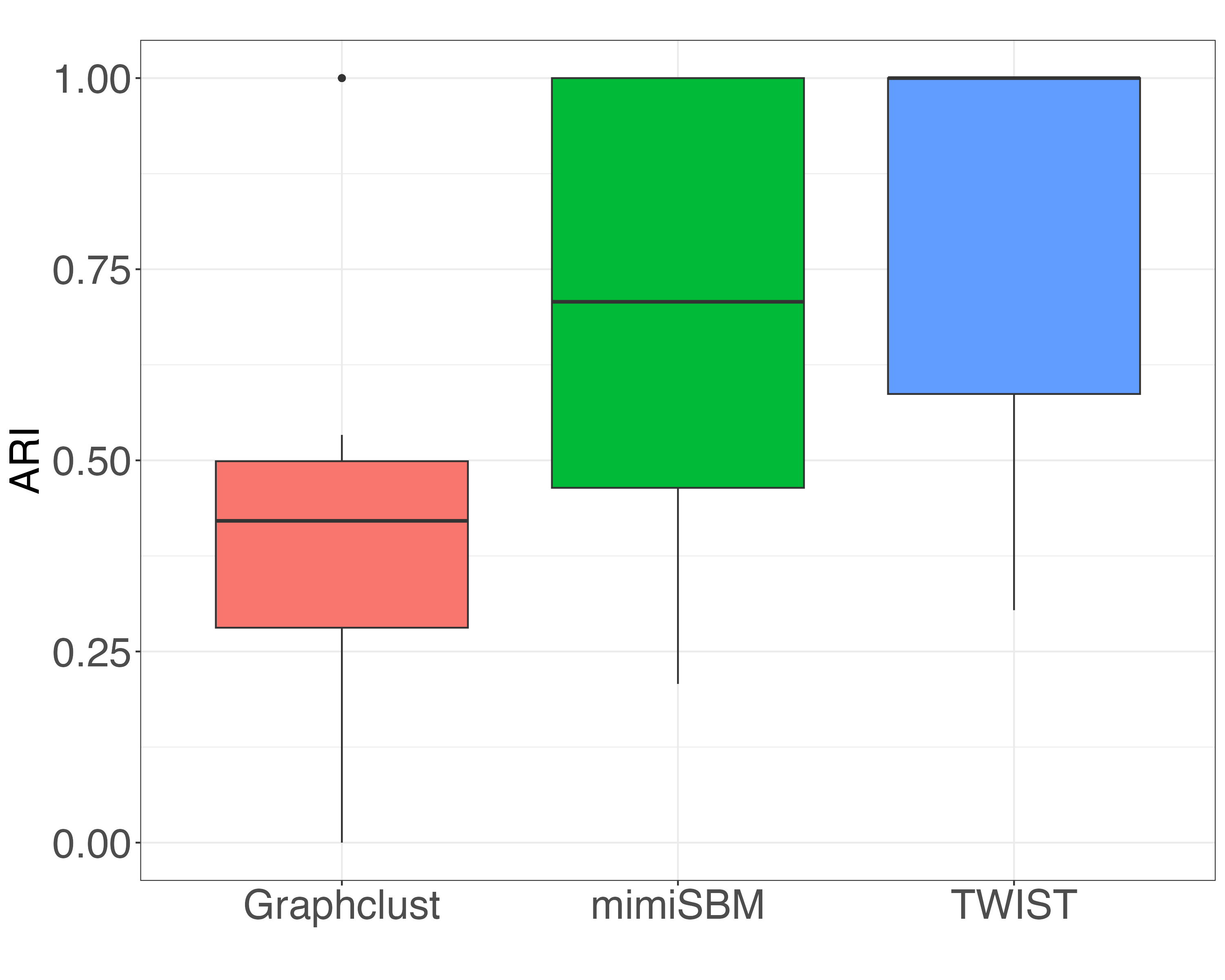}
         \caption{$N=50,V=15,K=5,Q=3$}

     \end{subfigure}
     \hfill
     \begin{subfigure}[b]{0.49\textwidth}
         \centering
         \includegraphics[width=\textwidth]{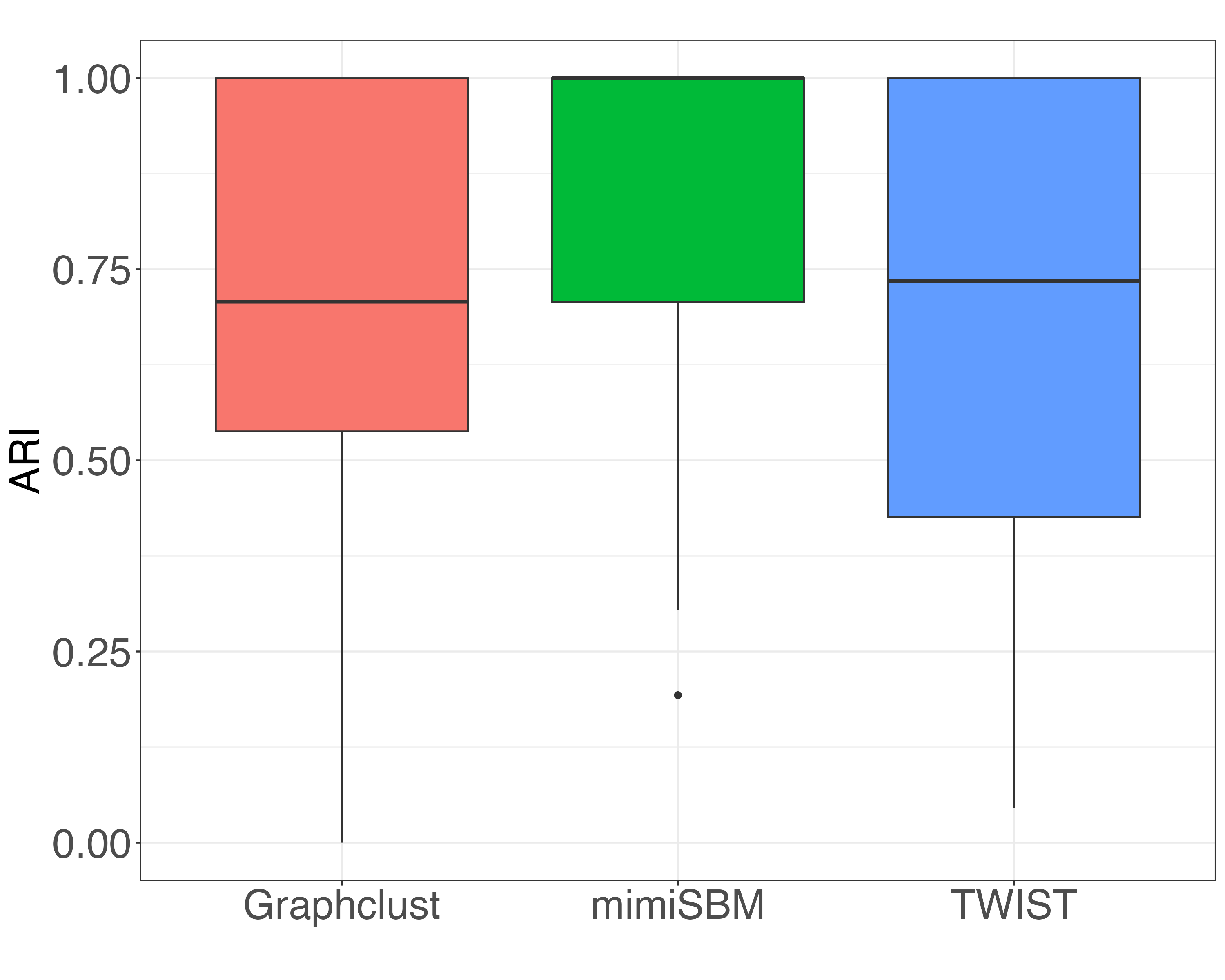}
         \caption{$N=200,V=15,K=5,Q=3$}

     \end{subfigure}
     \hfill
     \begin{subfigure}[b]{0.49\textwidth}
         \centering
         \includegraphics[width=\textwidth]{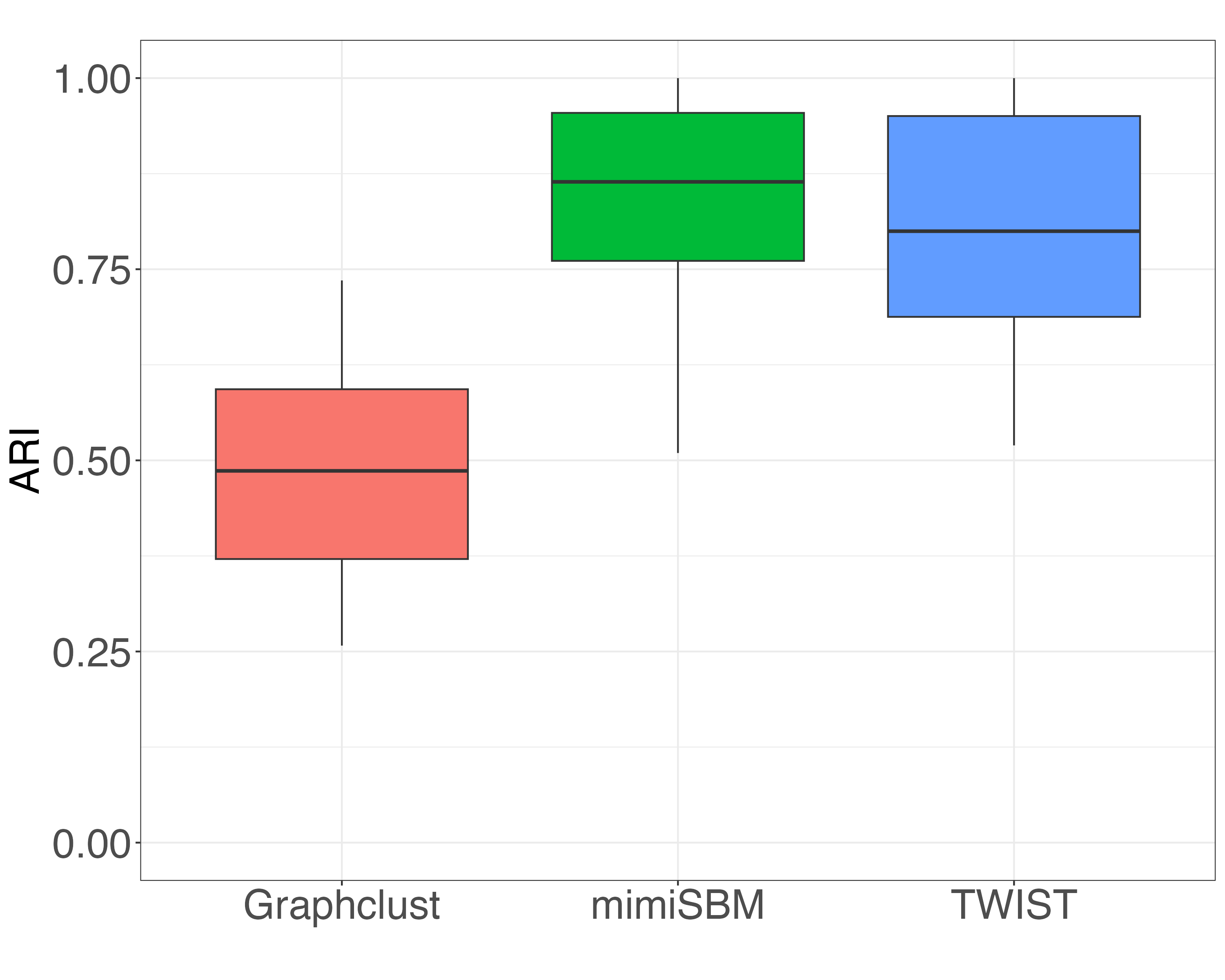}
         \caption{$N=200,V=50,K=10,Q=10$}

     \end{subfigure}
     \hfill
        \caption{Boxplot of ARI measure between true view clustering and output clustering of \textit{graphclust}, \textit{mimi-SBM} and \textit{TWIST} models.}
        \label{fig: competitors views LS}
\end{figure}


\subsubsection{Robustness to label-switching}
\label{sec:robustness_XP}
Given that the \textit{mimi-SBM} showed a satisfactory performance level in the previous section, even under the influence of label-switching perturbation, this section aims to further assess the robustness and limitations of our model concerning this criterion.

By varying the label switching rate, from $0$ to $1$ in steps of $0.10$, in order to see the evolution of clustering capacities on individuals and views.

\begin{figure}[!ht]
     \centering
     \begin{subfigure}[b]{0.49\textwidth}
         \centering
         \includegraphics[width=\textwidth]{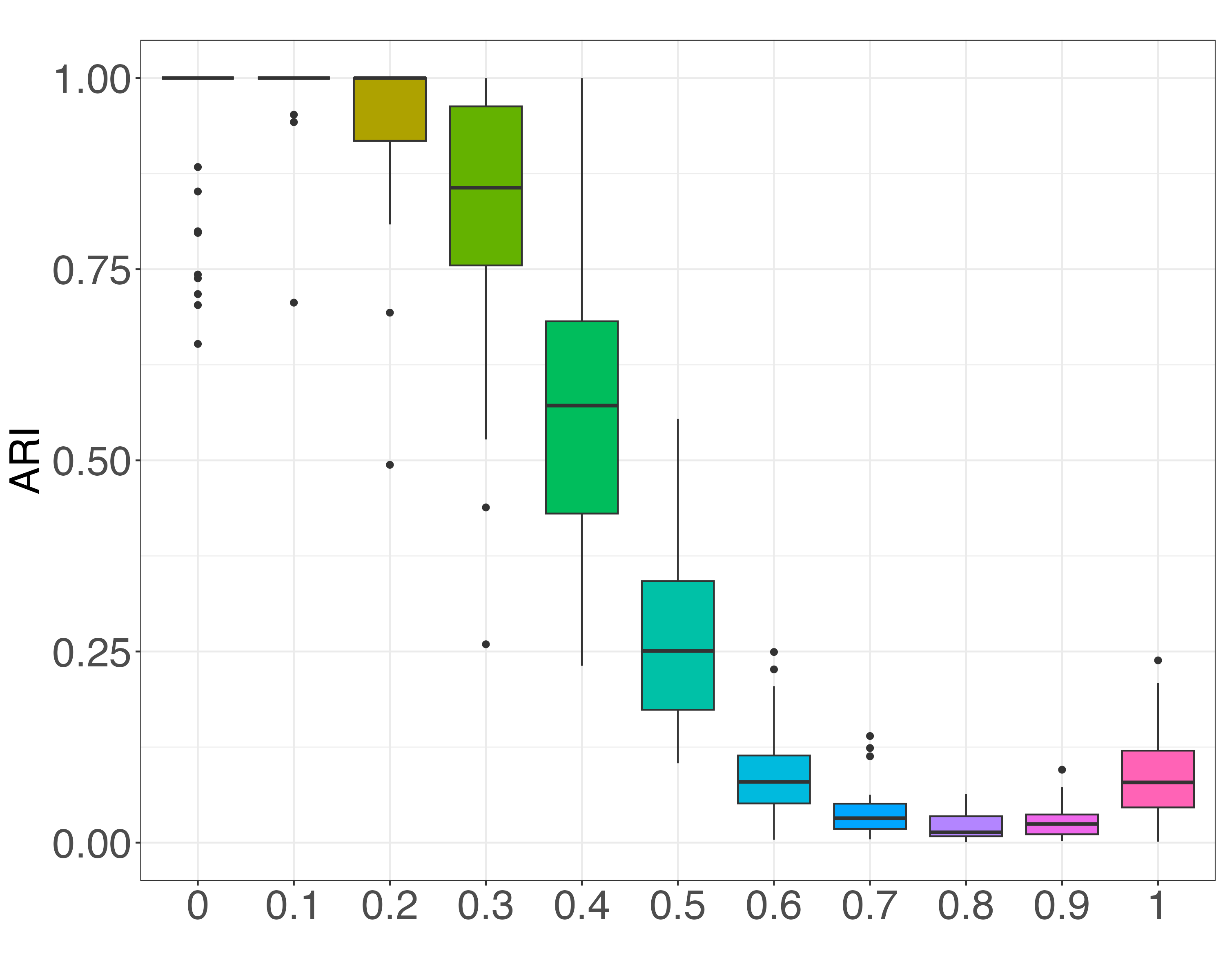}
         \caption{$N=50,V=15,K=5,Q=3$}

     \end{subfigure}
     \hfill
     \begin{subfigure}[b]{0.49\textwidth}
         \centering
         \includegraphics[width=\textwidth]{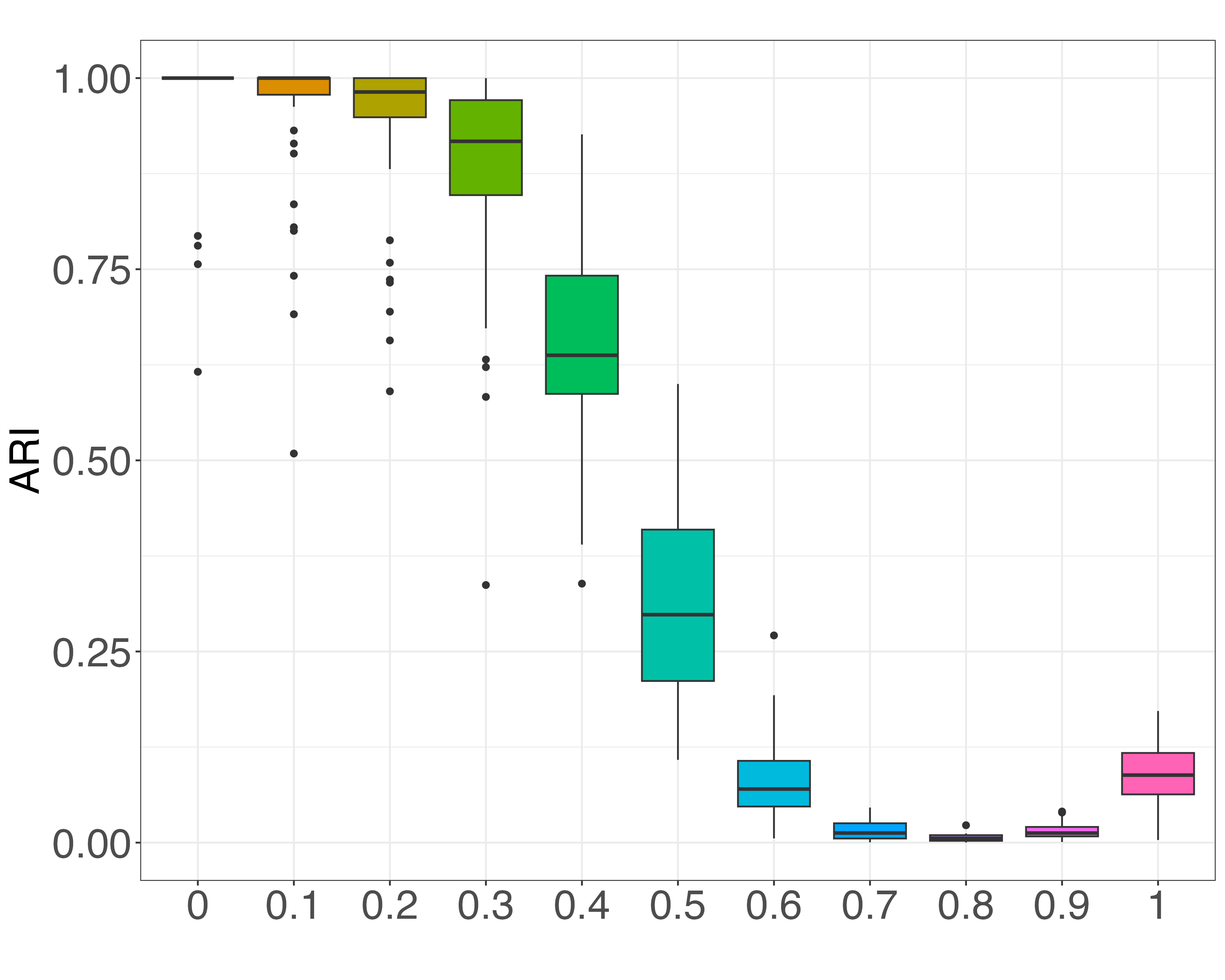}
         \caption{$N=200,V=15,K=5,Q=3$}

     \end{subfigure}
     \hfill
        \caption{Performances of \textit{mimi-SBM} on individual clustering through the evolution of label-switching rate.}
        \label{fig: Error Clustering}
\end{figure}

\begin{figure}[!ht]
     \centering
     \begin{subfigure}[b]{0.49\textwidth}
         \centering
         \includegraphics[width=\textwidth]{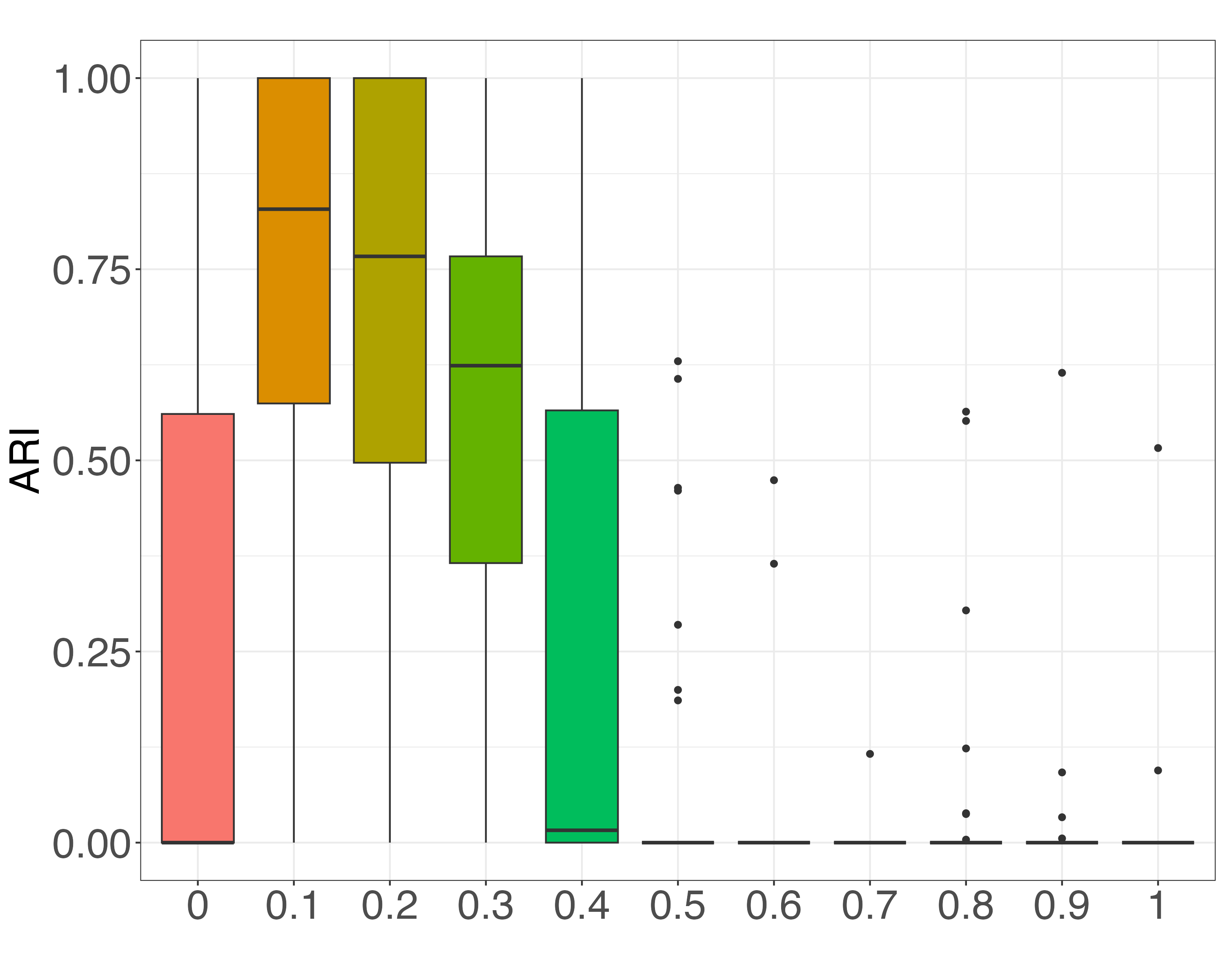}
         \caption{$N=50,V=15,K=5,Q=3$}

     \end{subfigure}
     \hfill
     \begin{subfigure}[b]{0.49\textwidth}
         \centering
         \includegraphics[width=\textwidth]{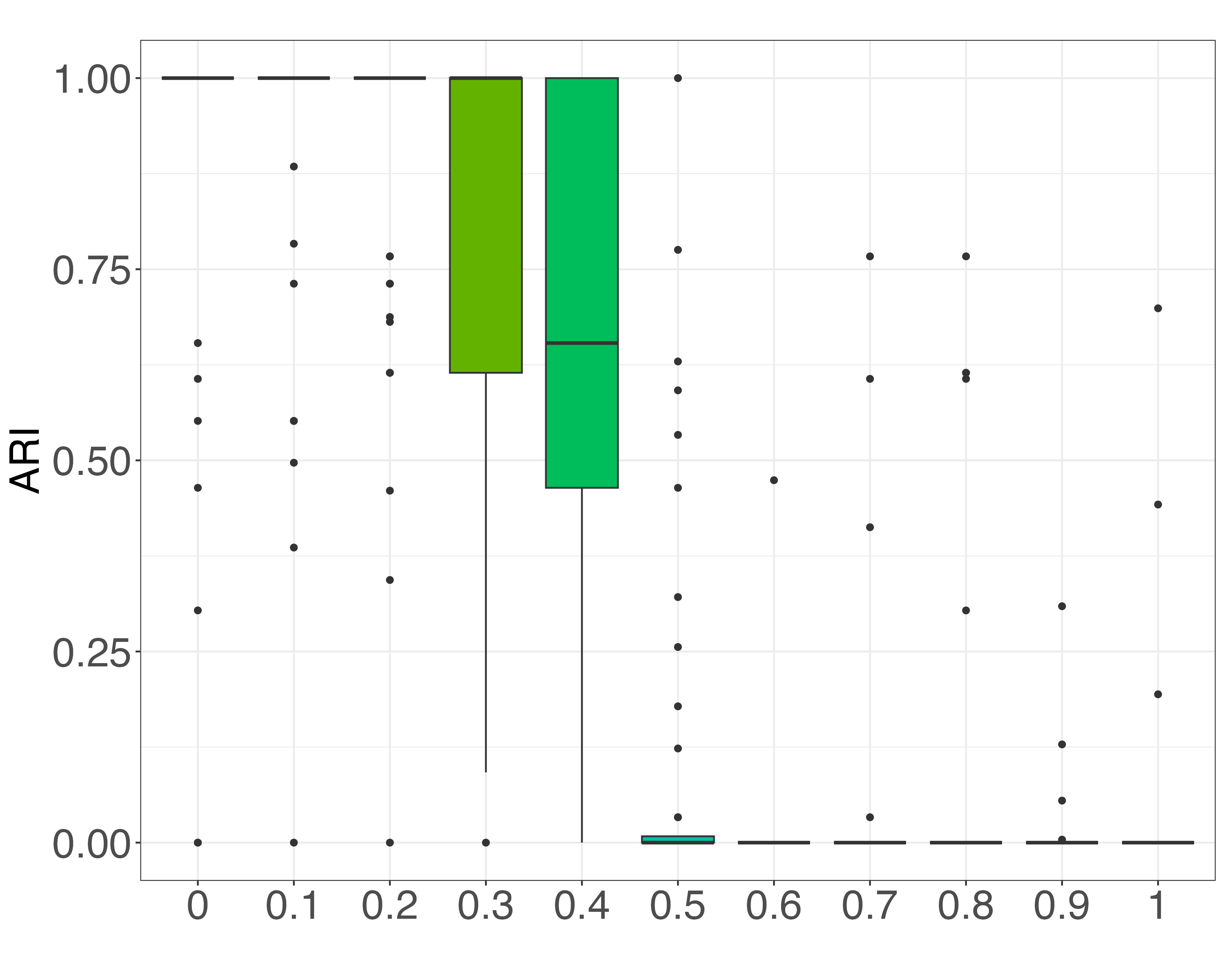}
         \caption{$N=200,V=15,K=5,Q=3$}

     \end{subfigure}
     \hfill
        \caption{Performances of \textit{mimi-SBM} on view clustering through the evolution of label-switching rate.}
        \label{fig: Error view mixture}
\end{figure}

For Figures~\ref{fig: Error Clustering} and~\ref{fig: Error view mixture}, clustering performances demonstrate a significant level of efficacy when the label-switching rate is low. 

As the rate of switched labels exceeds $40\%$, the stability of the individual clustering process progressively diminishes. This trend continues until the clustering process becomes entirely arbitrary when the switch-labeling rate surpasses $60\%$, as contrasted with the true partition.
An observable improvement in performance becomes evident as the switched label rate approaches $1$. This outcome is logically anticipated, as the reassignment of all individuals from one cluster to another results in their distribution across $K-1$ clusters instead of the initial $K$ clusters.

In the context of view-based clustering, we encounter a similar set of observations, albeit with a much more pronounced decline in performance. When the label-switching rate surpasses $20\%$, the ability of \textit{mimi-SBM} to effectively identify view components experiences a drastic reduction. Furthermore, when this rate exceeds $40\%$, the feasibility and relevance of conducting clustering based on these views are severely compromised.
One plausible explanation for this phenomenon is that, due to the perturbation, each adjacency matrix becomes highly noisy, lacking any discernible structure. Consequently, the model struggles to distinguish any specific connections within the mixtures, leading to a notably diminished clustering performance score.

\paragraph{Summary.} The \textit{mimi-SBM}  model has shown its capability in successfully recovering the stratification of individuals and the components of the mixture of views, even when the data is perturbed. However, like any statistical model, its performance, especially regarding the mixture of views, benefits from larger sample sizes. The accurate modeling of mixture components is crucial in various applications, making the mimi-SBM model highly valuable in a wide range of contexts.

\subsection{Worldwide Food Trading Networks }

\paragraph{Data.} This section delves into the analysis of a global food trading dataset initially assembled by ~\cite{de2015structural}, accessible at \url{http://www.fao.org}. The dataset includes economic networks covering a range of products, where countries are represented as nodes and the edges indicate trade links for particular food products.
Following the same preprocessing steps as \cite{jingCommunityDetectionMixture2020a}, we prepared the data to establish a common ground for comparing clustering outcomes. The original directed networks were simplified by omitting their directional features, thereby converting them into undirected networks.

Subsequently, to effectively filter out less significant information from the dataset, we eliminate links with a weight of less than $8$ and layers containing limited information (less than $150$ nodes). Finally, the intersections of the biggest networks of the preselected layers are then extracted.
Each layer reflects the international trade interactions involving 30 distinct food products among $99$ different countries and regions (nodes). 

\begin{figure}[!ht]
     \centering
     
     \begin{subfigure}[b]{\textwidth}
         \centering
         \includegraphics[width=\textwidth]{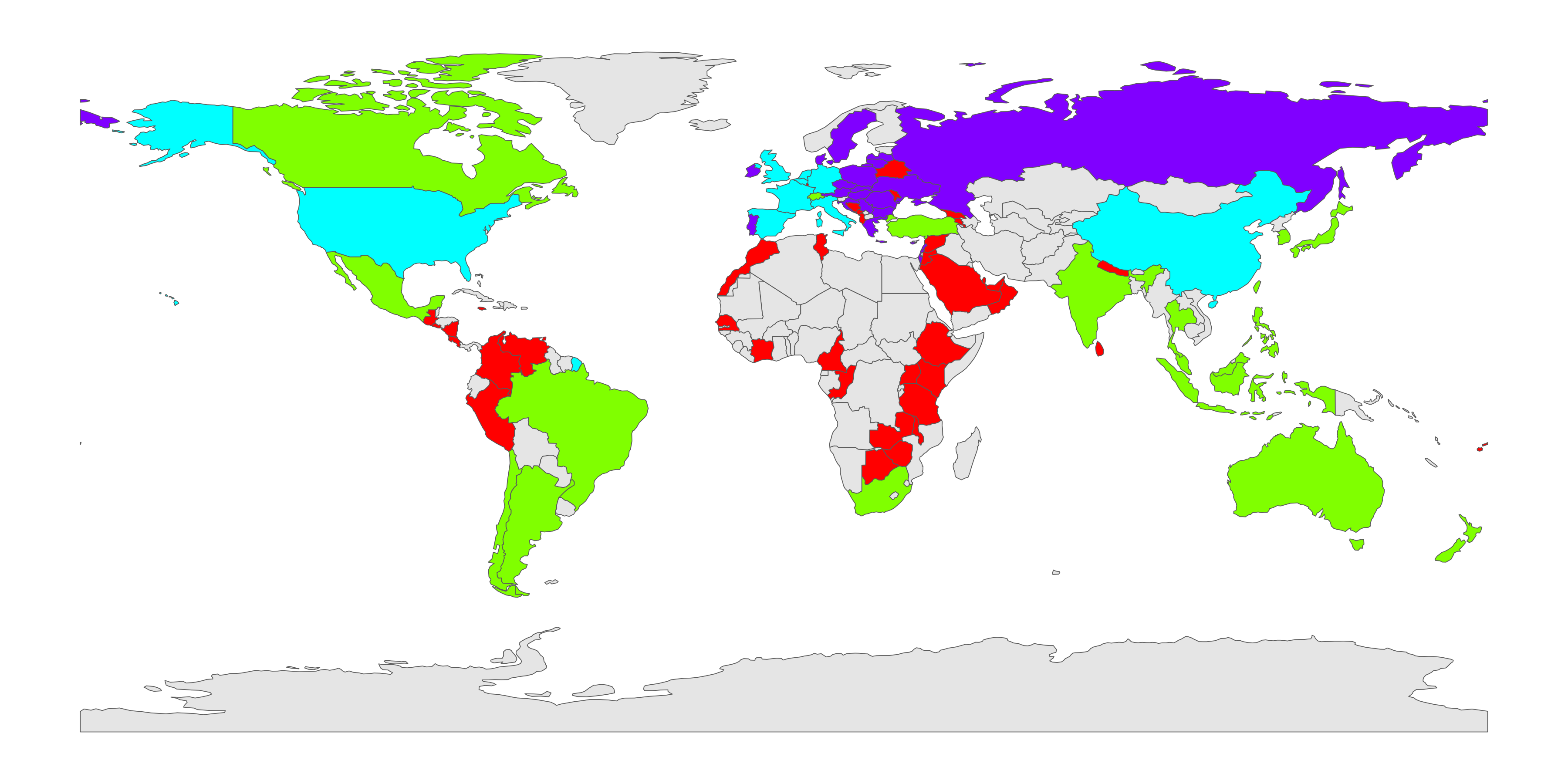}
     \end{subfigure}
     \hfill
        \caption{World Map of Clusters: Countries are color-coded based on the clusters identified by the model. The cyan cluster (cluster $1$) encompasses the West and China; The violet cluster (cluster $2$) consists of Russia and some parts of Western Europe; The red cluster (cluster $3$) includes countries from Africa and Central America; The green cluster (cluster $4$) covers Mexico, Canada, India, Australia, South Africa, Japan, among others; Countries depicted in grey are not included in the database analyzed.}
        \label{fig: Food K4}
\end{figure}

\paragraph{TWIST analysis.} In our research, we followed the analytical process described in \cite{jingCommunityDetectionMixture2020a}, to facilitate reliable comparison of results. Consistent with this methodology, we fixed the number of clusters at $K=4$ for individuals and $Q=2$ for views.


In Figure \ref{fig: Food K4}, clusters have their own interaction patterns:
\begin{itemize}
\item Cluster $1$ serves as a hub due to its centralization of exchanges, exhibiting a high intra-connectivity ($>90\%$) and substantial inter-connectivity ($>70\%$), as revealed by the multilayer adjacency probability analysis.
\item Cluster $2$ displays a robust intra-connectivity, with notable interactions observed with both clusters $1$ and $4$. Conversely, exchanges with cluster $3$ are infrequent for the commodities comprising the database.
\item Cluster $3$ and Cluster $4$ exhibit both intra-cluster and inter-cluster interactions, with a preference for inter-cluster interactions with cluster $1$. However, while Cluster $3$ predominantly interacts with Cluster $1$, Cluster $4$ demonstrates partial interaction with Cluster $2$.
\end{itemize}

\begin{table}[h]
\begin{center}
   \begin{tabular}{| c || c | }
     \hline
     View component 1 & Beverages\_non\_alcoholic ,      Food\_prep\_nes, \\
     & Chocolate\_products\_nes , Crude\_materials, \\
      &  Fruit\_prepared\_nes,  Beverages\_distilled\_alcoholic, \\          
     & Pastry,  Sugar\_confectionery,           Wine     
     \\ \hline
     View component 2 &   Cheese\_whole\_cow\_milk,    Cigarettes,            Flour\_wheat \\
     & Beer\_of\_barley,   Cereals\_breakfast,        Coffee\_green, \\
     & Milk\_skimmed\_dried,       Juice\_fruit\_nes,         Maize,  \\         
     & Macaroni,   Oil\_palm,      Milk\_whole\_dried,  \\      
 & Oil\_essential\_nes,     Rice\_milled,              Sugar\_refined,      Tea \\
 &  Spices\_nes,    Vegetables\_preserved\_nes,  Water\_ice\_etc,    \\      
 & Vegetables\_fresh\_nes,    Tobacco\_unmanufactured  \\ 
     \hline
   \end{tabular}
 \end{center}
    \caption{Table of members in view components.}
     \label{table: components}
 \end{table}

Exploration of the view components in Table \ref{table: components} reveals a marked tendency to distinguish between "processed products" and "unprocessed products", although there are some notable exceptions. In addition, it should be noted that Component $1$ displays more important connections than Component $2$, suggesting that the main flow of transactions is mainly concentrated on products included in Component 1. This observation reinforces Cluster 1's position as a central hub, remaining a predominant actor in the concentration of trade within the various components.

The analysis carried out in this study is reflected in a striking correlation with the steps taken in the precedent analysis. Firstly, we found that the same partitions of individuals were present, with only minor variations in clustering. The links forged within these groups proved to be consistent with market dynamics, highlighting, in particular, the hub role played by cluster $1$ in global trade. Furthermore, the overall partitioning of food types persisted, illustrating the persistent distinction between processed and unprocessed products, although a few exceptions were noted, similar to those observed in the previous analysis. In sum, our results largely converge with those of the \cite{jingCommunityDetectionMixture2020a} study, although a few discrepancies remain, underlining the importance of continuing research in this area to refine our understanding and approach. 

\paragraph{Our optimization.}

First, the criterion for choosing the optimal model was employed to guide the selection of hyperparameters. A grid search was conducted over a range of values, spanning from $1$ to $20$ for the hyperparameter $K$ and from $1$ to $10$ for $Q$, in concordance with parameters of the first core in ~\cite{jingCommunityDetectionMixture2020a} experimentation.
The model selection process led to the choice of hyperparameters $K = 20$ and $Q = 1$ as the most suitable configuration. 
The model found it excessively costly to introduce additional components across the views compared to the information gain achieved, so $Q=1$ was selected.
For individual clustering, $K = 20$ was based on the model's discovery of numerous micro-clusters representing countries based on their interaction habits. 
This indicates that the model successfully identified fine-grained distinctions among countries, revealing intricate subgroups within the data.

\begin{figure}[!h]
     \centering
     
     \begin{subfigure}[b]{\textwidth}
         \centering
         \includegraphics[width=\textwidth]{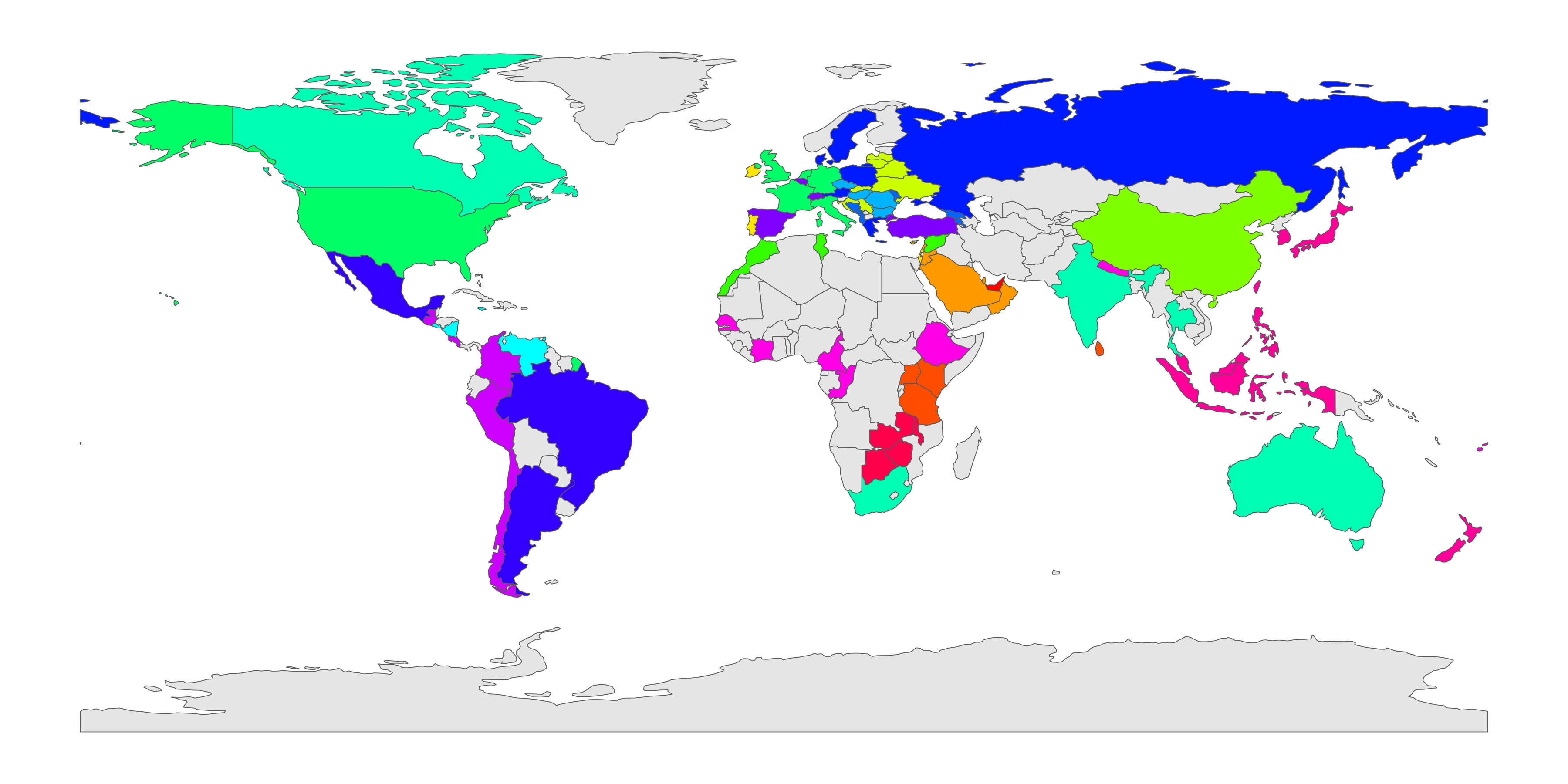}
     \end{subfigure}
     \hfill
        \caption{Clustering world map: countries are colored according to the clusters, and parameters defined by the model $(K=20)$.}
        \label{fig: Food K20}
\end{figure}

The results show that certain clusters have been substantially preserved, in particular Cluster $1$, which remains virtually intact, as does Cluster $4$. However, there has been a significant fragmentation of existing clusters, with China in particular remaining isolated. There have also been significant changes in the configuration of clusters, notably the inclusion of Russia along with other South American countries.
This change could be explained by the fact that we are now considering only one view component,  Russia is closer, in the sense of trading more with, the countries of South America than the rest of the world. 

\section*{Conclusion}
This paper proposes a new framework for Mixture of Multilayer SBM  \textit{mimi-SBM} that stratifies individuals as well as views. 

In order to get a manageable lower bound on the observed log-likelihood,  a variational Bayesian approach has been devised. Each model parameter has been estimated using a Variational Bayes EM algorithm. The advantage of such a Bayesian framework consists in allowing the development of an efficient model selection strategy. Moreover, we have provided the proof of model identifiability for the \textit{mimi-SBM} parameters.

In our simulation setting, the 
\textit{mimi-SBM} related algorithm has been shown to compete with methods based on tensor decomposition, hierarchical model-based SBM, and reference model in consensus clustering in two critical aspects of data analysis: individual clustering and view component identification. Specifically, our algorithm reliably recovered the primary sources of information in the majority of investigated cases. These remarkable performances attest to the efficacy of our approach, underscoring its potential for diverse applications requiring a profound understanding of complex data structures.
In real-world application on \textit{Worldwide Food Trading Networks}, when considering the paradigm provided by \cite{jingCommunityDetectionMixture2020a}, we obtain consistent results. However, upon optimizing our model using our metric, a distinctly different clustering emerges. This alternative clustering not only diverges significantly but also reflects much finer nuances in transactional natures.

An interesting follow-up would be to extend this approach to the context of deep learning, specifically in the context of variational auto-encoders using the Bayesian formulation. 
Additionally, further research is needed to develop theoretical proofs regarding the convergence of parameters for the component-connection probability tensor model.

\acks{KDS is funded by Sensorion and partially by Univ Evry, MS is funded by ENSIIE and partially Univ Evry and CA is funded by Univ Evry.}
%

\newpage
\bibliography{ms}

\newpage

 \begin{appendix}
 \section{Identifiability}
\label{sec:proof_identifiability}

This appendix is dedicated to the proof of the theorem of Section~\ref{sec:identifiability} related to the identifiability of the parameters of \textit{mimi-SBM}, recalled below. The proof is very similar to the one of~\cite{celisse2012consistency} 
and make use of algebraic properties to prove that the parameters depend solely on the marginal distribution of our data.

\begin{theorem}
Let $N \geq \max(2K,4Q)$ and $V \geq 2K$. Assume that for any $1 \leq k,l \leq K$ and every $1 \leq s \leq Q$, the coordinates of $\boldsymbol{\pi}^T  \boldsymbol{\alpha}_{k..}  \boldsymbol{\rho}$ are all different, $(\boldsymbol{\pi}^T  \boldsymbol{\alpha}_{..s} \boldsymbol{\pi})_{s=1:Q}$ are distinct, and each $(\boldsymbol{\alpha}_{kl.} \boldsymbol{\rho})_{k,l=1:K}$  differs. Then, the mimi-SBM parameter $\boldsymbol{\Theta} =\left( \boldsymbol{\pi},\boldsymbol{\rho},\boldsymbol{\alpha}\right)$ is identifiable.
\end{theorem}

\subsection{Assumptions}

\begin{enumerate}[{$\mathcal{A}$}1:]
    \item 
    $(\boldsymbol{\pi}^T  \boldsymbol{\alpha}_{k..}  \boldsymbol{\rho})_{k=1:K}$ are all different.
    \item 
    $(\boldsymbol{\pi}^T  \boldsymbol{\alpha}_{..s} \boldsymbol{\pi})_{s=1:Q}$ are all different. 
    \item 
    $N,V \geq 2K$.
    \item 
    $N \geq 4Q$.
    \item 
    $(\boldsymbol{\alpha}_{kl.} \boldsymbol{\rho})_{k,l=1:K}$ are all different. 
\end{enumerate}

\subsection{Identifiability of \texorpdfstring{$\boldsymbol{\pi}$}{pi}}
\label{sec:id_pi}



To prove the identifiability of $\boldsymbol{\pi}$, we first need to establish some correspondences. 
For any \(1 \leq k \leq K\), 
$\forall (i,j,v)$, let $r_k$ be the probability that an edge between $i$ and $j$ in layer $v$ given individual $i$ is in the cluster $k$:
\begin{align*}
r_k &= \mathbb{P}(A_{ijv} = 1 \mid Z_i = k ) \\
&= \sum_l \sum_s \mathbb{P}(A_{ijv} = 1 \mid Z_i = k, Z_j=l, W_v = s ) \ \pi_l \ \rho_s \\
&= \sum_l \sum_s \alpha_{kls} \ \pi_l \ \rho_s \\
&=  \boldsymbol{\pi}^T  \boldsymbol{\alpha}_{k..} \boldsymbol{\rho}\,.
\end{align*}

\begin{proposition}[Invertibility of $\mathbf{R}$]
\label{prop:R}
Let  $\mathbf{R}$ denote a Vandermonde matrix of size $K \times K$ such as 
\(\displaystyle R_{ik} = (r_k)^{i-1}, \ 1 \leq i, k \leq K\). 
$\mathbf{R}$ is invertible, since the coordinates of $r$ are all different according to Assumption~$\mathcal{A}1$.
\end{proposition}

Furthermore, for $2 \leq i \leq K$, the joint probability of having $(i-1)$ edges is given by:
\begin{align*}
& \mathbb{P}(A_{121} = 1, A_{132} = 1, \dots, A_{1i(i-1)} = 1 \mid Z_1 = k ) \\
& \qquad= \sum_l \sum_s \mathbb{P}(A_{121} = 1, A_{132} = 1, \dots, A_{1i(i-1)} = 1 \mid Z_1 = k, Z_2=l, W_1 = s ) \ \pi_l \ \rho_s \\
&\qquad= \mathbb{P}(A_{132} = 1, \dots, A_{1i(1-1)} = 1 \mid Z_1 = k) \times \sum_l \sum_s \mathbb{P}(A_{121} = 1 \mid Z_1 = k, Z_2=l, W_1 = s ) \ \pi_l \ \rho_s \\
&\qquad= \mathbb{P}(A_{132} = 1, \dots, A_{1i(1-1)} = 1 \mid Z_1 = k ) \ r_k \\
&\qquad= (r_k)^{i-1} \\
&\qquad= R_{ik} \, .
\end{align*} 

Now,  
we  define $u_0 = 1$ and for $1 \leq i \leq 2K-1$:
\begin{align*}
 u_i &= \mathbb{P}(A_{121} = 1, A_{132} = 1, \dots, A_{1i(1-1)} = 1,A_{1(i+1)i} = 1) \\
 &= \sum_k \mathbb{P}(A_{121} = 1, A_{132} = 1, \dots, A_{1(1+1)i} = 1 \mid Z_1=k) \ \pi_k \\
 &= \sum_k (r_k)^i \ \pi_k \, .
\end{align*}
By Assumption~$\mathcal{A}$3, $(u_i)_{i=1:(2K-1)}$ are well defined. 
%
Hence, $u_0 = 1$ and 
$(u_i)_{i=1:(2K-1)}$ 
are known and defined from the marginal $\mathbb{P}_A$. 
As a consequence, $(u_i)_{i=1:(2K-1)}$ 
are identifiable.
\medskip


Also, let $\mathbf{M}$ of size $(K+1) \times K$ be the matrix given by $M_{ij} = u_{i+j-2}$ for $1 \leq i \leq K+1$ 
and $1 \leq j \leq K$,  
and let $\mathbf{M}_{-i}$ denote the square matrix obtained by removing the row $i$ from $\mathbf{M}$. 
The coefficients of $\mathbf{M}_{-(K+1)}$,  for $1 \leq i,j \leq K$, are:
\begin{align}
M_{ij} &  = \sum_{k =1}^K (r_k)^{i-1} \ \pi_k\ (r_k)^{j-1}\, , 
\ \text{and } \nonumber \\
\mathbf{M}_{-(K+1)} & = \mathbf{R} \operatorname{Diag}(\boldsymbol{\pi})\mathbf{R}^T\, . \label{eq:M_pi}
\end{align}

\begin{proposition}[Relations between $\mathbf{R}$, $\mathbf{M}$ and $\boldsymbol{\pi}$]
\label{prop:RMpi}
From Proposition~\ref{prop:R} and Equation~\eqref{eq:M_pi}, we can define
\begin{align*}
\mathbf{M}_{-(K+1)} = \mathbf{R} \  \operatorname{Diag}(\boldsymbol{\pi})\ \mathbf{R}^T \, .
\end{align*}
\end{proposition}

%
%

%
%
%



The correspondence of the different terms being established, we now need to prove the identifiability of $\boldsymbol{\pi}$, which means showing that $\mathbf{M}_{-(K+1)}$ and $\mathbf{R}$ are identifiable.
\medskip

First, for the identifiability of $r_k$, with $\delta_k = \operatorname{Det}(\mathbf{M}_{-k})$, we define a polynomial function $B$ such as:
\begin{align*}
B(x) = \sum_{k=0}^K (-1)^{K+k} \ \delta_{k+1} \ x^k \, .
\end{align*}

This polynomial function has two important properties.

\begin{proposition}
\label{prop:deg_B} Let $\operatorname{deg}(B)$ denote the degree of $B$. We have  $\operatorname{deg}(B) = K$.
\end{proposition}

\begin{proof} Let $\delta_{K+1} = \operatorname{Det}(\mathbf{M}_{-(K+1)})$, with $M_{-(K+1)} = \mathbf{R} \  \operatorname{Diag}(\boldsymbol{\pi})\ \mathbf{R}^T$ as  stated in Proposition~\ref{prop:RMpi}, and $\mathbf{R}$ being invertible as stated in Proposition~\ref{prop:R}. %
In consequence, $\mathbf{M}_{-(K+1)}$ is the product of invertible matrices, $\delta_{K+1} = \operatorname{Det}(\mathbf{M}_{-(K+1)}) \ne 0$ and, moreover, $\operatorname{deg}(B) = K$.
\end{proof}

\begin{proposition}
\label{prop:root_B} For $ 1 \leq k  \leq K$, $B(r_k) = 0$.
\end{proposition}

\begin{proof}
Let  $\mathbf{N}_k$ of size $(K+1) \times (K+1)$ be the concatenation in columns of the matrix $\mathbf{M}$ with the vector \(\displaystyle V_k = [1, r_k, r^2_k, \dots, r_k^K]^T\).
\medskip

Now let's calculate the determinant of $\mathbf{N}_k$ developed by the last column:
\begin{align*} 
\det(\mathbf{N}_k) &= \sum_{l=0}^K (-1)^{K+1 +l+1} \, (r_k)^l \, \det(\mathbf{M}_{l+1}) \\
&= \sum_{l=0}^K (-1)^{K +l}  \, \delta_{l+1} \,  (r_k)^l \\
&= B(r_k).
\end{align*}

In addition, the $j$th column of the $\mathbf{M}$ matrix can be written as $\displaystyle M_{.j} = \sum\nolimits_{k=1}^K r_k^{j-1} \pi_k V_k$.
Therefore, $\operatorname{rank}(\mathbf{N}_k) < K+1$ and $\det(\mathbf{N}_k)= 0$ for $\ 1 \leq k \leq K$. In consequence,  $B(r_k) = 0$ for  $\ 1 \leq k \leq K$.
\end{proof}

With $(r_k)_{k=1:K}$ being the roots of $B$ (proposition \ref{prop:root_B}), they are functions of $(\delta_k)_{k=1:K+1}$ which are themselves derived from $\mathbb{P}_A$. 
Also, $(r_k)_{k=1:K}$ can be expressed in a unique way (up to label switching) from $\mathbb{P}_A$, thus $(r_k)_{k=1:K}$ are identifiable. In consequence, $\mathbf{R}$ is also identifiable by definition. 
Finally, since $\mathbf{M}_{-(K+1)}$ and $\mathbf{R}$ are identifiable and invertible, 
\(\displaystyle\operatorname{Diag}(\boldsymbol{\pi})  = \mathbf{R}^{-1} \mathbf{M}_{-(K+1)} (\mathbf{R}^T)^{-1} \). 
%
%
In conclusion, $\boldsymbol{\pi}$ is identifiable.

\subsection{Identifiability of \texorpdfstring{$\boldsymbol{\rho}$}{}}

Identifiability of $\boldsymbol{\rho}$ is similar to $\boldsymbol{\pi}$, the main difference lies in the assumptions made and the quantities defined. 

For any \(1 \leq s \leq Q\), 
$\forall (i,j,v)$, let $t_s$ be the probability of an edge between $i$ and $j$ in layer $v$ given view $v$ is in the component $s$:
\begin{align*}
t_s &= \mathbb{P}(A_{ijv} = 1 \mid W_v = s ) \\
&= \sum_l \sum_k \mathbb{P}(A_{ijv} = 1 \mid Z_i = k, Z_j=l, W_v = s ) \ \pi_l \ \pi_k \\
&= \sum_l \sum_k \alpha_{kls} \ \pi_l \ \pi_k \\
&=  \boldsymbol{\pi}^T  \boldsymbol{\alpha}_{..s}  \boldsymbol{\pi} \,.
\end{align*}

\begin{proposition}[Invertibility of $\mathbf{T}$]
\label{prop:T}
Let  $\mathbf{T}$ denote a Vandermonde matrix of size $Q \times Q$ such as 
\(\displaystyle T_{is} = (t_s)^{i-1}, \ 1 \leq i, s \leq Q\). 
$\mathbf{T}$ is invertible, since the coordinates of $(t_s)$ are all different according to Assumption~$\mathcal{A}2$.
\end{proposition}

Let's define the joint probability of $i-1$ edges given the latent component of the view $1$:
\begin{align*}
& \mathbb{P}(A_{121} = 1, A_{341} = 1, \dots, A_{2i-1 \, 2i \,  1} = 1 \mid W_1 = s ) \\ 
& \qquad= \sum_l \sum_k \mathbb{P}(A_{121} = 1, A_{341} = 1, \dots, A_{2i-1 \, 2i \,  1} \mid Z_1 = k, Z_2=l, W_1 = s ) \ \pi_l \ \pi_k \\
& \qquad= \mathbb{P}(A_{341} = 1, \dots, A_{2i-1 \, 2i \,  1} = 1 \mid W_1 = s ) \times 
\sum_l \sum_k \mathbb{P}(A_{121} = 1 \mid Z_1 = k, Z_2=l, W_1 = s ) \ \pi_l \ \pi_k \\
& \qquad= \mathbb{P}(A_{341} = 1, \dots, A_{2i-1 \, 2i \,  1} = 1 \mid W_1 = s  ) \times \ t_s \\
& \qquad= (t_s)^{i-1}\,.
\end{align*}

Now,  
we  define $v_0 = 1$ and for $1 \leq i \leq 2Q-1$:
\begin{align*}
 v_i &= \mathbb{P}(A_{121} = 1, A_{341} = 1, \dots, A_{2i \, 2i+1 \,  1} = 1 ) \\
 &= \sum_s \mathbb{P}(A_{121} = 1, A_{341} = 1, \dots, A_{2i \, 2i+1 \,  1} = 1 \mid W_1 = s ) \ \rho_s \\
 &= \sum_s (t_s)^i \ \rho_s .
\end{align*}

By Assumption~$\mathcal{A}$4, $(v_i)_{i=1:(2Q-1)}$ are well defined. 
Hence, $v_0 = 1$ and 
 $(v_i)_{i=1:(2Q-1)}$ are known and defined from the marginal $\mathbb{P}_A$. 
As a consequence, 
$(v_i)_{i=1:(2Q-1)}$ are identifiable.
\medskip

Also, let $\tilde{\mathbf{M}}$ be the matrix of size $(Q+1) \times Q$ given by $\tilde{M}_{ij} = v_{i+j-2}$ for $1 \leq i \leq Q+1$ 
and $1 \leq j \leq Q$,  
and let $\tilde{\mathbf{M}}_{-i}$ denote the square matrix obtained by removing the row $i$ from $\tilde{\mathbf{M}}$. 
The coefficients of $\tilde{\mathbf{M}}_{-(K+1)}$,  for $1 \leq i,j \leq Q$, are:
\begin{align}
\tilde{M}_{ij} &  = \sum_{s =1}^Q (t_s)^{i-1} \ \rho_s\ (r_s)^{j-1}\, , 
\ \text{and } \nonumber \\
\tilde{\mathbf{M}}_{-(Q+1)} & = \mathbf{T} \ \operatorname{Diag}(\boldsymbol{\rho})\ \mathbf{T}^T\, . \label{eq:M_rho}
\end{align}

\begin{proposition}[Relations between $\mathbf{T}$, $\tilde{\mathbf{M}}$ and ${\boldsymbol{\rho}}$]
\label{prop:RMrho}
From Proposition~\ref{prop:T} and Equation~\eqref{eq:M_rho}, we can define
\begin{align*}
\tilde{\mathbf{M}}_{-(Q+1)} = \mathbf{T} \  \operatorname{Diag}(\boldsymbol{\rho})\ \mathbf{T}^T \, .
\end{align*}
\end{proposition}

The correspondence of the different terms being established, we now need to prove the identifiability of $\boldsymbol{\rho}$, which means showing that $\tilde{\mathbf{M}}_{-(Q+1)}$ and $\mathbf{T}$ are identifiable.
\medskip

As for the previous proof regarding the identifiability of $t_s$, with $\delta_s = \operatorname{Det}(\tilde{\mathbf{M}}_{-s})$, we define a polynomial function $\tilde{B}$ such as: 
\begin{align*}
\tilde{B}(x) = \sum_{s=0}^Q (-1)^{Q+s} \ \delta_{s+1} \ x^s
\end{align*}
This polynomial function has again two important properties summarized in the following proposition.
\begin{proposition}
\label{prop:deg_root_B_tilde} Let $\operatorname{deg}(\tilde{B})$ denote the degree of $\tilde{B}$. We have  $\operatorname{deg}(\tilde{B}) = Q$ and
 $\tilde{B}(t_s) = 0$, for $ 1 \leq s  \leq Q$.
\end{proposition}

\begin{proof} The proof follow the same lines as  those of Proposition~\ref{prop:deg_B} and Proposition~\ref{prop:root_B}.
\label{proof:deg_root_B_tilde}
\end{proof}  

With $(t_s)_{s=1:Q}$ 
being the roots of $\tilde{B}$ (proposition \ref{prop:deg_root_B_tilde}), they are functions of $(\delta_s)_{s=1:Q+1}$ which are themselves derived from $\mathbb{P}_A$. 
Also, $(t_s)_{s=1:Q}$ can be expressed in a unique way (up to label switching) from $\mathbb{P}_A$, thus $(t_s)_{s=1:Q}$ are identifiable. In consequence, $\mathbf{T}$ is also identifiable by definition. 
Finally, since $\tilde{\mathbf{M}}_{-(Q+1)}$ and $\mathbf{T}$ are identifiable and invertible, 
\(\displaystyle\operatorname{Diag}(\boldsymbol{\rho})  = \mathbf{T}^{-1} \tilde{\mathbf{M}}_{-(Q+1)} (\mathbf{T}^T)^{-1} \). 
In conclusion, $\boldsymbol{\rho}$ is identifiable.

\subsection{Identifiability of \texorpdfstring{$\boldsymbol{\alpha}$}{}}

To establish the identifiability of $\boldsymbol{\alpha}$, the initial proof relies on matrix inversion. However, within our framework, tensor inversion is not as straightforward as uniqueness may not be inherently guaranteed.
To overcome this issue, we will reparametrize our problem to revert to a matrix-based formulation. To do this, we shift from utilizing the reference frame of nodes (individuals) to that of edges (connections).

%

First, the tensor $\mathbf{A}$ is transformed into a matrix $\tilde{\mathbf{A}}$ of size $\tilde{N} \times Q$, with $\tilde{N} = N(N-1)/2$ in an undirected framework. Each column corresponds to a vectorization of the upper triangular matrix  of each layer of $\mathbf{A}$. Thus, the edge $\mathbf{A}_{ijv}$ will be described by $\tilde{\mathbf{A}}_{\tilde{i}v}$, with $\tilde{i}$ being the index corresponding to the edge $(i,j)$ between nodes $i$ and $j$.

Then, we can map the clustering of observations into a clustering of edges, which results 
in a matrix $\tilde{\mathbf{Z}}$ of size $\tilde{N} \times \tilde{K}$, with $\tilde{K} = K(K+1)/2$. Each row of the matrix corresponds to the clustering of the pair of nodes making up the edges $1 \leq \tilde{i} \leq \tilde{N}$.

Also, we denote $\tilde{\boldsymbol{\pi}}$ the proportion vector of pairs such that
$\tilde{\boldsymbol{\pi}}_{\tilde{k}} = \boldsymbol{\pi}_{k} \boldsymbol{\pi}_{l}$, for $1 \leq \tilde{k} \leq \tilde{K}$, 
%
where $1 \leq k,l \leq K$ are the initial clusters corresponding to the index of the $\tilde{k}$ in the reparametrization.


Finally, $\tilde{\boldsymbol{\alpha}}$ is a $\tilde{K} \times Q$ matrix whose rows represents the clusters related to the pairs of nodes while the columns are the components. The terms of $\tilde{\boldsymbol{\alpha}}$ represent the probabilities of connection between these clusters and components.  
\medskip

Now, let's define a function $\phi$ such as:
\begin{align*}
\phi(\mathbf{A},\mathbf{Z},\boldsymbol{\pi},\boldsymbol{\alpha}) = (\tilde{\mathbf{A}}, \tilde{\mathbf{Z}},\tilde{\boldsymbol{\pi}},\tilde{\boldsymbol{\alpha}})\,.
\end{align*}

The function is bijective for $\mathbf{A}$ and $\boldsymbol{\alpha}$ and injective for $\mathbf{Z}$ and $\boldsymbol{\pi}$ ; 
The bijective relationship involving the parameter $\boldsymbol{\alpha}$ and $\tilde{\boldsymbol{\alpha}}$ enables the establishment of identifiability. The aim is therefore to show the identifiability of $\tilde{\boldsymbol{\alpha}}$.

\begin{remark}
These transformations map our problem into a LBM framework (see Figure~\ref{fig: proof lbm}). Hence, the identifiability of $\tilde{\boldsymbol{\alpha}}$ will be developed accordingly.
\end{remark}

\begin{figure}[!ht]
     \centering
     \begin{subfigure}[b]{0.49\textwidth}
         \centering
         \includegraphics[width=\textwidth]{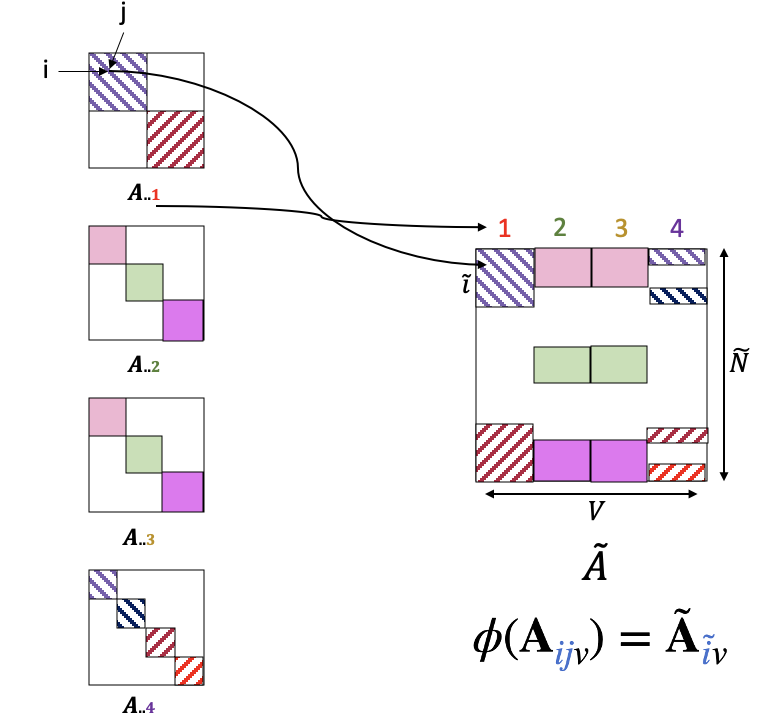}
     \end{subfigure}
     \hfill
     \begin{subfigure}[b]{0.49\textwidth}
         \centering
         \includegraphics[width=\textwidth]{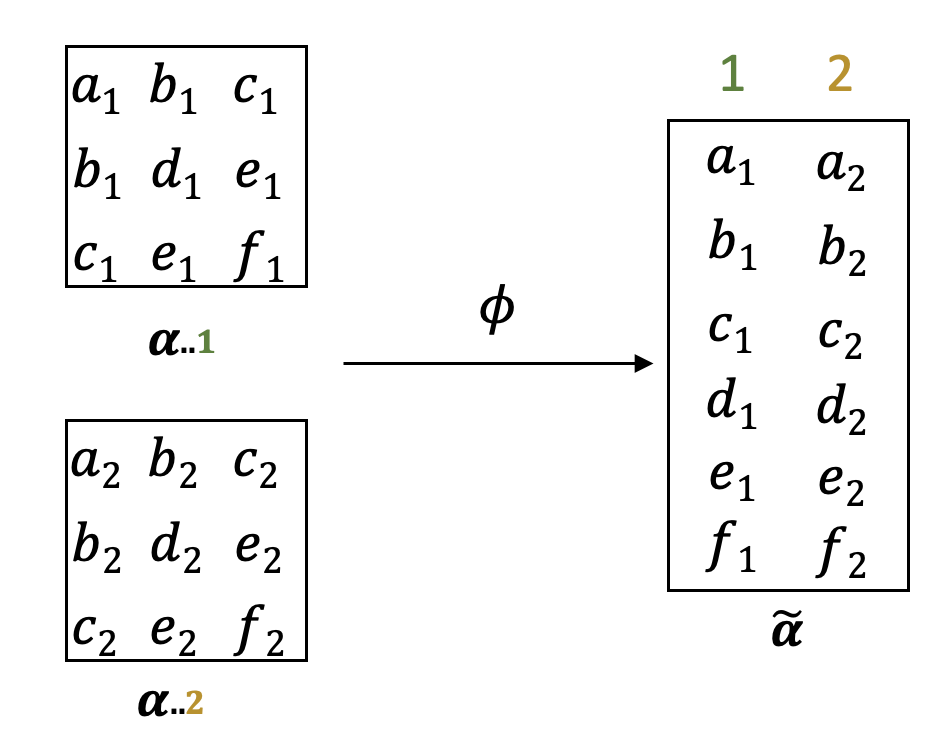}
     \end{subfigure}
     \caption{Illustration of transformation on $\mathbf{A}$ and $\boldsymbol{\alpha}$. }
     
     \label{fig: proof lbm}
\end{figure}

The proof is identical to the ones of Section~\ref{sec:id_pi} 
For any \(1 \leq \tilde{k} \leq \tilde{K}\) and $\forall i,j$, let's define:
\begin{align*}
\tilde{r}_{\tilde{k}} &= \mathbb{P}(\tilde{A}_{ij} = 1 \mid \tilde{Z}_i = \tilde{k} ) \\
&= \sum_s \mathbb{P}(\tilde{A}_{ij} = 1 \mid \tilde{Z}_i = \tilde{k}, W_v = s ) \ \rho_s \\
&=  \sum_s \tilde{\alpha}_{\tilde{k}s} \rho_s \\
&=   (\tilde{\boldsymbol{\alpha}} \boldsymbol{\rho})_{\tilde{k}}\,. 
\end{align*}

\begin{proposition}[Invertibility of $\tilde{\mathbf{R}}$]
\label{prop:R_tilde}
Let  $\tilde{\mathbf{R}}$ denote a Vandermonde matrix of size $\tilde{K} \times \tilde{K}$ such as 
\(\displaystyle \tilde{R}_{i\tilde{k}} = (\tilde{r}_{\tilde{k}})^{i-1} \text{, for } 1\leq i \leq \tilde{K} \text{ and }  1\leq \tilde{k}\leq \tilde{K} \). 
$\tilde{\mathbf{R}}$ is invertible, since the coordinates of $r$ are all different according to Assumption~$\mathcal{A}5$.
\end{proposition}

The rest of the proof is identical to the one of Section~\ref{sec:id_pi} so that it can be show that $\mathbf{\tilde{R}}$ is identifiable and so $\tilde{\boldsymbol{\pi}}_{\tilde{k}}$.
\medskip

We now focus on the the identifiability of $\tilde{\boldsymbol{\alpha}}$. Let  $\mathbf{U}$ be a matrix of size  $\tilde{K} \times Q$ such that the $(i,j)$ entry of is the joint probability of having $i$ connections in the first row and $j-1$ connections in the first column:
\begin{align}
\mathbf{U}_{ij} &= \mathbb{P}( \tilde{\mathbf{A}}_{11}=1,\tilde{\mathbf{A}}_{12}=1,\dots,\tilde{\mathbf{A}}_{1i}=1,\tilde{\mathbf{A}}_{21}=1,\dots,\tilde{\mathbf{A}}_{j1}=1 ) \notag \\
&= \sum_{\tilde{k}} \sum_q \tilde{{\pi}}_{\tilde{k}} \, \rho_s \,  \mathbb{P}( \tilde{\mathbf{A}}_{11}=1,\tilde{\mathbf{A}}_{12}=1,\dots,\tilde{\mathbf{A}}_{1i}=1,\tilde{\mathbf{A}}_{21}=1,\dots,\tilde{\mathbf{A}}_{j1}=1 \mid \tilde{\mathbf{Z}}_1 = \tilde{k}, \mathbf{W}_1 = s ) \notag \\
&= \sum_{\tilde{k}} \sum_q \tilde{{\pi}}_{\tilde{k}} \, \rho_s \,  \tilde{{\alpha}}_{\tilde{k}s} \, \mathbb{P}(\tilde{\mathbf{A}}_{12}=1,\dots,\tilde{\mathbf{A}}_{1i}=1,\tilde{\mathbf{A}}_{21}=1,\dots,\tilde{\mathbf{A}}_{j1}=1 \mid \tilde{\mathbf{Z}}_1 = \tilde{k}, \mathbf{W}_1 = s ) \notag \\
&= \sum_{\tilde{k}} \sum_q \tilde{{\pi}}_{\tilde{k}} \, \rho_s \,  \tilde{{\alpha}}_{\tilde{k}s} \, \tilde{r}_{\tilde{k}}^{i-1} t_s^{j-1} . 
\label{eq:U}
\end{align}
%


\begin{proposition}[Relations between  $\tilde{\mathbf{R}}$, ${\mathbf{T}}$,  $\mathbf{U}$, $\tilde{\boldsymbol{\alpha}}$, ${\boldsymbol{\tilde{\pi}}}$ and ${\boldsymbol{\boldsymbol{\rho}}}$]

\label{prop:RMrho2}
From Proposition~\ref{prop:R_tilde} and Equation~\eqref{eq:U}, 
we can define 
\(\mathbf{U} = \tilde{\mathbf{R}}  \operatorname{Diag}(\tilde{\boldsymbol{\pi}}) \, \tilde{\boldsymbol{\alpha}} \, \operatorname{Diag}({\boldsymbol{\rho}}) \, \mathbf{T}^T
\), 
with ${\mathbf{U}}$, $\tilde{\mathbf{R}}$, $\operatorname{Diag}(\tilde{\boldsymbol{\pi}})$, $\operatorname{Diag}({\boldsymbol{\rho}})$ and ${\mathbf{T}}$ being invertible. 
Therefore, 
\begin{align*}
\tilde{\boldsymbol{\alpha}} = \tilde{\mathbf{R}}^{-1} \operatorname{Diag}(\tilde{\boldsymbol{\pi}})^{-1} \, \mathbf{U} \, \operatorname{Diag}(\rho)^{-1} \, (\mathbf{T}^T)^{-1}\,.
\end{align*}
\end{proposition}

In addition to Proposition~\ref{prop:RMrho2}, $\mathbf{U}$ being defined from $\mathbb{P}_{\tilde{\mathbf{A}}}$,  all its coefficients are identifiable. As a consequence, $\tilde{\boldsymbol{\alpha}}$ is identifiable.  In conclusion, ${\boldsymbol{\alpha}} = \phi^{-1}(\tilde{\boldsymbol{\alpha}})$ is identifiable.

\newpage
\section{Details of VBEM algorithm}

\subsection{Variational parameters of clustering \texorpdfstring{$\tau_{ik}$}{tau\_{ik}}}

The optimal approximation for $q(\bZ_i)$ is 
\begin{equation*}
q(\bZ_i) = \mathcal{M}(\bZ_i; (\tau_{i1},\dots,\tau_{iK})),
\end{equation*}

where $\tau_{ik}$ is the probability of node $i$ to belong to class $k$. It satisfies the relation

\begin{equation*}
\tau_{ik} \propto  e^{\psi(\beta_k) - \psi(\sum_{k'} \beta_{k'}) }   \prod_{j \ne i}^N  \prod_{l=1}^K   \prod_{v=1}^V  \prod_{s=1}^Q e ^{\tau_{jl} \, \nu_{vs}  \Big[ A_{ijv} \Big( 
\psi(\eta_{kls}) - \psi(\xi_{kls})  
\Big) + \psi(\xi_{kls} ) - \psi( \eta_{kls} + \xi_{kls} ) \Big]},
\end{equation*}
where $\psi$ is digamma function. Distribution $q(\bZ)$ is optimized with a fixed point algorithm.

\begin{proof}
According to the model, the optimal distribution $q(\bZ_i$) is given by
\begin{equation*}
\begin{aligned}
\log q(\bZ_i) &= \mathbb{E}_{\bZ^{\backslash i},\alpha,\pi,W,\rho} \left[ \log \p (\bA,\bZ,\bW,\balpha,\bpi,\brho) \right] \\
&\propto  \mathbb{E}_{\bZ^{\backslash i},\balpha,\bW} [ \log \mathbb{P} (\bA|\bZ,\bW,\balpha)] + \mathbb{E}_{\bZ^{\backslash i},\bpi} [ \log \mathbb{P} (\bZ|\pi)] \\
&\propto  \mathbb{E}_{\bZ^{\backslash i},\balpha,\bW} \left[ \sum_{i'=1, j > i'}^N  \sum_{k,l=1}^K   \sum_{v=1}^V  \sum_{s=1}^Q \mathbb{1}_{\bZ_{i'},\bZ_j,\bW_v}  \Big( \log \mathbb{P}(A_{i'jv} | Z_{i'} = k,Z_j = l,W_v = s ,\balpha) \Big) \right] \\
&\quad + \mathbb{E}_{\bZ^{\backslash i},\bpi} \left[ \sum_{i'=1}^N  \sum_{k}^K  \log \mathbb{P}(Z_{i'}=k|\bpi) \right]\\
&\propto  \sum_k \mathbb{1}_{Z_i=k} \Big\{ \mathbb{E}_{\bpi}[\log(\pi_k)] \ + \sum_{j \ne i}^N  \sum_{l=1}^K   \sum_{v=1}^V  \sum_{s=1}^Q \tau_{jl} \, \nu_{vs}  \ \mathbb{E}_{\balpha} \Big[ A_{ijv} \log(\alpha_{kls}) + (1- A_{ijv}) \log(1- \alpha_{kls}) \Big] \Big\} \\
\end{aligned}.
\end{equation*}
Remember that :
\begin{itemize}
    \item $\bpi \sim Dir(\bpi;\bbeta)$, so $\pi_k \sim Beta(\pi_k;\beta_k,\sum_{k'} \beta_{k'} - \beta_k)$ ;
    \item $\mathbb{E}_{\bpi}[\log(\pi_k)] = \psi(\beta_k) - \psi(\sum_{k'} \beta_{k'})$;
    \item $q(\alpha_{kls}) = Beta(\alpha_{kls}; \eta_{kls},\xi_{kls})$ ;
    \item $\mathbb{E}_{\alpha_{kls}}[\log(\alpha_{kls})] = \psi(\eta_{kls}) - \psi(\xi_{kls}  + \eta_{kls})$;
    \item $\mathbb{E}_{\alpha_{kls}}[\log(1- \alpha_{kls})] = \psi(\xi_{kls} ) - \psi( \eta_{kls} + \xi_{kls}).$
\end{itemize}
In consequence,
\begin{equation*}
\begin{aligned}
\log q(Z_i) &\propto  \sum_k \mathbb{1}_{Z_i=k} \Big\{ \psi(\beta_k) - \psi(\sum_{k'} \beta_{k'}) + \ 
\sum_{j \ne i}^N  \sum_{l=1}^K   \sum_{v=1}^V  \sum_{s=1}^Q \tau_{jl} \, \nu_{vs}  \Big[ A_{ijv} \Big( 
( \psi(\eta_{kls}) - \psi(\xi_{kls}  + \eta_{kls})) \\
&\quad - (\psi(\xi_{kls} ) - \psi( \eta_{kls} + \xi_{kls} )) \Big) + \psi(\xi_{kls} ) - \psi( \eta_{kls} + \xi_{kls} ) \Big] \Big\} \\
&= \sum_k \mathbb{1}_{Z_i=k} \Big\{ \psi(\beta_k) - \psi(\sum_{k'} \beta_{k'}) \ +
\sum_{j \ne i}^N  \sum_{l=1}^K   \sum_{v=1}^V  \sum_{s=1}^Q \tau_{jl} \, \nu_{vs}
\Big[ A_{ijv} \Big( 
\psi(\eta_{kls}) - \psi(\xi_{kls} ) 
\Big) \\
&\quad + \psi(\xi_{kls} ) - \psi( \eta_{kls} + \xi_{kls} ) \Big] \Big\} .
\end{aligned}
\end{equation*}
We can therefore deduce that, by applying the exponential :
\begin{equation*}
\begin{aligned}
q(Z_{i}=k) &\propto  e^{\psi(\beta_k) - \psi(\sum_{k'} \beta_{k'}) +\sum_{j \ne i}^N  \sum_{l=1}^K   \sum_{v=1}^V  \sum_{s=1}^Q \tau_{jl} \, \nu_{vs}   
\Big[ A_{ijv} \Big( 
\psi(\eta_{kls}) - \psi(\xi_{kls}  
\Big) + \psi(\xi_{kls} ) - \psi( \eta_{kls} + \xi_{kls} ) \Big] }\\
& = e^{\psi(\beta_k) - \psi(\sum_{k'} \beta_{k'}) }   \prod_{j \ne i}^N  \prod_{l=1}^K   \prod_{v=1}^V  \prod_{s=1}^Q e ^{\tau_{jl} \, \nu_{vs}  \Big[ A_{ijv} \Big( 
\psi(\eta_{kls}) - \psi(\xi_{kls})  
\Big) + \psi(\xi_{kls} ) - \psi( \eta_{kls} + \xi_{kls} ) \Big]} \\
\end{aligned}.
\end{equation*}
Therefore,
\begin{equation*}
\tau_{ik} \propto  e^{\psi(\beta_k) - \psi(\sum_{k'} \beta_{k'}) }   \prod_{j \ne i}^N  \prod_{l=1}^K   \prod_{v=1}^V  \prod_{s=1}^Q e ^{\tau_{jl} \, \nu_{vs}  \Big[ A_{ijv} \Big( 
\psi(\eta_{kls}) - \psi(\xi_{kls})  
\Big) + \psi(\xi_{kls} ) - \psi( \eta_{kls} + \xi_{kls} ) \Big]}.
\end{equation*}
So,
\begin{equation*}
q(\bZ_i) = \mathcal{M}(\bZ_i; (\tau_{i1},\dots,\tau_{iK})).
\end{equation*}
\end{proof}


\subsection{Variational parameters of component membership \texorpdfstring{$\nu_{vs}$}{nu\_vs}}

The optimal approximation for $q(\bW_v)$ is 

\begin{equation*}
q(W_v) = \mathcal{M}(W_v; (\nu_{v1},\dots,\nu_{vQ})),
\end{equation*}

with

\begin{equation*}
\begin{aligned}
\nu_{vs} &\propto  e^{\psi(\theta_s) - \psi(\sum_{s'} \theta_{s'}) }   \prod_{i \ne j}^N  \prod_{k\ne l}^K  e ^{\tau_{ik} \, \tau_{jl}  \Big[ A_{ijv} \Big( 
\psi(\eta_{kls}) - \psi(\xi_{kls}  
\Big) + \psi(\xi_{kls} ) - \psi( \eta_{kls} + \xi_{kls} ) \Big]}  \\
&\quad \prod_{k}^K \prod_{i<j}^N   e ^{\tau_{ik} \, \tau_{jk}  \Big[ A_{ijv} \Big( 
\psi(\eta_{kks}) - \psi(\xi_{kks}  
\Big) + \psi(\xi_{kks} ) - \psi( \eta_{kks} + \xi_{kks} ) \Big] }.
\end{aligned}
\end{equation*}
$\nu_{vs}$ is the probability of layer $v$ to belong to component $s$.

\begin{proof}
As previously mentioned, in accordance with the principles of variational Bayes, the optimal probability distribution can be expressed as follows:
\begin{equation*}
\begin{aligned}
\log q(\bW_v) &= \mathbb{E}_{\bW^{\backslash v},\balpha,\bpi,\bZ,\brho} [ \log \mathbb{P} (\bA,\bZ,\bW,\balpha,\bpi,\brho)] \\
&\propto  \mathbb{E}_{\bW^{\backslash v},\balpha,\bZ} [ \log \mathbb{P} (\bA|\bZ,\bW,\balpha)] + \mathbb{E}_{\bW^{\backslash v},\brho} \left[ \ln \mathbb{P} (\bW|\rho) \right] \\
&\propto    \mathbb{E}_{W^{\backslash v},\balpha,\bZ} \big[ \sum_{i=1, j > i}^N  \sum_{k,l=1}^K   \sum_{v=1}^V  \sum_{s=1}^Q \mathbb{1}_{\bZ_i,\bZ_j,W_v}  \Big( \log \mathbb{P}(A_{ijv} | Z_i = k,Z_j = l,W_v = s ,\balpha) \Big)] \\
&+ \mathbb{E}_{\bW^{\backslash v},\brho} [ \sum_{v=1}^V  \sum_{s}^Q  \log \mathbb{P}(W_v=s|\brho) \big]\\
&\propto  \sum_q \mathbb{1}_{W_v=s} \Big\{ \mathbb{E}_{\rho}[\log(\rho_s)] \ + \sum_{k\ne l}^K \sum_{i=1,j \ne i}^N      \tau_{ik} \, \tau_{jl}  \  \mathbb{E}_{\balpha} \Big[ A_{ijv} \log(\alpha_{kls}) + (1- A_{ijv}) \log(1- \alpha_{kls}) \Big] \\
&+ \sum_{k}^K \sum_{i=1,i < j }^N     \tau_{ik} \, \tau_{jk}  \ \mathbb{E}_{\balpha} \Big[ A_{ijv} \log(\alpha_{kks}) + (1- A_{ijv}) \log(1- \alpha_{kks}) \Big]
\Big\}.
\end{aligned}
\end{equation*}
Reminder : 
\begin{itemize}
    \item $\rho \sim Dir(\pi;\theta)$,  so $\rho_s \sim Beta(\rho_s;\theta_s,\sum_{s'} \theta_{s'} - \theta_s)$;
    \item $\mathbb{E}_{\rho}[\log(\rho_s)] = \psi(\theta_s) - \psi(\sum_{s'} \theta_{s'})$ .
\end{itemize} 
Hence,
\begin{equation*}
\begin{aligned}
\log q(W_v) &\propto  \sum_q \mathbb{1}_{W_v=q} \Big\{ \psi(\theta_s) - \psi(\sum_{s'} \theta_{s'}) \ +
\sum_{i=1,j > i}^N  \sum_{k,l=1}^K  \tau_{ik} \, \tau_{jl}  \Big[ A_{ijv} \Big( 
( \psi(\eta_{kls}) - \psi(\xi_{kls}  + \eta_{kls})) \\
&\quad - (\psi(\xi_{kls} ) - \psi( \eta_{kls} + \xi_{kls} )) \Big) + \psi(\xi_{kls} ) - \psi( \eta_{kls} + \xi_{kls} ) \Big] \Big\} \\
&= \sum_q \mathbb{1}_{W_v=q} \Big\{ \psi(\theta_s) - \psi(\sum_{s'} \theta_{s'}) \ +
\sum_{i=1,j > i}^N  \sum_{k,l=1}^K   \tau_{ik} \, \tau_{jl}   
\Big[ A_{ijv} \Big( 
\psi(\eta_{kls}) - \psi(\xi_{kls})  
\Big) \\
&\quad + \psi(\xi_{kls} ) - \psi( \eta_{kls} + \xi_{kls} ) \Big] \Big\}.
\end{aligned}
\end{equation*}
Consequently,
\begin{equation*}
\begin{aligned}
q(W_{v}=q) &\propto  e^{\psi(\theta_s) - \psi(\sum_{s'} \theta_{s'}) \ +
\sum_{i=1,j > i}^N  \sum_{k,l=1}^K   \tau_{ik} \, \tau_{jl}   
\Big[ A_{ijv} \Big( 
\psi(\eta_{kls}) - \psi(\xi_{kls}  
\Big) + \psi(\xi_{kls} ) - \psi( \eta_{kls} + \xi_{kls} ) \Big] }\\
& = e^{\psi(\theta_s) - \psi(\sum_{s'} \theta_{s'}) }  \prod_{k\ne l}^K  \prod_{i=1,j \ne i}^N   e ^{\tau_{ik} \, \tau_{jl}  \Big[ A_{ijv} \Big( 
\psi(\eta_{kls}) - \psi(\xi_{kls}  
\Big) + \psi(\xi_{kls} ) - \psi( \eta_{kls} + \xi_{kls} ) \Big] }\\
& \quad \prod_{k}^K \prod_{i<j}^N   e ^{\tau_{ik} \, \tau_{jk}  \Big[ A_{ijv} \Big( 
\psi(\eta_{kks}) - \psi(\xi_{kks}  
\Big) + \psi(\xi_{kks} ) - \psi( \eta_{kks} + \xi_{kks} ) \Big] }.
\end{aligned}
\end{equation*}
So,
\begin{equation*}
\begin{aligned}
\nu_{vs} &\propto  e^{\psi(\theta_s) - \psi(\sum_{s'} \theta_{s'}) }   \prod_{i \ne j}^N  \prod_{k\ne l}^K  e ^{\tau_{ik} \, \tau_{jl}  \Big[ A_{ijv} \Big( 
\psi(\eta_{kls}) - \psi(\xi_{kls})  
\Big) + \psi(\xi_{kls} ) - \psi( \eta_{kls} + \xi_{kls} ) \Big]}  \\
&\quad \prod_{k}^K \prod_{i<j}^N   e ^{\tau_{ik} \, \tau_{jk}  \Big[ A_{ijv} \Big( 
\psi(\eta_{kks}) - \psi(\xi_{kks}  
\Big) + \psi(\xi_{kks} ) - \psi( \eta_{kks} + \xi_{kks} ) \Big] }
\end{aligned},
\end{equation*}
and 
\begin{equation*}
q(W_v) = \mathcal{M}(W_v; (\nu_{v1},\dots,\nu_{vQ})).
\end{equation*}
\end{proof}
\subsection{Optimization of \texorpdfstring{$q(\bpi)$ ($\beta_k$)}{q(pi) (beta\_k)}}
Due to the selection of prior distributions, the distribution $q(\bpi)$ remains within the same family of distributions as the prior distribution $\p(\bpi)$.

\begin{equation*}
q(\boldsymbol{\pi}) = \mathrm{Dir}(\boldsymbol{\pi};\boldsymbol{\beta}),
\end{equation*}
with 

\begin{equation*}
\beta_k = \beta_{k}^0 +\sum_{i}^N \tau_{i k}.
\end{equation*}

\begin{proof}
The optimal probability distribution can be formulated in the following manner:
\begin{equation*}
\begin{aligned}
\log q(\boldsymbol{\pi}) &\propto \mathbb{E}_{\bW,\balpha,\bZ,\brho} [ \log \mathbb{P} (\bA,\bZ,\bW,\balpha,\bpi,\brho)] \\
&\propto \mathbb{E}_{\bZ}[\log \p(\mathbf{Z} \mid \boldsymbol{\pi})] + \log p(\boldsymbol{\pi}) \\
&\propto \sum_{i}^N \sum_{k}^K  \tau_{i k} \log \pi_k  +\sum_{k =1 }^K \left(\beta_{k}^0-1\right) \log \pi_k\\
&\propto  \sum_{k}^K  \left( \beta_{k}^0 + (\sum_{i}^N \tau_{i k} ) -1\right) \log \pi_k\\
\end{aligned}\ .
\end{equation*}
After exponentiation and normalization, we obtain:
\begin{equation*}
q(\boldsymbol{\pi}) = \mathrm{Dir}(\boldsymbol{\pi};\boldsymbol{\beta}),
\end{equation*}
with
\begin{equation*}
\beta_k = \beta_{k}^0 +\sum_{i}^N \tau_{i k}.
\end{equation*}
\end{proof} 
 \subsection{Optimization of \texorpdfstring{$q(\brho)$ ($\theta_s$)}{q(rho) (theta\_s)}}

As previously mentioned, the selection of prior distributions enables us to remain within the same family of distributions.

\begin{equation*}
q(\boldsymbol{\rho}) = \mathrm{Dir}(\boldsymbol{\rho};\boldsymbol{\theta}),
\end{equation*}
with
\begin{equation*}
\theta_s = \theta_{s}^0 +\sum_{v=1}^V \nu_{v s}.
\end{equation*}

\begin{proof}
According to variational Bayes, the optimal probability distribution can be expressed as follows:
\begin{equation*}
\begin{aligned}
\log q(\boldsymbol{\rho}) &\propto \mathbb{E}_{\bW,\balpha,Z} [ \log \mathbb{P} (\bA,\bZ,\bW,\balpha,\bpi,\brho)] \\
&\propto \mathrm{E}_{\mathbf{W}}[\log p(\mathbf{W} \mid \boldsymbol{\rho})] + \log \p(\boldsymbol{\rho}) \\
&\propto \sum_{v}^V \sum_{s}^Q  \nu_{v s} \log \rho_s  +\sum_{q =1 }^Q \left(\theta_{s}^0-1\right) \log \rho_s\\
&\propto  \sum_{s}^Q  \left( \theta_{s}^0 + (\sum_{v}^V \nu_{v s} ) -1\right) \log \rho_s\\
\end{aligned}.
\end{equation*}
After exponentiation and normalization, we have
\begin{equation*}
q(\boldsymbol{\rho}) = \mathrm{Dir}(\boldsymbol{\rho};\boldsymbol{\theta}),
\end{equation*}
with
\begin{equation*}
\theta_s = \theta_{s}^0 +\sum_{v=1}^V \nu_{v s}.
\end{equation*}
\end{proof}


\subsection{Optimization of \texorpdfstring{$q(\balpha)$ ($\eta_{kls}$ and $\xi_{kls}$)}{q(alpha) (eta\_kls and xi\_kls)}}

Once again, the distribution form of the prior distribution $\p(\balpha)$ is preserved through the variational optimization process.

\begin{equation*}
q(\alpha_{k l s } ) = \mathrm{Beta} (\alpha_{k l s }; \eta_{kls} , \xi_{kls}).
\end{equation*}
When $k \ne l$, parameters $\eta_{kls}$ and $\xi_{kls}$ are given by:
\begin{equation*}
\begin{aligned}
\eta_{kls} &= \eta_{k l s}^0+\sum_{i \ne j}^N \sum_v^V \tau_{i k} \tau_{j l} \nu_{v s} A_{i j v}\\
\xi_{kls} &= \xi_{k l s}^0+\sum_{i \ne j}^N \sum_v^V \tau_{i k} \tau_{j l} \nu_{v s} \left(1-A_{i j v}\right)
\end{aligned}.
\end{equation*}
Otherwise, when $k$ equals $l$, the parameters $\eta_{kks}$ and $\xi_{kks}$  are determined by:
\begin{equation*}
\begin{aligned}
\eta_{kks} &= \eta_{k k s}^0+\sum_{i < j}^N \sum_v^V \tau_{i k} \tau_{j k} \nu_{v s} A_{i j v}\\
\xi_{kks} &= \xi_{k k s}^0+\sum_{i < j}^N \sum_v^V \tau_{i k} \tau_{j k} \nu_{v s} \left(1-A_{i j v}\right)
\end{aligned}.
\end{equation*}

\begin{proof}
In accordance with the principles of variational Bayes, the optimal probability distribution can be formulated as follows:
\begin{equation*}
\begin{aligned}
 \log q(\boldsymbol{\balpha}) &\propto \mathrm{E}_{\mathbf{Z},\bW}[\log \p(\mathbf{A}, \mathbf{Z}, \boldsymbol{\alpha}, \mathbf{W})] \\
&\propto \mathrm{E}_{\mathbf{Z},\bW}[\log p(\mathbf{A} \mid \mathbf{Z},\bW, \boldsymbol{\alpha})]+\log \p(\boldsymbol{\alpha}) \\
&= \sum_{i<j}^N \sum_{k, l}^K \sum_{v}^V \sum_{s}^Q \tau_{i k} \tau_{j l} \nu_{v s} \left(A_{i j v} \log (\alpha_{k l s})+\left(1-A_{i j v}\right) \log \left(1-\alpha_{k l s}\right)\right) \\
& \quad+\sum_{k \leq l}^K \sum_{s}^Q \left(\left(\eta_{k l s}^0-1\right) \log (\alpha_{k l s})+\left(\xi_{k l s}^0-1\right) \log \left(1-\alpha_{k l s}\right)\right)\\
&= \sum_{k<l}^K \sum_{i \neq j}^N \sum_{v}^V \sum_{s}^Q \tau_{i k} \tau_{j l} \nu_{v s} \left(A_{i j v} \log (\alpha_{k l s})+\left(1-A_{i j v}\right) \log \left(1-\alpha_{k l s}\right)\right) \\
& \quad +\sum_{k=1}^K \sum_{i<j}^N \sum_{v}^V \sum_{s=1}^Q  \tau_{i k} \tau_{j k} \nu_{v s}\left(A_{i j v} \log (\alpha_{k k s})+\left(1-A_{i j v}\right) \log \left(1-\alpha_{k k s}\right)\right) \\
& \quad+\sum_{k \leq l}^K \sum_{s}^Q \left(\left(\eta_{k l s}^0-1\right) \log (\alpha_{k l s})+\left(\xi_{k l s}^0-1\right) \log \left(1-\alpha_{k l s}\right)\right)\\
&= \sum_{k<l}^K \sum_{s}^Q  \left(\eta_{k l s}^0-1+\sum_{i \ne j}^N \sum_v^V \tau_{i k} \tau_{j l} \nu_{v s} A_{i j v}\right) \log (\alpha_{k l s}) \,+ \\
&\quad \quad \left(\xi_{k l s}^0-1+\sum_{i \ne j}^N \sum_v^V \tau_{i k} \tau_{j l} \nu_{v s} \left(1-A_{i j v}\right)\right) \log \left(1-\alpha_{k l s}\right) \\
& \quad+\sum_{k=1}^K \sum_{s}^Q  \left(\eta_{k k s}^0-1+\sum_{i<j}^N \sum_v^V \tau_{i k} \tau_{j k} \nu_{v s} A_{i jv}\right) \log \alpha_{k k s} \,+\\
&\quad \quad \left(\xi_{k k s}^0-1+\sum_{i<j}^N \sum_v^V \tau_{i k} \tau_{j k} \nu_{v s} \left(1-A_{i j v}\right)\right) \log \left(1-\alpha_{k k s}\right)
\end{aligned}.
\end{equation*}
Therefore,
\begin{equation*}
q(\alpha_{k l s } ) = \mathrm{Beta} (\alpha_{k l s }; \eta_{kls} , \xi_{kls}),
\end{equation*}
if $k \ne l$,
\begin{equation*}
\begin{aligned}
\eta_{kls} &= \eta_{k l s}^0+\sum_{i \ne j}^N \sum_v^V \tau_{i k} \tau_{j l} \nu_{v s} A_{i j v}\\
\xi_{kls} &= \xi_{k l s}^0+\sum_{i \ne j}^N \sum_v^V \tau_{i k} \tau_{j l} \nu_{v s} \left(1-A_{i j v}\right)
\end{aligned};
\end{equation*}
otherwise,
\begin{equation*}
\begin{aligned}
\eta_{kks} &= \eta_{k k s}^0+\sum_{i < j}^N \sum_v^V \tau_{i k} \tau_{j k} \nu_{v s} A_{i j v}\\
\xi_{kks} &= \xi_{k k s}^0+\sum_{i < j}^N \sum_v^V \tau_{i k} \tau_{j k} \nu_{v s} \left(1-A_{i j v}\right)
\end{aligned}.
\end{equation*}
\end{proof}


\newpage
\section{Evidence Lower Bound}
The lower bound assumes a simplified form after the variational Bayes M-step. It relies solely on the posterior probabilities $\tau_{ik}$ and $\nu_{vs}$ and the normalizing constants of the Dirichlet and Beta distributions.
\begin{equation*}
\begin{aligned}
\mathcal{L}\left( \ q(.) \ \right) &=  \log \left\{\frac{\Gamma\left(\sum_{k=1}^K \beta_k^0\right) \prod_{k=1}^K \Gamma\left(\beta_k\right)}{\Gamma\left(\sum_{k=1}^K \beta_k\right) \prod_{k=1}^K \Gamma\left(\beta_k^0\right)}\right\}
+\log \left\{\frac{\Gamma\left(\sum_{s=1}^Q \theta_s^0\right) \prod_{s=1}^Q \Gamma\left(\theta_s\right)}{\Gamma\left(\sum_{s=1}^Q \theta_s\right) \prod_{s=1}^Q \Gamma\left(\theta_s^0\right)}\right\}\\
&\quad +\sum_{k \leq l}^K  \sum_{s=1}^Q  \log \left\{\frac{\Gamma\left(\eta_{k l s}^0+\xi_{k l s }^0\right) \Gamma\left(\eta_{k  l s}\right) \Gamma\left(\xi_{k l s}\right)} {\Gamma\left(\eta_{k l s} +\xi_{k  l s} \right) \Gamma\left(\eta_{k l s}^0\right) \Gamma\left(\xi_{k l s}^0\right)}\right\} \\
&\quad -\sum_{i}^N \sum_{k}^K  \tau_{i k} \log \tau_{i k} \ - \sum_{v}^V \sum_{s}^Q  \nu_{v s} \log \nu_{v s}\\
\end{aligned}
\end{equation*}
\begin{proof}
The lower bound can be expressed as:
\begin{equation*}
\begin{aligned}
\mathcal{L}\left(q(.)\right) &= \sum_\bZ \sum_\bW \int \int \int q(\bZ,\bW,\balpha,\bpi,\brho) \log\dfrac{\p\left( \bA,\bZ,\bW, \balpha,\bpi,\brho \right)}{q(\bZ,\bW,\balpha,\bpi,\brho)} \  d\balpha \ d\bpi \ d\brho\\
&= \mathbb{E}_{\mathbf{Z}, \mathbf{W},\boldsymbol{\alpha},\boldsymbol{\rho},\boldsymbol{\pi}}[\log \p(\mathbf{A}, \mathbf{Z}, \boldsymbol{\alpha}, \mathbf{W}, \boldsymbol{\rho},\boldsymbol{\pi}) ] - \mathbb{E}_{\mathbf{Z},\mathbf{W},\boldsymbol{\alpha},\boldsymbol{\rho},\boldsymbol{\pi}}[\log q(\mathbf{Z}, \boldsymbol{\alpha}, \mathbf{W}, \boldsymbol{\rho},\boldsymbol{\pi}) ]
\end{aligned} .
\end{equation*}

We can decompose the following terms as:
\begin{equation*}
\begin{aligned}
\mathbb{E}_{\mathbf{Z},\mathbf{W},\boldsymbol{\alpha},\boldsymbol{\rho},\boldsymbol{\pi}}[\log p(\mathbf{A}, \mathbf{Z}, \boldsymbol{\alpha}, \mathbf{W}, \boldsymbol{\rho},\boldsymbol{\pi}) ] =& \  \mathbb{E}_{\mathbf{Z},\mathbf{W},\boldsymbol{\alpha}}[\log \p(\bA \mid \mathbf{Z},\mathbf{W},\boldsymbol{\alpha})] + \mathbb{E}_{\boldsymbol{\alpha}}[\log p(\boldsymbol{\alpha}) ] \\ 
+& \ \mathbb{E}_{\mathbf{Z},\boldsymbol{\pi}}[\log p( \mathbf{Z} \mid \boldsymbol{\pi}) ] + \mathbb{E}_{\boldsymbol{\pi}}[\log p(\boldsymbol{\pi}) ] \\
+& \ \mathbb{E}_{\mathbf{W},\boldsymbol{\rho}}[\log p( \mathbf{W} \mid \boldsymbol{\rho}) ] + \mathbb{E}_{\boldsymbol{\rho}}[\log p(\boldsymbol{\rho}) ]
\end{aligned} \ ,
\end{equation*}
and
\begin{equation*}
\begin{aligned}
\mathbb{E}_{\mathbf{Z},\mathbf{W},\boldsymbol{\alpha},\boldsymbol{\rho},\boldsymbol{\pi}}[\log q( \mathbf{Z}, \boldsymbol{\alpha}, \mathbf{W}, \boldsymbol{\rho},\boldsymbol{\pi}) ] &= \ 
\mathbb{E}_{\mathbf{Z}}[\log q( \mathbf{Z} ) ] + \ \mathbb{E}_{\boldsymbol{\pi}}[\log q(\boldsymbol{\pi}) ] \\
&+ \ \mathbb{E}_{\mathbf{Z}}[\log q( \mathbf{W} ) ] +  \mathbb{E}_{\boldsymbol{\rho}}[\log q(\boldsymbol{\rho}) ] \\
&+ \mathbb{E}_{\boldsymbol{\alpha}}[\log q(\boldsymbol{\alpha}) ]
\end{aligned} \ .
\end{equation*}

Now, the next step involves developing each of these terms and simplifying them as extensively as possible.

\begin{equation*}
\begin{aligned}
\mathbb{E}_{\mathbf{Z},\mathbf{W},\boldsymbol{\alpha}}[\log \p(\bA \mid \mathbf{Z},\mathbf{W},\boldsymbol{\alpha})] + \mathbb{E}_{\boldsymbol{\alpha}}[\log \p(\boldsymbol{\alpha}) ] 
&=\sum_{i<j}^N \sum_{k, l}^K \sum_{v}^V \sum_{s}^Q \tau_{i k} \tau_{j l} \nu_{v s} \Big\{  A_{ijv} \Big( 
\psi(\eta_{kls}) - \psi(\xi_{kls})  
\Big) + \psi(\xi_{kls} ) \\
&\quad - \psi( \eta_{kls} + \xi_{kls} ) 
\Big\}   +\sum_{k \leq l}^K \sum_{s}^Q  \Big\{ \log \Gamma (\eta_{k l s}^0 + \xi_{k l s}^0)   
 - \log \Gamma (\eta_{k l s}^0) \\ &\quad- \log \Gamma (\xi_{k l s}^0) 
+\left(\eta_{k l s}^0-1\right) \left( \psi(\eta_{kls}) - \psi(\xi_{kls}  + \eta_{kls})\right)+ \\
&\quad \left(\xi_{k l s}^0-1\right) \left( \psi(\xi_{kls} ) - \psi( \eta_{kls} + \xi_{kls} ) \right)\Big\}\\
\end{aligned}
\end{equation*}

\begin{equation*}
\begin{aligned}
\mathbb{E}_{\mathbf{Z},\boldsymbol{\pi}}[\log p( \mathbf{Z} \mid \boldsymbol{\pi}) ] + \mathbb{E}_{\boldsymbol{\pi}}[\log p(\boldsymbol{\pi}) ] 
&=  \sum_{i}^N \sum_{k}^K  \tau_{i k} \left( \psi(\beta_k) - \psi(\sum_{k'} \beta_{k'})\right) \\ 
&+ \log \Gamma(\sum_{k'}\beta_{k'}^0) - \log \left( \sum_{k'} \Gamma(\beta_{k'}^0) \right) + \sum_{k =1 }^K \left(\beta_{k}^0-1\right)  \left( \psi(\beta_k) - \psi(\sum_{k'} \beta_{k'})\right)\\
\end{aligned}
\end{equation*}

\begin{equation*}
\begin{aligned}
\mathbb{E}_{\mathbf{W},\boldsymbol{\rho}}[\log p( \mathbf{W} \mid \boldsymbol{\rho}) ] + \mathbb{E}_{\boldsymbol{\rho}}[\log p(\boldsymbol{\rho}) ] &=  \sum_{v}^V \sum_{s}^Q  \nu_{v s} \left( \psi(\theta_s) - \psi(\sum_{s'} \theta_{s'})\right) \\ 
&+ \log \Gamma(\sum_{s'}\theta_{s'}^0) - \log \left( \sum_{s'} \Gamma(\theta_{s'}^0) \right) + \sum_{s =1 }^Q \left(\theta_{s}^0-1\right)  \left( \psi(\theta_s) - \psi(\sum_{s'} \theta_{s'})\right)\\
\end{aligned}
\end{equation*}

\begin{equation*}
\begin{aligned}
\mathbb{E}_{\mathbf{Z}}[\log q( \mathbf{Z} ) ] + \ \mathbb{E}_{\boldsymbol{\pi}}[\log q(\boldsymbol{\pi}) ] =& \sum_{i}^N \sum_{k}^K  \tau_{i k} \log \tau_{i k}  \\
&+ \log \Gamma(\sum_{k'}\beta_{k'}) - \log \left( \sum_{k'} \Gamma(\beta_{k'}) \right) + \sum_{k =1 }^K \left(\beta_{k}-1\right)  \left( \psi(\beta_k) - \psi(\sum_{k'} \beta_{k'})\right)\\
\end{aligned}
\end{equation*}

\begin{equation*}
\begin{aligned}
\mathbb{E}_{\mathbf{Z}}[\log q( \mathbf{W} ) ] +  \mathbb{E}_{\boldsymbol{\rho}}[\log q(\boldsymbol{\rho}) ] =& \sum_{v}^V \sum_{s}^Q  \nu_{v s} \log \nu_{v s}  \\\\
&+ \log \Gamma(\sum_{s'}\theta_{s'}) - \log \left( \sum_{s'} \Gamma(\theta_{s'}) \right) + \sum_{s =1 }^Q \left(\theta_{s}-1\right)  \left( \psi(\theta_s) - \psi(\sum_{s'} \theta_{s'})\right)\\
\end{aligned}
\end{equation*}

\begin{equation*}
\begin{aligned}
\mathbb{E}_{\boldsymbol{\alpha}}[\log q(\boldsymbol{\alpha}) ] =& \sum_{k \leq l}^K \sum_{s}^Q  \Big\{ \log \Gamma (\eta_{k l s} + \xi_{k l s})   - \log \Gamma (\eta_{k l s}) - \log \Gamma (\xi_{k l s}) \\
&+\left(\eta_{k l s}-1\right) \left( \psi(\eta_{kls}) - \psi(\xi_{kls}  + \eta_{kls})\right)+\left(\xi_{k l s}-1\right) \left( \psi(\xi_{kls} ) - \psi( \eta_{kls} + \xi_{kls} ) \right)\Big\}\\
\end{aligned}
\end{equation*}

Now that all the terms have been developed, it's just a matter of grouping them together, to obtain the ELBO below.

\begin{equation*}
\begin{aligned}
\mathcal{L}\left( \ q(.) \ \right) &= \sum_{k < l}^K \sum_{s}^Q \left( \eta_{kls}^0 + \left(   \sum_{i \ne j}^N  \sum_{v}^V  \tau_{i k} \tau_{j l} \nu_{v s} A_{ijv} \right)  -\eta_{kls} \right)\Big( \psi(\eta_{kls}) - \psi( \eta_{kls} + \xi_{kls} ) \Big)  \\
&+ \sum_{k = 1}^K \sum_{s}^Q \left( \eta_{kks}^0 + \left(   \sum_{i<j}^N  \sum_{v}^V  \tau_{i k} \tau_{j k} \nu_{v s} A_{ijv} \right)  -\eta_{kks} \right)\Big( \psi(\eta_{kks}) - \psi( \eta_{kks} + \xi_{kks} ) \Big)  \\
&+ \sum_{k < l}^K \sum_{s}^Q \left( \xi_{kls}^0 + \left(   \sum_{i \ne j}^N  \sum_{v}^V  \tau_{i k} \tau_{j l} \nu_{v s} (1-A_{ijv}) \right)  -\eta_{kls} \right)\Big( \psi(\xi_{kls}) - \psi( \eta_{kls} + \xi_{kls} ) \Big)  \\
&+ \sum_{k = 1}^K \sum_{s}^Q \left( \xi_{kks}^0 + \left(   \sum_{i<j}^N  \sum_{v}^V  \tau_{i k} \tau_{j k} \nu_{v s} (1-A_{ijv}) \right)  -\xi_{kks} \right)\Big( \psi(\xi_{kks}) - \psi( \eta_{kks} + \xi_{kks} ) \Big)  \\
&+ \sum_{k =1}^K \left( \beta_k^0 +    \sum_{i =1}^N  \tau_{i k}   \ -\beta_k \right) \left( \psi(\beta_k) - \psi(\sum_{k'} \beta_{k'})\right)  \\
&+ \sum_{q =1}^Q \left( \theta_s^0 +    \sum_{v =1}^V  \nu_{v s}   \ -\theta_s \right) \left( \psi(\theta_s) - \psi(\sum_{s'} \theta_{s'})\right)  \\
&- \sum_{i}^N \sum_{k}^K  \tau_{i k} \log \tau_{i k} \ - \sum_{v}^V \sum_{s}^Q  \nu_{v s} \log \nu_{v s}\\
&+ \log \left\{\frac{\Gamma\left(\sum_{k=1}^K \beta_k^0\right) \prod_{k=1}^K \Gamma\left(\beta_k\right)}{\Gamma\left(\sum_{k=1}^K \beta_k\right) \prod_{k=1}^K \Gamma\left(\beta_k^0\right)}\right\}
+\log \left\{\frac{\Gamma\left(\sum_{s=1}^Q \theta_s^0\right) \prod_{s=1}^Q \Gamma\left(\theta_s\right)}{\Gamma\left(\sum_{s=1}^Q \theta_s\right) \prod_{s=1}^Q \Gamma\left(\theta_s^0\right)}\right\}\\
& \quad + \sum_{k \leq l}^K  \sum_{s=1}^Q  \log \left\{\frac{\Gamma\left(\eta_{k l s}^0+\xi_{k l q }^0\right) \Gamma\left(\eta_{k  l s}\right) \Gamma\left(\xi_{k l s}\right)} {\Gamma\left(\eta_{k l s} +\xi_{k  l s} \right) \Gamma\left(\eta_{k l s}^0\right) \Gamma\left(\xi_{k l s}^0\right)}\right\}
\end{aligned}
\end{equation*}
However, by definition of the parameters, we have many terms that cancel each other out:
\begin{itemize}
    \item $ \eta_{kls} = \eta_{kls}^0 + \left(   \sum_{i \ne j}^N  \sum_{v}^V  \tau_{i k} \tau_{j l} \nu_{v s} A_{ijv} \right)$
    \item  $\eta_{kks} = \eta_{kks}^0 + \left(  \sum_{i<j}^N  \sum_{v}^V  \tau_{i k} \tau_{j k} \nu_{v s} A_{ijv} \right)$
    \item $\eta_{kls} = \xi_{kls}^0 + \left(   \sum_{i \ne j}^N  \sum_{v}^V  \tau_{i k} \tau_{j l} \nu_{v s} (1-A_{ijv}) \right)$
    \item $\xi_{kks}= \xi_{kks}^0 + \left(   \sum_{i<j}^N  \sum_{v}^V  \tau_{i k} \tau_{j k} \nu_{v s} (1-A_{ijv}) \right)$
    \item $\beta_k = \beta_k^0 +    \sum_{i =1}^N  \tau_{i k}$
    \item $\theta_s  = \theta_s^0 +    \sum_{v =1}^V  \nu_{v s}$  
\end{itemize}
Hence:
\begin{equation*}
\begin{aligned}
\mathcal{L}\left(q(.)\right) &=  \log \left\{\frac{\Gamma\left(\sum_{k=1}^K \beta_k^0\right) \prod_{k=1}^K \Gamma\left(\beta_k\right)}{\Gamma\left(\sum_{k=1}^K \beta_k\right) \prod_{k=1}^K \Gamma\left(\beta_k^0\right)}\right\}
+\log \left\{\frac{\Gamma\left(\sum_{s=1}^Q \theta_s^0\right) \prod_{s=1}^Q \Gamma\left(\theta_s\right)}{\Gamma\left(\sum_{s=1}^Q \theta_s\right) \prod_{s=1}^Q \Gamma\left(\theta_s^0\right)}\right\}\\
&\quad +\sum_{k \leq l}^K  \sum_{s=1}^Q  \log \left\{\frac{\Gamma\left(\eta_{k l s}^0+\xi_{k l q }^0\right) \Gamma\left(\eta_{k  l s}\right) \Gamma\left(\xi_{k l s}\right)} {\Gamma\left(\eta_{k l s} +\xi_{k  l s} \right) \Gamma\left(\eta_{k l s}^0\right) \Gamma\left(\xi_{k l s}^0\right)}\right\} \\
&\quad -\sum_{i}^N \sum_{k}^K  \tau_{i k} \log \tau_{i k} \ - \sum_{v}^V \sum_{s}^Q  \nu_{v s} \log \nu_{v s}\\
\end{aligned}
\end{equation*}
\end{proof}
\end{appendix}


\end{document}